\documentclass[journal]{IEEEtran}
\usepackage{amsmath, amssymb, amsfonts}
\usepackage{newtxtext, newtxmath}

\usepackage{amsthm}
\usepackage{cases}
\usepackage{color, xcolor}
\usepackage{graphicx}
\usepackage{subfigure}
\usepackage{algorithm}
\usepackage{algorithmic}
\usepackage{epstopdf}
\usepackage{multirow}
\usepackage{rotating}
\usepackage{booktabs}
\usepackage{pifont}
\usepackage[demo, 
            export]{adjustbox}
\usepackage{tabularray}
\UseTblrLibrary{booktabs}
\usepackage{cases}
\usepackage[colorlinks, 
linkcolor=blue, 
anchorcolor=blue, 
citecolor=blue
]{hyperref}

\newtheorem{theorem}{Theorem}
\newtheorem{lemma}{Lemma}

\graphicspath{{figures/}}

\begin{document}

\title{Bi-Level Unsupervised Feature Selection}

\author{Jingjing Liu,  Xiansen Ju,  Xianchao Xiu, and Wanquan Liu, \IEEEmembership{Senior Member,~IEEE}

\thanks{This work was supported in part by the National Natural Science Foundation of China under Grant 62204044 and Grant 12371306,  and in part by the State Key Laboratory of Integrated Chips and Systems under Grant SKLICS-K202302. (\textit{Corresponding author: Xianchao Xiu}.)}
\thanks{Jingjing Liu and Xiansen Ju are with the Shanghai Key Laboratory of Automobile Intelligent Network Interaction Chip and System,  School of Microelectronics,  Shanghai University, Shanghai 200444,  China (e-mail: jjliu@shu.edu.cn, jxs@shu.edu.cn)}
\thanks{Xianchao Xiu is with the School of Mechatronic Engineering and Automation,  Shanghai University,  Shanghai 200444,  China (e-mail: xcxiu@shu.edu.cn).}
\thanks{Wanquan Liu is with the School of Intelligent Systems Engineering, Sun Yat-sen University, Guangzhou 510275, China (e-mail: liuwq63@mail.sysu.edu.cn).}
}

\maketitle

\begin{abstract}
Unsupervised feature selection (UFS) is an important task in data engineering. However, most UFS methods construct models from a single perspective and often fail to simultaneously evaluate feature importance and preserve their inherent data structure, thus limiting their performance. To address this challenge, we propose a novel bi-level unsupervised feature selection (BLUFS) method, including a clustering level and a feature level. Specifically, at the clustering level, spectral clustering is used to generate pseudo-labels for representing the data structure, while a continuous linear regression model is developed to learn the projection matrix. At the feature level, the $\ell_{2,0}$-norm constraint is imposed on the projection matrix for more effectively selecting features. To the best of our knowledge, this is the first work to combine a bi-level framework with the $\ell_{2,0}$-norm. To solve the proposed bi-level model, we design an efficient proximal alternating minimization (PAM) algorithm, whose subproblems either have explicit solutions or can be computed by fast solvers. Furthermore, we establish the convergence result and computational complexity. Finally, extensive experiments on two synthetic datasets and eight real datasets demonstrate the superiority of BLUFS in clustering and classification tasks.
\end{abstract}

\begin{IEEEkeywords}
  Unsupervised feature selection, bi-level, clustering, $\ell_{2,0}$-norm, proximal alternating minimization
\end{IEEEkeywords}

\section{Introduction}\label{Introduction}

Recently, the rapid development of information technology has led to a significant increase in the number of instances and feature dimensions in various application datasets, especially in the fields of image processing and text classification, see \cite{jordan2015machine,bolon2020feature,hancer2020survey,li2024exploring,zheng2025structured}.
Unsupervised feature selection (UFS), as an effective dimensionality reduction technique, has attracted much attention due to its independence from label information and good interpretability \cite{li2015unsupervised,9051653,yang2025tensor}.
Nowadays, UFS has been successfully served as an effective preprocessing step for downstream tasks covering clustering and classification \cite{solorio2020review}.

In high-dimensional data, UFS aims to identify a subset of features that retains the most discriminative information while preserving the intrinsic data structure \cite{mitra2002unsupervised,tang2023unsupervised}. Generally speaking, existing UFS methods can be categorized into three types: filtering-based, wrapper-based, and embedded-based. The representative UFS methods include Laplacian score (LapScore) \cite{he2005laplacian}, multi-cluster feature selection (MCFS) \cite{cai2010unsupervised}, unsupervised discriminative feature selection (UDFS) \cite{yang2011}, structured optimal graph feature selection (SOGFS) \cite{nie2016unsupervised}, and robust neighborhood embedding (RNE) \cite{liu2020robust}. Notably, embedded-based methods integrate structure learning and feature selection within a unified framework, thereby combining the strengths of both filtering- and wrapper-based methods.

\begin{figure*}[t]
  \makeatletter
  \renewcommand{\@thesubfigure}{\hskip\subfiglabelskip}
  \makeatother
  \centering
  \includegraphics[width=\textwidth]{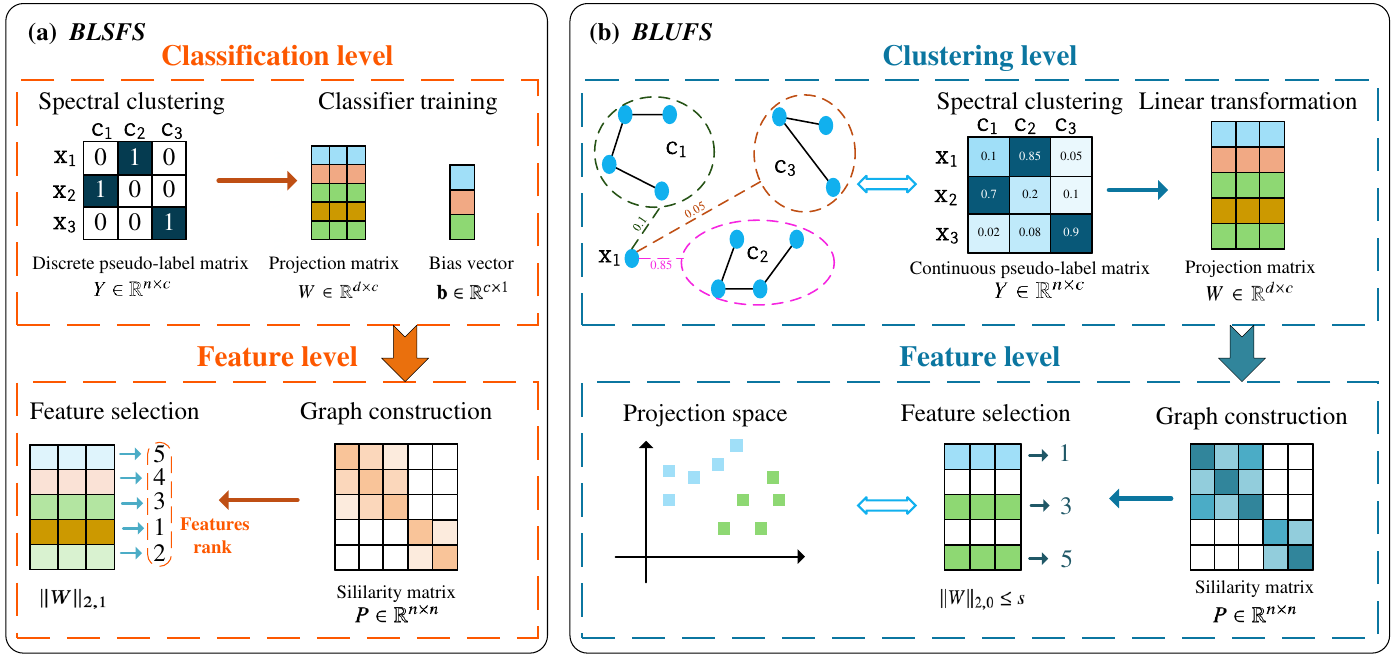}
  \caption{Comparisons of BLSFS \cite{hu2024bi} and our proposed BLUFS. BLSFS is composed of classification and  feature levels. In the classification level, discrete pseudo-labels are obtained through spectral clustering, and then a projection matrix is learned via classifier training. In the feature level, a similarity matrix is obtained using an adaptive graph learning model. Finally, by applying the \(\ell_{2,1}\)-norm to \(W\), features with higher rankings are obtained.  BLUFS is composed of  clustering and feature levels. Unlike BLSFS, in the clustering level, continuous pseudo-labels representing the clustering structure are obtained through spectral clustering, and a projection matrix is obtained via a linear transformation model. In the feature level, after obtaining the similarity matrix, the number of features is strictly limited and sufficient sparsity is ensured by imposing the \(\ell_{2,0}\)-norm constraint on \(W\).}
  \label{flow chat}
\end{figure*}

Principal component analysis (PCA) and its variants are widely recognized in the UFS domain for their simplicity and computational efficiency \cite{5640135,greenacre2022principal}. Recent studies have focused on enhancing PCA-based methods by incorporating sparsity constraints, which significantly improves their feature selection capability \cite{zou2018selective,drikvandi2023sparse}.
For example, Yang et al. \cite{yang2011} proposed a sparse PCA model by introducing the $\ell_{2,1}$-norm regularization \cite{nie2010efficient}, which promotes group sparsity in the projection matrix. 
Li et al. \cite{9580680} introduced the $\ell_{2,p}$-norm with $p\in(0,1)$ \cite{wang2017ell} and reformulated PCA as a reconstruction error minimization problem, which is denoted as SPCAFS.
To improve the interpretability, Nie et al. \cite{nie2022learning} proposed FSPCA, a feature sparse PCA model that employs the $\ell_{2,0}$-norm \cite{cai2013exact} to enforce row sparsity in the projection matrix, enabling consistent feature selection across multiple components. 
In a related work, Nie et al. \cite{10409536} designed a coordinate descent framework with provable convergence guarantees.
Zhou et al. \cite{ZHOU2025107512} further introduced an evolutionary sparsity-based method to enhance the flexibility and performance.
In fact, the $\ell_{2,0}$-norm has been applied in many data engineering fields.
For example, Chen et al. \cite{chen2021l2} constructed $\ell_{2,0}$-norm-based multivariate generalized linear models; Zhang et al. \cite{zhang2023structured} introduced the $\ell_{2,0}$-norm and non-negative matrix factorization for fault detection; Zhu et al. \cite{zhu2025joint} applied the $\ell_{2,0}$-norm to subspace clustering; Huang et al. \cite{huang2025inhomogeneous} proposed the $\ell_{2,0}$-norm penalized graph trend filtering model for estimating piecewise smooth signals.
In fact, compared with the $\ell_{2,1}$-norm and the $\ell_{2,p}$-norm with $p \in (0,1)$, the $\ell_{2,0}$-norm directly enforces group sparsity without requiring intricate parameter tuning to balance sparsity and accuracy \cite{WANG2023109472,WANG2023224}. This property allows it to preserve key data features and improve feature-level interpretability \cite{xiu2022efficient}.

The aforementioned methods conduct feature selection solely at the feature level, without considering downstream tasks such as clustering or classification. However, in many real-world scenarios, feature selection is a hierarchical decision-making process \cite{liu2021investigating}. Therefore, adopting a hierarchical structure to model dependencies between multiple tasks is particularly important for feature selection. Recently, researchers have begun to explore feature selection by integrating information from different levels. Zhou et al. \cite{zhou2023bi} attempted to perform feature selection by integrating multiple feature selection methods at different levels, which is called is BLSEF. At the clustering level, multiple clustering results are generated using various algorithms and then aggregated via self-paced learning to form a consensus that subsequently guides feature selection at the feature level. 
At the feature level, multiple base feature selection methods are integrated to learn a consensus feature score. Hu et al. \cite{hu2024bi} tried to perform feature selection at both the classification level and the feature level, so that the selected features can be suitable for different downstream tasks. Specifically, spectral clustering is employed at the classification level to generate pseudo-labels that guide feature selection. Meanwhile, adaptive graph learning is applied at the feature level, followed by a ranking procedure to select the most informative features. This method is abbreviated as BLSFS.
It is worth noting that although the above bi-level UFS methods show good performance, there exist two shortcomings. On the one hand, the pseudo-labels in spectral clustering are too strict, and on the other hand, the \(\ell_{2,0}\)-norm has not yet been used to directly achieve feature sparsity.

Inspired by the aforementioned observations, we propose a bi-level unsupervised feature selection method, called BLUFS, to distinguish it from existing methods. 
In our proposed BLUFS, spectral clustering is employed to generate a more flexible scaled cluster indicator matrix. 
Unlike BLSFS, the constraints are relaxed to construct a continuous matrix representation, thus avoiding the NP-hard challenges \cite{li2024unsupervised}. 
This enables a more flexible representation of the similarities between data points and their relationship with clustering information, enhancing the accuracy and robustness of feature selection. 
Additionally, a projection matrix is learned using the linear transformation model and the obtained pseudo-labels, and the Frobenius norm is introduced to regularize overfitting.
The learned projection matrix is then used to construct an adaptive similarity graph in the projected feature space.
The $\ell_{2, 0}$-norm constraint is imposed on the projection matrix to eliminate redundant features while preserving the most discriminative ones.
A unified optimization framework is established to optimize both aspects, thereby promoting mutual reinforcement between the clustering and feature levels. 
See Fig. \ref{flow chat} for the diagram of our proposed BLUFS and comparisons with BLSFS \cite{hu2024bi}.

The main innovations of this paper can be summarized as the following three aspects.
\begin{itemize}
\item  We construct a new model for UFS. Unlike existing bi-level methods, we employ continuous pseudo-labels to represent the data structure. Additionally, we attempt to integrate the \(\ell_{2,0}\)-norm with the bi-level method to achieve sufficiently sparse representation.
\item We design a proximal alternating minimization (PAM)-based optimization algorithm and theoretically prove that it can converge to the global optimum under mild conditions.
\item  We conduct extensive numerical experiments to evaluate the superiority of our proposed BLUFS in clustering and classification. In addition, we analyze the importance of different levels. The results show that although the feature level plays a dominant role in the feature selection process, the clustering level is also indispensable, which suggests the necessity of the bi-level structure.

\end{itemize}

The main content of this article is arranged as follows.
Section \ref{Preliminaries} introduces the notations and related work.
Section \ref{Bi-Level Unsupervised Feature Selection} presents our new model. 
Section \ref{Optimization} provides the optimization algorithm, convergence theory, and computational complexity.
Section \ref{Experiments} reports the numerical experiments and results.
Section \ref{Conclusion} concludes this paper.

\section{Preliminaries}\label{Preliminaries}
This section introduces some notations used throughout this paper and related work.

\subsection{Notations}
In this paper, matrices are denoted by capital letters, vectors by boldface letters, and scalars by lowercase letters.
Let $\mathbb{R}^{d \times n}$ be the sets of all $d \times n$-dimensional matrices.
For an arbitrary matrix $X$, let $X_i$ represent its $i$-th row and $X_{ij}$ stand for the $ij$-th element.
$X^\top$ denotes the transpose of matrix $X$ and $X^{-1}$ represents the inverse of matrix $X$.
The Frobenius norm of $X$ is $\|X\|_F = (\sum_{i=1}^d \sum_{j=1}^n X_{ij}^2)^{1/2}$.
For $p\in (0, 1]$, the $\ell_{2, p}$-norm of $X$ is defined as $\|X\|_{2, p}=(\sum_{i=1}^{d}\|X_i\|^p)^{1/p}$.
When $p=1$, it is exactly the $\ell_{2, 1}$-norm. 
In addition,  $\|X\|_{2, 0}$  counts the numbers of non-zero rows of $X$,  respectively.
An additionally notation will be introduced wherever it appears.

\subsection{Spectral clustering}
Spectral clustering \cite{ng2001spectral} is an effective clustering method based on spectral analysis techniques in graph theory, which can effectively handle complex data distributions that are not linearly separable. Due to the lack of label information, some scholars have recently used spectral clustering to obtain pseudo-labels as a substitute for true labels in feature selection.
The primary objective of clustering is to partition data points into distinct classes such that points with high similarity are grouped into the same class, while others are assigned to different classes \cite{von2007tutorial}.
Therefore,  a \( k \)-nearest neighbors ($k$-NN) graph is constructed,  and the Gaussian kernel is chosen as the similarity evaluation function. Here, \(X\in \mathbb{R}^{d \times n}\) is the instance matrix, with \(\mathbf{x}_i\) representing different instances.
The similarity matrix \( S \) is defined as 
\begin{equation}
  S_{ij} =
  \begin{cases}
  \exp\left(-\frac{\|\mathbf{x}_i - \mathbf{x}_j\|^2}{2\sigma^2}\right),  & \mathbf{x}_i \in \mathcal{N}_k(\mathbf{x}_j) \text{ or } \mathbf{x}_j \in \mathcal{N}_k(\mathbf{x}_i) \\
  0,  & \text{otherwise}.
  \end{cases}
  \end{equation}
Specifically, the similarity matrix \( S \) is constructed to measure the similarity between different instances, where \( \mathcal{N}_k(\mathbf{x}) \) denotes the set of \(k\)-NN of \( \mathbf{x} \).
The pseudo-label generation is formulated as follows

\begin{equation}
\begin{aligned}
 \min_{Y}\quad &\text{Tr}(Y^\top L Y) \\
 \text{s.t. }\quad &Y \in \{0,1\}^{n \times c}, \sum_{j=1}^{c} Y_{ij} = 1 \quad \forall i = 1, 2, \ldots, n,
\end{aligned}
\label{BLSFS2}
\end{equation}
where $\text{Tr}(\cdot)$ represents the trace of a matrix, $Y \in \mathbb{R}^{n\times c}$ serves as the pseudo-label matrix, ${D}\in \mathbb{R}^{n\times n}$ is a diagonal matrix where $D_{ii}=\sum_{j=1}^{n}S_{ij}$, and the normalized Laplacian matrix $L$ is defined as $L = I - D^{-\frac{1}{2}} S D^{-\frac{1}{2}}$. Here,
 \(\sum_{j=1}^{c} Y_{ij} = 1\) ensures that any instance \(x_i\) has only one unique label. Based on the definition of the matrix \(L\), the objective of \eqref{BLSFS2} can be rewritten as the form of
\begin{equation}\label{BLSFS3}
\begin{aligned}
\text{Tr}(Y^\top LY) &= \text{Tr}(Y^\top Y) - \text{Tr}(Y^\top D^{-\frac{1}{2}}SD^{-\frac{1}{2}}Y) \\
&= n - \text{Tr}(Y^\top D^{-\frac{1}{2}} S D^{-\frac{1}{2}} Y).
\end{aligned}
\end{equation}
Let \( \hat{S} = D^{-\frac{1}{2}}SD^{-\frac{1}{2}} \), thus \eqref{BLSFS3} can be transformed into
\begin{equation}
\begin{aligned}
\max_{{Y}}\quad&\text{Tr}({Y}^\top \hat{{S}} {Y})\\
\text{s.t. }\quad &Y \in \{0,1\}^{n \times c}, \sum_{j=1}^{c} Y_{ij} = 1 \quad \forall i = 1, 2, \ldots, n.
\label{BLSFS4}
\end{aligned}
\end{equation}
By solving the above problem, the optimal pseudo-label matrix \(Y\) can be obtained, which can be used to guide the feature selection.

\subsection{Bi-Level Spectral Feature Selection}
For the BLSFS, it merges spectral clustering with adaptive graph learning to perform feature selection at both classification and feature levels. At the classification level, pseudo-labels are generated through spectral clustering by solving \eqref{BLSFS4}. Then a linear classifier is trained using pseudo-labels to acquire the regression matrix \( W \in \mathbb{R}^{d \times c}\),  which guides the feature selection.
The regression matrix \( W  \) is learned through a linear regression to minimize the loss function
\begin{equation}
\begin{aligned}
\min_{{W},  {b}}\quad&\| {X}^\top{W}  + \mathbf{1}\mathbf{b}^\top - {Y} \|_F^2 \\
\text{s.t. }\quad &Y \in \{0,1\}^{n \times c}, \sum_{j=1}^{c} Y_{ij} = 1 \quad \forall i = 1, 2, \ldots, n,
\label{BLSFS5}
\end{aligned}
\end{equation}
where $\mathbf{b} \in \mathbb{R}^{c \times 1}$ stands for the bias vector and $\mathbf{1} =[1, 1, ..., 1]\in \mathbb{R}^{n \times 1} $. 
Combining \eqref{BLSFS4} and \eqref{BLSFS5}, it yields a multiobjective optimization problem as 
\begin{equation}
\begin{aligned}
  \begin{cases}
    \min \limits_{W,b}\quad \| {X}^\top{W}  + \mathbf{1}\mathbf{b}^\top - {Y} \|_F^2\\
    \max\limits_{Y}\quad\textrm{Tr}(Y^\top\hat S Y)
  \end{cases}
    \text{s.t. }\quad Y \in \{0,1\}^{n \times c}.
    \end{aligned}
\end{equation}

At the feature level, the regression matrix \( {W} \) projects the original data in a low-dimensional space to maintain the geometric structure by constructing an adaptive similarity graph. The optimization objective is
\begin{equation}
\begin{aligned}
 \min_{P,W}\quad &\sum_{i,j=1}^{n} ( \|W^\top x_i - W^\top x_j\|_2^2 P_{ij} + \alpha P_{ij}^2 ) + \gamma \|W\|_{2,1} \\
 \text{s.t. } \quad &P \geq 0, P^\top \mathbf{1} = \mathbf{1}.
\end{aligned}
\end{equation}
The similarity matrix \( P\in \mathbb{R}^{n \times n} \) is constructed so that \( P_{ij} \) denotes the similarity between instances \( \mathbf{x}_i \) and \( \mathbf{x}_j \) in the transformed space. By imposing the \(\ell_{2,1}\)-norm regularization on \(W\), a feature ranking is obtained to select the most important features.

\section{Bi-Level Unsupervised Feature Selection}\label{Bi-Level Unsupervised Feature Selection}

It is worth noting that BLSFS employs a discrete cluster indicator matrix, which oversimplifies the representation of the data structure and leads to unreliable pseudo-labels.
In addition, feature selection using \(\ell_{2,1}\)-norm regularization does not strictly limit the number of selected features, which may lead to insufficient sparsity.
To address the above problems, this section first describes the specific forms of clustering level and feature level, and then proposes our new bi-level model. 


\subsection{Clustering Level}

The performance of UFS is significantly affected by noise and irrelevant features, mainly due to lack of information on the labels \cite{10857472}.  To this end, we construct a matrix \( E = [\mathbf{e}_1, \mathbf{e}_2, \ldots, \mathbf{e}_n]^\top \in \{0, 1\}^{n \times c} \), where each \(\mathbf{e}_i \in \{0, 1\}^{c \times 1}\) represents the clustering indicator vector corresponding to the sample \(\mathbf{x}_i\). Specifically, the \(j\)-th element of \(\mathbf{e}_i\) is set to 1 if sample \(\mathbf{x}_i\) is assigned to the \(j\)-th cluster; otherwise, it remains 0.

On this basis, a matrix \( Y \in \mathbb{R}^{n \times c} \) is introdued to represent the scaled clustering indicator matrix  as
\begin{equation}
Y = [\mathbf{y}_1,  \mathbf{y}_2, ...,  \mathbf{y}_n]^\top = E(E^\top E)^{-\frac{1}{2}}, 
\end{equation}
where \( \mathbf{y}_i \) is the scaled cluster indicator of \( \mathbf{x}_i \). It turns out that
\begin{equation}
Y^\top Y =   (E^\top E)^{-\frac{1}{2}}E^\top E(E^\top E)^{-\frac{1}{2}}  = I, 
\end{equation}
where $  I\in \mathbb{R}^{n \times n}$ is the identity matrix.

As mentioned above, spectral clustering is used to learn pseudo-labels, which is given by
\begin{equation}\label{dis}
\min_{Y}\quad\text{Tr}(Y^\top L Y) \quad \text{s.t.} \quad Y = E (E^{\top} E)^{-\frac{1}{2}}.
\end{equation}
However, it is inherently discrete in nature, given that the elements of the feasible solution are restricted to two distinct values. This characteristic renders the problem NP-hard.
To address this issue, the constraints on \( Y \) are relaxed to a continuous form, which simplifies the optimization process and enables a more flexible representation of data point correlations, thereby enhancing the accuracy of clustering.
Then, \eqref{dis} can be relaxed to
\begin{equation}
  \label{eq:sp}
\min_{Y}\quad\text{Tr}(Y^\top L Y) \quad \text{s.t.} \quad Y^{\top} Y = I.
\end{equation}
Similarly, according to \eqref{BLSFS3}, it can be formulated as
\begin{equation}
  \label{eq:sp1}
\max_{Y}\quad\text{Tr}(Y^\top \hat{S} Y)  \quad \text{s.t.} \quad Y^{\top} Y = I.
\end{equation}
After obtaining the scaled cluster indicator matrix \(Y\), the projection matrix \(W\) is constructed on the data matrix \(X\) through a linear transformation model, that is
\begin{equation}
\label{eq:clustering}
\min\limits_{W, Y}\quad \|X^\top W-Y\|_{F}^2+ \lambda \|W\|_{F}^2, 
\end{equation}
where \( W \in \mathbb{R}^{d \times c}\) is the linear  projection matrix,  and \( \lambda \) is the regularization parameter.
In the objective, the first term is used to measure the association between the features and pseudo-labels. The second term is the regularization term, which is applied to avoid overfitting.

Finally, by integrating the spectral clustering problem \eqref{eq:sp1} and the linear transformation problem \eqref{eq:clustering}, the clustering level can be expressed as 
\begin{equation}
  \begin{cases}
    \min\limits_{W}\quad \|X^\top W-Y\|_{F}^2+ \lambda \|W\|_{F}^2 \\
    \max\limits_{Y}\quad \textrm{Tr}(Y^\top\hat S Y)
  \end{cases}
  \begin{aligned}
    \textrm{s.t.}\quad Y^\top Y=I.
    \end{aligned}
    \label{Clustering level}
\end{equation}
Unlike BLSFS, we use a continuous pseudo-label matrix to replace the true labels, which provides a flexible representation of the clustering structure. Additionally, we establish a linear transformation model to build the relationship between different classes and features and introduce a regularization term to prevent overfitting.
\subsection{Feature Level}


\begin{figure}[t]
  \centering
  \subfigcapskip=-1pt
  \subfigure[$\ell_{2,1}$-norm]{
      \centering
      \includegraphics[width=4cm]{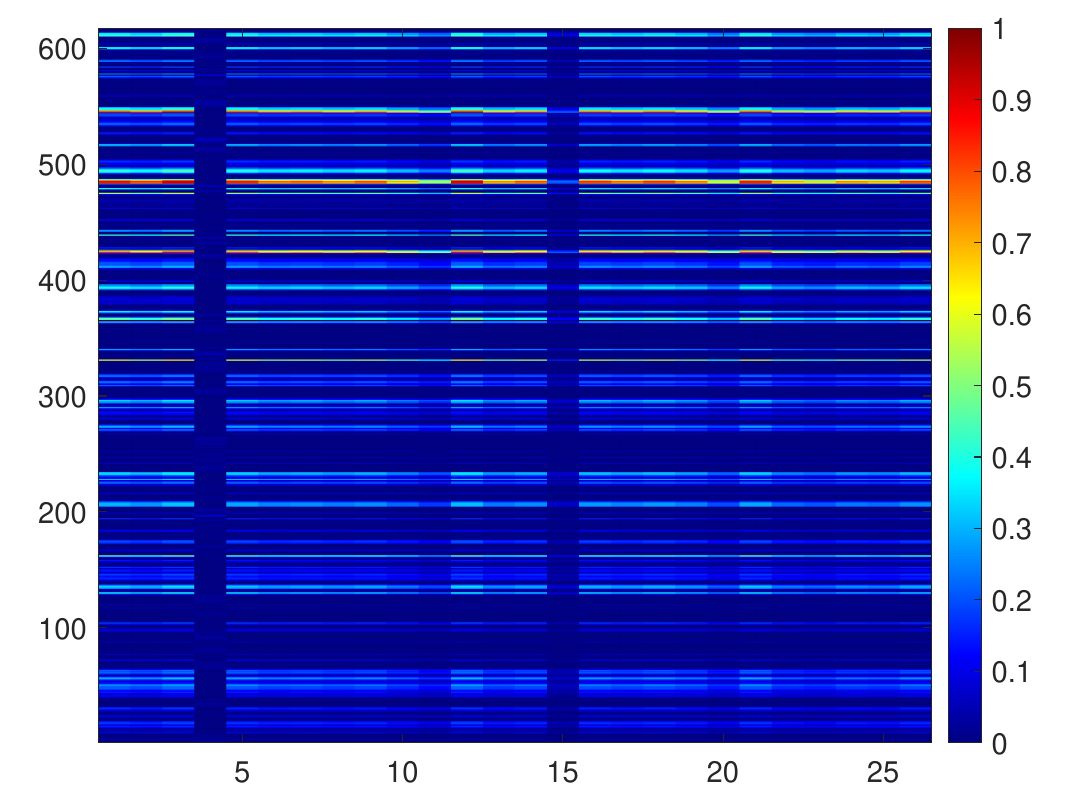}
  } 
  \subfigcapskip=-1pt
  \subfigure[$\ell_{2,0}$-norm]{
      \centering
      \includegraphics[width=4cm]{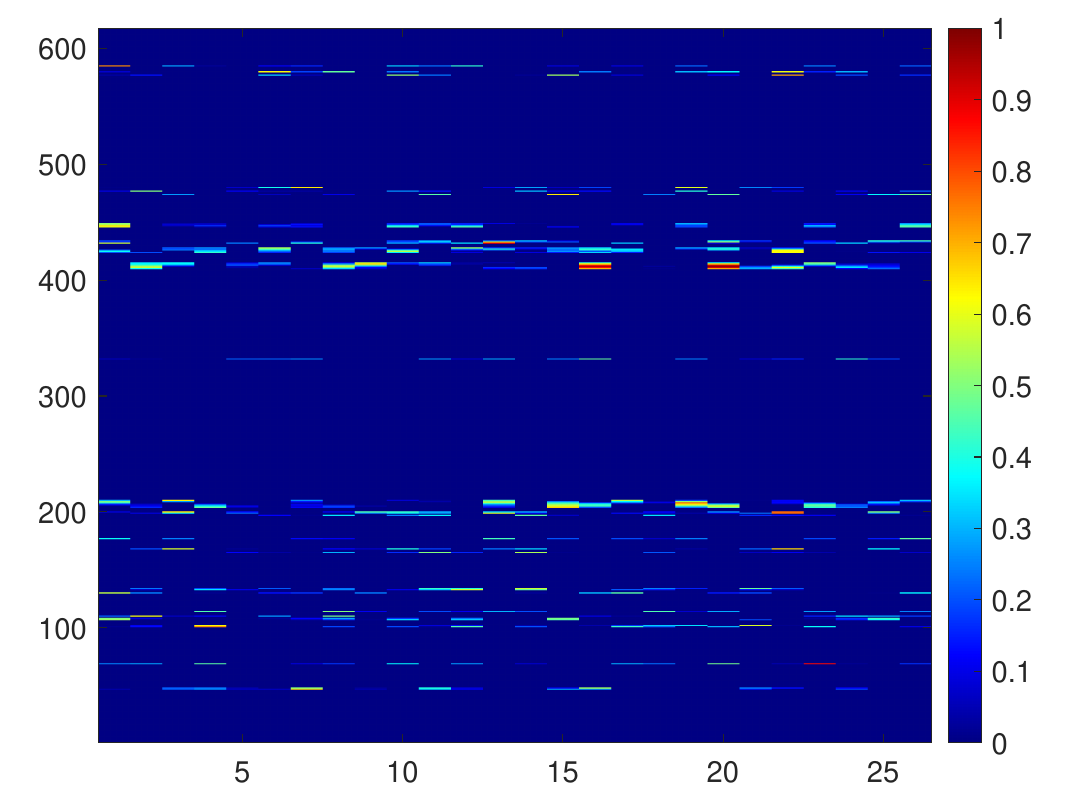}
  } 
  \vspace{-0.1cm}
  \caption{Heatmap visualization of learned sparse matrices by $\ell_{2,1}$-norm and $\ell_{2,0}$-norm.}
  \centering
  \label{Norm}
  \end{figure}

At the feature level, an adaptive learning graph model \cite{7835729} is employed to extract the most intrinsic structural information. Its main purpose is to learn an adaptive similarity matrix \(P\), which is used to represent the similarity between different instances in the projected space \cite{9868157}. Then, the \(\ell_{2,0}\)-norm constraint is imposed on \(W\) to retain the most critical features.
 Therefore, the final feature selection model is formulated as
\begin{equation}
  \begin{aligned}
  \min_{W, P}\quad&\sum_{i, j=1}^n\|W^\top x_i-W^\top x_j\|_2^2 P_{ij}+\mu \|P\|_{F}^2  \\
  \textrm{s.t.}\quad&P\geq 0,  P^\top \textbf{1}=\textbf{1},   \|W\|_{2, 0} \leq s,
  \end{aligned}
  \end{equation}
where \(s\) is the number of features to be selected, which can be set according to actual need, the constraints \( P \geq 0 \) and \( P^{\top} \mathbf{1} = \mathbf{1} \) ensure that \( P \) remains a valid probability distribution. 

Based on previous studies, the $\ell_{2,1}$-norm has been widely adopted in sparse learning due to its ability to induce structured sparsity. However, it may compromise the effectiveness of the associated methods. Specifically, in objective functions that incorporate the $\ell_{2,1}$-norm, a trade-off parameter is typically required to balance the empirical loss and the sparsity.
To provide an intuitive visualization of the resulting sparsity pattern, Fig. \ref{Norm}(a) presents the heatmap of the sparse matrix obtained by the BLSFS method under the $\ell_{2,1}$-norm regularization. It is difficult to clearly distinguish the sparse rows, which requires manual ranking and selection of features. This process often leads to unstable feature subsets.
In contrast, the sparsity constraint induced by the $\ell_{2,0}$-norm in BLUFS yields a much more explicit sparse structure. As illustrated in Fig. \ref{Norm}(b), the elements in the deep blue rows are exactly zero, indicating complete row sparsity. The $\ell_{2,0}$-norm avoids the need for tuning regularization parameters and manual feature ranking. More importantly, it offers a guarantee of stability. Therefore, we employ the $\ell_{2,0}$-norm at the feature level to ensure sufficient sparsity, making it particularly well-suited for high-dimensional feature selection.

\subsection{Our  Model}
To combine the clustering level and the feature level, we propose the bi-level unsupervised feature selection (BLUFS)  as
\begin{equation}
  \begin{aligned}
  \min_{P, Y, W}\quad &\|X^\top  W-Y\|_F^2 +\lambda\|W\|_F^2 -\alpha\textrm{Tr}(Y^\top \hat S Y)\\
    &+\beta\sum_{i, j=1}^n\|W^\top \mathbf{x}_i -W^\top \mathbf{x}_j\|_2^2P_{ij}+\mu\|P\|_F^2\\
    \textrm{s.t.}\quad&\|W\|_{2, 0} \leq s ,  P^\top \textbf{1}=\textbf{1},  P \geq 0 ,  Y^\top Y=I.
  \end{aligned}
  \label{eq:3}
\end{equation}
Indeed, it is a multi-objective optimization problem, and the constraints include non-convex sparse constraints and orthogonal constraints, which are difficult to calculate directly. Hence, the following section will develop an efficient optimization algorithm with convergence analysis.

\section{Optimization Algorithm}\label{Optimization}

This section presents the optimization scheme, followed by convergence analysis and complexity analysis.

\subsection{Optimization Scheme}
Denote the objective of \eqref{eq:3} by $f(P,  W,  Y)$.
Under the PAM framework, each variable can be updated alternately via
\begin{align}
\label{1}
&P^{k+1} = \arg \min_{P}\quad f(P,  W^k,  Y^k) + \tau_1 \|P - P^k\|_F^2,  \\
\label{2}
&W^{k+1} = \arg \min_{W}\quad f(P^{k+1},  W,  Y^k) + \tau_2 \|W - W^k\|_F^2,  \\
\label{3}
&Y^{k+1} = \arg \min_{Y}\quad f(P^{k+1},  W^{k+1},  Y) + \tau_3 \|Y - Y^k\|_F^2, 
\end{align}
where $\tau_1,  \tau_2,  \tau_3 > 0$ and $k$ is the iteration number.
It is worth noting that quadratic penalties are introduced to ensure convergence without significantly increasing the complexity.
The general iterative framework is presented in Algorithm \ref{algorithm 1},  and the detailed rules for $P$,  $W$,  and $Y$ are analyzed below.

\begin{algorithm}[t]
  \caption{Optimization algorithm for BLUFS  }
  \label{algorithm 1}
  \textbf{Input:} Data $X$,  parameters $\tau_1, \tau_2, \tau_3, \lambda, \alpha, \beta,\mu$ \\
  \textbf{Initialize:}   $k=0$, $(P^{k}, W^{k}, Y^{k})$\\
  \textbf{While} not converged \textbf{do}
  \begin{enumerate}
      \item Update $P^{k+1}$  by \eqref{1}
      \item Update $W^{k+1}$  by \eqref{2}
      \item Update $Y^{k+1}$  by \eqref{3}
      \item Check convergence
  \end{enumerate}

  \textbf{End While}

  \textbf{Output:} $(P^{k+1}, W^{k+1}, Y^{k+1})$
\end{algorithm}

\subsubsection{Update $P$} Fixing the other variables,  the subproblem in terms of $P$ can be expressed as
\begin{equation}
  \begin{aligned}
    \min_{P}\quad&\beta\sum_{i, j=1}^n\|W^\top \mathbf{x}_i-W^\top \mathbf{x}_j\|_2^2 P_{ij}+\mu \|P\|_{F}^2+\tau_1\|P-P^k\|_F^2  \\
    \textrm{s.t.}\quad&P\geq 0,  P^\top \textbf{1}=\textbf{1}.
  \end{aligned}
  \label{eq:4}
\end{equation}
To simplify, $V_{i}$ denotes a column vector and $V_{ij}$ is expressed as
\begin{equation}
V_{ij}=\|W^\top \mathbf{x}_i-W^\top \mathbf{x}_j\|_2^2.
  \label{eq:6}
\end{equation}
Additionally, the Frobenius norm term is expanded and substituted into the objective function
\begin{equation}
    \min_{P} \quad\sum_{i,j}^n ( V_{ij} P_{ij} + (\mu + \tau_1) P_{ij}^2 - 2 \tau_1 P_{ij} P_{ij}^k ).
\end{equation}
Add the constraint \( P^\top \textbf{1}=\textbf{1}\), introduce the multiplier \( \eta_j \), and construct the Lagrangian function
\begin{equation}  \label{eq:8}
\begin{aligned}
 \mathcal{L}(P_{ij}, \eta) = &\sum_{i,j}^n ( V_{ij} P_{ij} + (\mu + \tau_1) P_{ij}^2 - 2 \tau_1 P_{ij} P_{ij}^k )\\
 &+ \sum_j^n \eta_j ( \sum_i^n P_{ij} - 1 ).
  \end{aligned}
\end{equation}
Taking the partial derivative with respect to each $P_{ij}$ and setting it to zero, and incorporating the constraint $P_{ij} \geq 0$, the closed-form solution for $P_{ij}$ can be obtained as
\begin{equation}
  P_{ij}=\max~\{-\frac{V_{ij}}{2\mu}+\frac{\tau_1}{\mu+\tau_1}P_{ij}^k+\eta, 0\}.
  \label{eq:9}
\end{equation}

To construct a sparse similarity matrix, a method similar to the $k$-NN graph is adopted: for each instance, only similarities to its $k$-NN are retained, while all other entries are set to zero. 
By sorting the column vector \( V_i \) in ascending order,  that is,  \( V_{i1} \leq V_{i2} \leq \cdots \leq V_{in} \),  it can be inferred that
\begin{equation}
  P_{ij} =
  \begin{cases}
    -\frac{V_{i(j-1)}}{2\mu}+\frac{\tau_1}{\mu+\tau_1}P_{i(j-1)}^k+\eta,  & j \leq k \\
  0,  & \text{otherwise}.
  \end{cases}
  \label{eq:10}
\end{equation}
Here, $\mu$ is treated as a hyperparameter and $\eta$ is an unknown constant. Based on \eqref{eq:10} and the constraint \(P_{i}^\top\textbf{1}=1\),  it follows that
\begin{equation}
  \sum_{j=1}^k(-\frac{V_{ij}}{2\mu}+\frac{\tau_1}{\mu+\tau_1}P_{ij}^k+\eta)=1.
\end{equation}
Solving this yields a closed-form expression for $\eta$, i.e.,
\begin{equation}
\frac{1}{k}+\frac{1}{k}\sum_{j=1}^k (\frac{V_{ij}}{2\mu}-\frac{\tau_1}{\mu+\tau_1}P_{ij}^k)=\eta.
\label{eq:11}
\end{equation}
When the value of \( \eta \) is determined by \eqref{eq:11},  the value of \( P \) can be obtained using \eqref{eq:9}.

\subsubsection{Update $W$} Fixing the other variables,  the subproblem in terms of $W$ can be expressed as
\begin{equation}
  \begin{aligned}
    \label{eq:12}
  \min_{W}\quad&\|X^\top W-Y\|_{F}^2+\lambda\|W\|_{F}^2+\tau_2\|W-W^k\|_F^2\\
  &+\beta\sum_{ij=1}^n\|W^\top x_i-W^\top x_j\|_2^2P_{ij}\\
  \textrm{s.t.}\quad&\|W\|_{2, 0}\leq s.
  \end{aligned}
\end{equation}
Setting the partial derivative of the objective concerning \( W \) as zero, it holds
\begin{equation}
  2X(X^{\top}W - Y) + 2\lambda W + 2XL_{P}X^{\top}W +2\tau_2(W-W^k)= 0,
  \label{eq:13}
\end{equation}
and thus the solution is
\begin{equation}
  W=(XX^\top+(\lambda+\tau_2)I+\beta X L_PX^\top)^{-1} (X Y+\tau_2 W^k), 
  \label{eq:14}
\end{equation}
where $D_P \in \mathbb{R}^{n\times n}$ is the diagonal matrix satisfying the relationship  $D_P=diag\{ \sum_{j=1}^{n} P_{ij}\}$,  and  $ L_P = D_P-(P + P^\top)/2$ is the Laplacian matrix of $P$.  

\subsubsection{Update $Y$} Fixing the other variables,  the subproblem in terms of $Y$ can be expressed as
\begin{equation}
\begin{aligned}
\min_{Y} \quad & \|X^\top W - Y\|_F^2 - \alpha\, \text{Tr}(Y^\top \hat{S} Y) + \tau_3 \|Y - Y^k\|_F^2 \\
\text{s.t.} \quad & Y^\top Y = I.
\end{aligned}
\label{31}
\end{equation}
To solve the subproblem with respect to $Y$ under the orthogonality constraint $Y^\top Y = I$, we adopt an exact penalty approach. The original problem is reformulated as an unconstrained optimization on a Frobenius norm ball $\mathcal{B}_\rho = \{ Y \in \mathbb{R}^{c \times n} \mid \|Y\|_F \leq \rho \}$. The new objective function is defined as
\begin{equation}
\begin{aligned}
\min_{Y \in \mathcal{B}_\rho} \quad h(Y)&= \|X^\top W - Y\|_F^2 - \alpha\, \text{Tr}(Y^\top \hat{S} Y) + \tau_3 \|Y - Y^k\|_F^2\\
&-\frac{1}{2} \langle \Lambda(Y), Y^\top Y - I \rangle + \frac{\theta}{4} \|Y^\top Y - I\|_F^2,
\label{32}
\end{aligned}
\end{equation}
with the penalty matrix
\begin{equation}
\Lambda(Y) = \frac{1}{2} \left( Y^\top \nabla l(Y) + \nabla l(Y)^\top Y \right),
\end{equation}
and $\theta > 0$ is the penalty parameter.
Here, $l(Y)$ is 
\begin{equation}
l(Y)= \|X^\top W - Y\|_F^2 - \alpha\, \text{Tr}(Y^\top \hat{S} Y) + \tau_3 \|Y - Y^k\|_F^2.
\end{equation}
To avoid computing the full Hessian of $l(Y)$, the gradient of $h(Y)$ is approximated by
\begin{equation}
\begin{aligned}
D(Y) = & -2(X^\top W - Y) - 2\alpha \hat{S} Y + 2\tau_3 (Y - Y^k) \\
& - Y \Lambda(Y) + \theta Y(Y^\top Y - I).
\end{aligned}  
\label{36}
\end{equation}
Then $Y$ is updated via a Barzilai-Borwein gradient step
\begin{equation}
\hat{Y}^{k+1} = Y^k - \xi_k D(Y^k),
\label{37}
\end{equation}
where $\xi_k > 0$ is the step size. To ensure feasibility, $\hat{Y}^{k+1}$ is projected back to $\mathcal{B}_\rho$ if necessary
\begin{equation}
Y^{k+1} = \mathcal{P}_{\mathcal{B}_\rho}(\hat{Y}^{k+1}),
\end{equation}
where $\mathcal{P}_{\mathcal{B}_\rho}$ denotes the projection onto the Frobenius norm ball. If $\hat{Y}^{k+1}$ already satisfies the constraint, the projection returns it unchanged. Otherwise, it is scaled appropriately to lie on the boundary of $\mathcal{B}_\rho$. The complete update process for $Y$ is summarized in Algorithm \ref{algorithm 2}.

\begin{algorithm}[t]
  \caption{Exact penalty function method of solving \eqref{31}}
  \label{algorithm 2}
  \textbf{Input:} Data $X$,  parameters $\theta,  \rho,  \xi$ \\
  \textbf{Initialize:} $Y^0 \leftarrow Y^k$\\
  \textbf{While} not converged \textbf{do}
  \begin{enumerate}
      \item Compute $D(Y^k)$ by  \eqref{36}
      \item Compute $\hat{Y}^{k+1}$ by  \eqref{37}
      \item \textbf{If} $\|\hat{Y}^{k+1}\|_F > \rho$ 
      \textbf{then}
      \begin{equation}
      Y^{k+1} = \frac{\rho}{\|\hat{Y}^{k+1}\|_F} \hat{Y}^{k+1}
      \label{eq:22}
      \end{equation}
      \textbf{else}
      \begin{equation}
      Y^{k+1} = \hat{Y}^{k+1}
      \label{eq:23}
      \end{equation}
      \textbf{end if}
      \item Check convergence
  \end{enumerate}
  \textbf{End While}\\
  \textbf{Output:} $Y^{k+1}$
\end{algorithm}

\subsection{Convergence Analysis}
For convenience, denote $Q = (P,  W,  Y)$. Then define
\begin{equation}
\begin{aligned}
\label{eq:Q}
f(Q) = &\min_{P, Y, W} \quad\|X^\top  W-Y\|_F^2 +\lambda\|W\|_F^2 -\alpha\textrm{Tr}(Y^\top \hat S Y)\\
&+\beta\sum_{i, j=1}^n\|W^\top x_i -W^\top x_j\|_2^2P_{ij}+\mu\|P\|_F^2.\\
\end{aligned}
\end{equation}
Before proceeding, we first prove the sufficient decrease property in the following lemma.

\begin{lemma}
  \label{lemma3}
Assume that $\{Q^k\}_{k \in \mathbb{N}}$ is generated by Algorithm $\ref{algorithm 1}$. Then the following inequality holds
\begin{equation}
f(Q^{k+1}) + \tau \|Q^{k+1} - Q^k\|_F^2 \leq f(Q^k), 
\end{equation}
where $\tau =  \min\{\tau_1,  \tau_2,  \tau_3\}$.
\end{lemma}

\begin{proof}
Since $P^{k+1}, W^{k+1},  $ and $Y^{k+1}$ are optimal solutions of $\eqref{1}$,  $\eqref{2}$,  and $\eqref{3}$,  respectively, then
\begin{equation}
  \label{45}
\begin{aligned}
f(P^{k+1},  W^k,  Y^k) + \tau_1 \|P^{k+1} - P^k\|_F^2\\
\leq f(P^k,  W^k,  Y^k), 
\end{aligned}
\end{equation}
\label{46}
\begin{equation}
  \begin{aligned}
f(P^{k+1},  W^{k+1},  P^k) + \tau_2 \|W^{k+1} - W^k\|_F^2 \\
\leq f(P^{k+1},  W^k,  Y^k), 
\end{aligned}
\end{equation}
\begin{equation}
  \label{47}
  \begin{aligned}
f(P^{k+1},  W^{k+1},  Y^{k+1}) + \tau_3 \|Y^{k+1} - Y^k\|_F^2 \\
\leq f(P^{k+1},  W^{k+1},  Y^k).
\end{aligned}
\end{equation}
By combining \eqref{45}-\eqref{47}, it holds
\begin{equation}
\begin{aligned}
&f(P^{k+1},  W^{k+1},  Y^{k+1}) + \tau_1 \|P^{k+1} - P^k\|_F^2 \\
&+ \tau_2 \|W^{k+1} - W^k\|_F^2 + \tau_3 \|Y^{k+1} - Y^k\|_F^2 \\
&\leq f(P^k,  W^k,  Y^k), 
\end{aligned}
\end{equation}
which yields
\begin{equation}
f(Q^{k+1}) + \tau \|Q^{k+1} - Q^k\|_F^2 \leq f(Q^k).
\end{equation}
Hence, the proof is completed.
\end{proof}

Leveraging the relevant arguments in \cite{bolte2014proximal,123} and the Kurdyka-Lojasiewicz property, we have proven the global convergence of Algorithm \ref{algorithm 1}.
\begin{theorem}
  Assume that \(\{Q^k\}_{k \in \mathbb{N}}\) is generated by Algorithm \ref{algorithm 1}.
  Then,  the sequence \(\{Q^k\}_{k \in \mathbb{N}}\) globally converges to a critical point of \(f(Q)\),  i.e., 
\begin{equation}
 0 \in \partial f(Q^*), 
\end{equation}
where \(\partial f(\cdot)\) represents the limiting subdifferential set.
\end{theorem}

For the sake of continuity, the proof is omitted here. In fact, this theorem implies that the generated sequence converges to a point that satisfies the first-order optimality condition of \(f(Q)\).

\subsection{Complexity Analysis}
In Algorithm \ref{algorithm 1}, the most computationally expensive step is optimizing the pseudo-label matrix \(Y\).
Using the exact penalty function method,  eigenvalue decomposition is avoided and requires \(\mathcal{O}(n^2c)\).
Updating \(W\) involves computing
\(
\left( X X^\top + (\lambda + \tau_2) I + \beta X L_P X^\top \right)^{-1} (X Y + \tau_2 W), 
\)
that takes \(\mathcal{O}(d^3)\).
Updating \(P\) requires calculating \(V_{ij}\) in \(\mathcal{O}(cdn)\).
Thus, the total complexity for processing high-dimensional data is \(\max\{\mathcal{O}(n^2c),  \mathcal{O}(d^3),  \mathcal{O}(cdn)\}\), which is not worse than the traditional embedded feature selection methods \cite{nie2016unsupervised, li2017reconstruction}.

\begin{figure*}[t]
  \makeatletter
  \renewcommand{\@thesubfigure}{\hskip\subfiglabelskip}
  \makeatother
  \centering

  \subfigure[(a) LapScore]{
    \includegraphics[width=1.35 in]{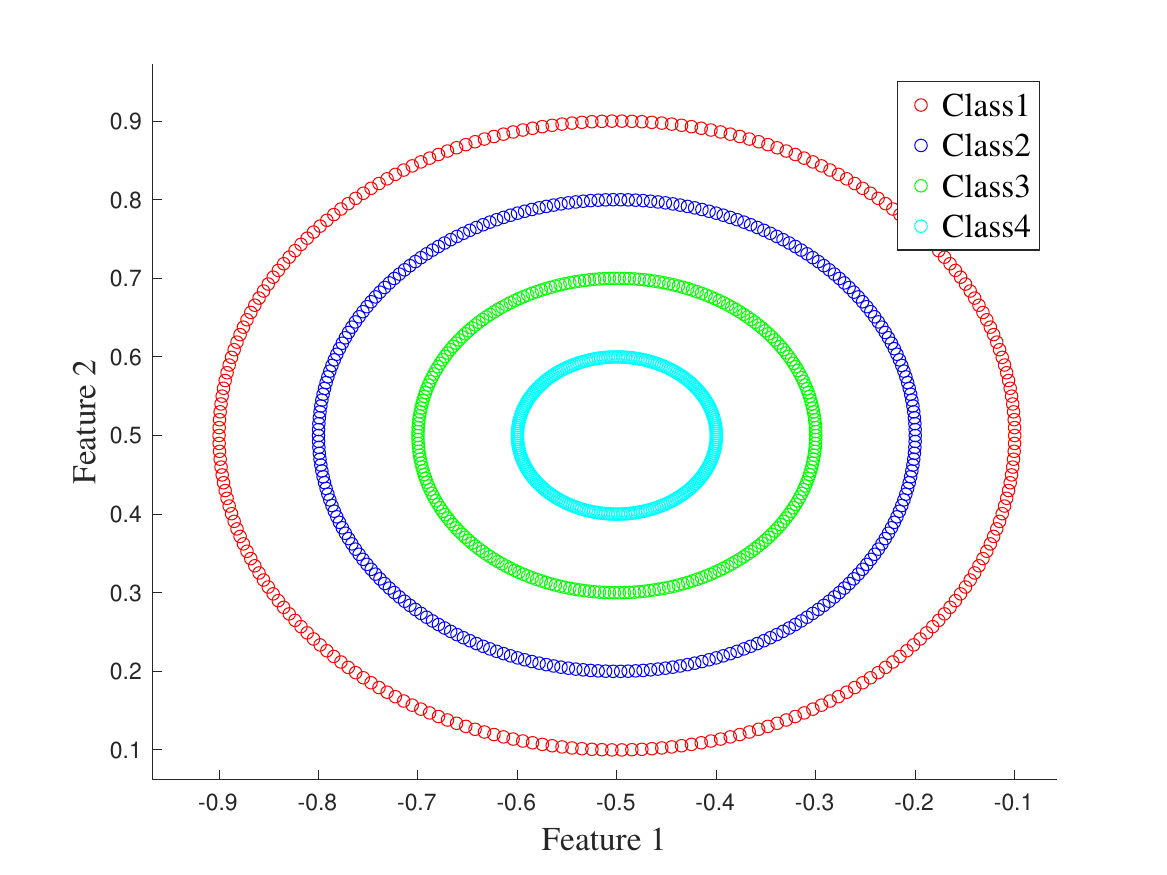}}
  \subfigure[(b) MCFS]{\label{tsne3}
    \includegraphics[width=1.35 in]{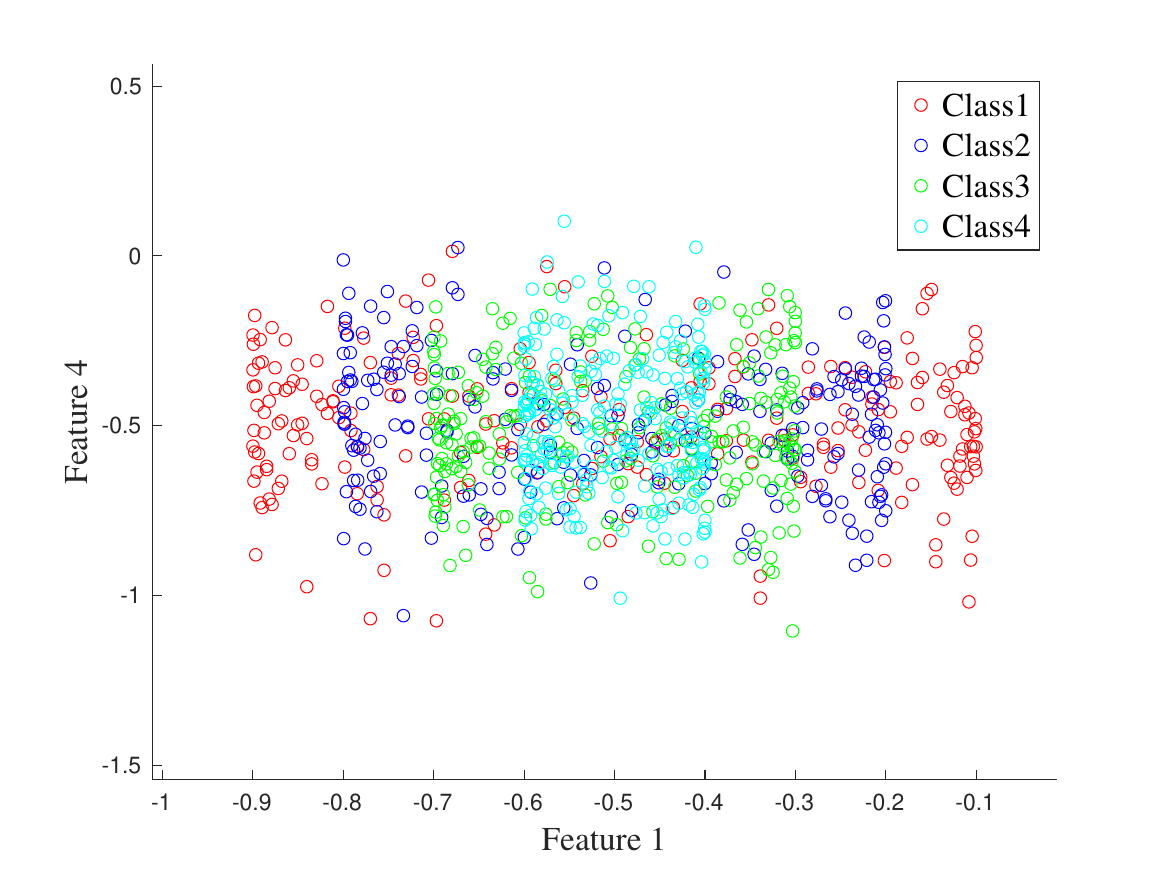}}
   \subfigure[(e) UDFS]{
    \includegraphics[width=1.35 in]{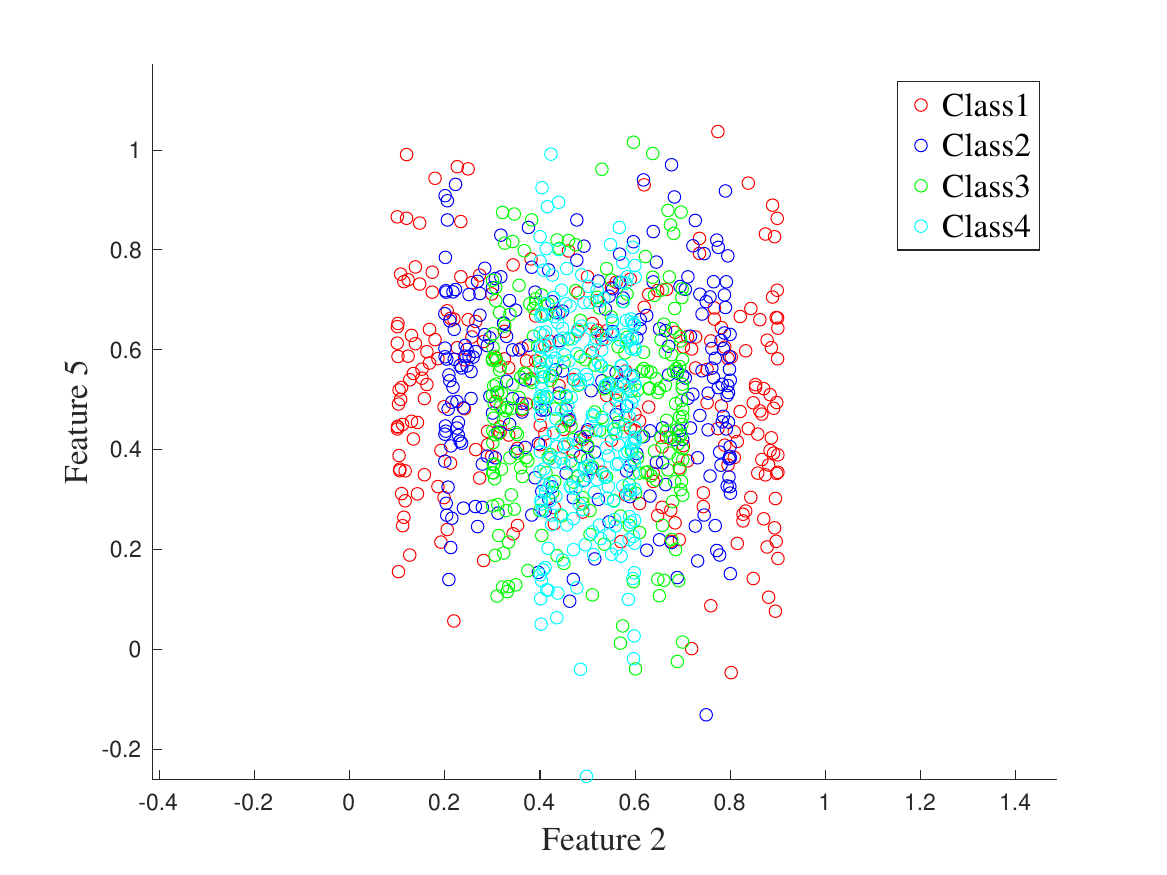}}
  \subfigure[(d) RNE]{
    \includegraphics[width=1.35 in]{figures/Synthetic/dartboard1-LapScore.pdf}}
     \subfigure[(c) SOGFS]{
    \includegraphics[width=1.35 in]{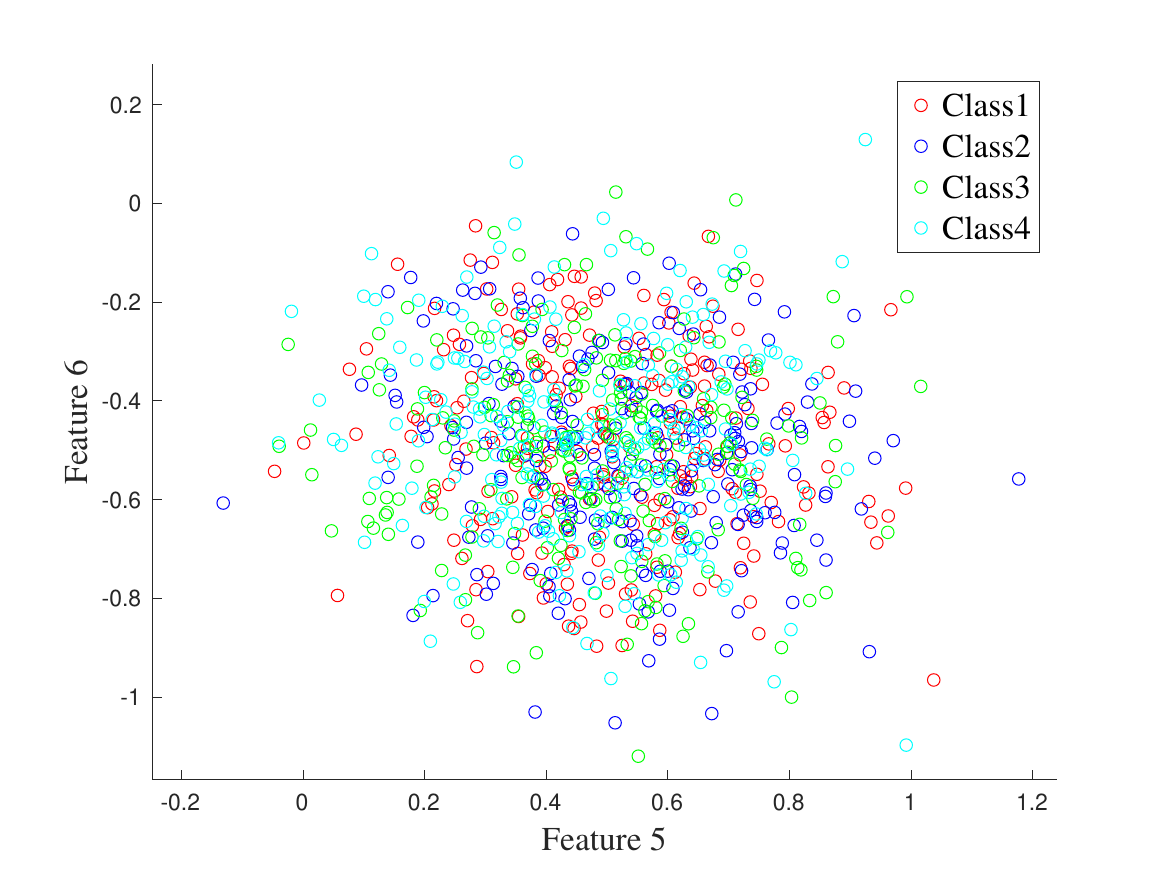}}
  \subfigure[(f) SPCAFS]{
    \includegraphics[width=1.35 in]{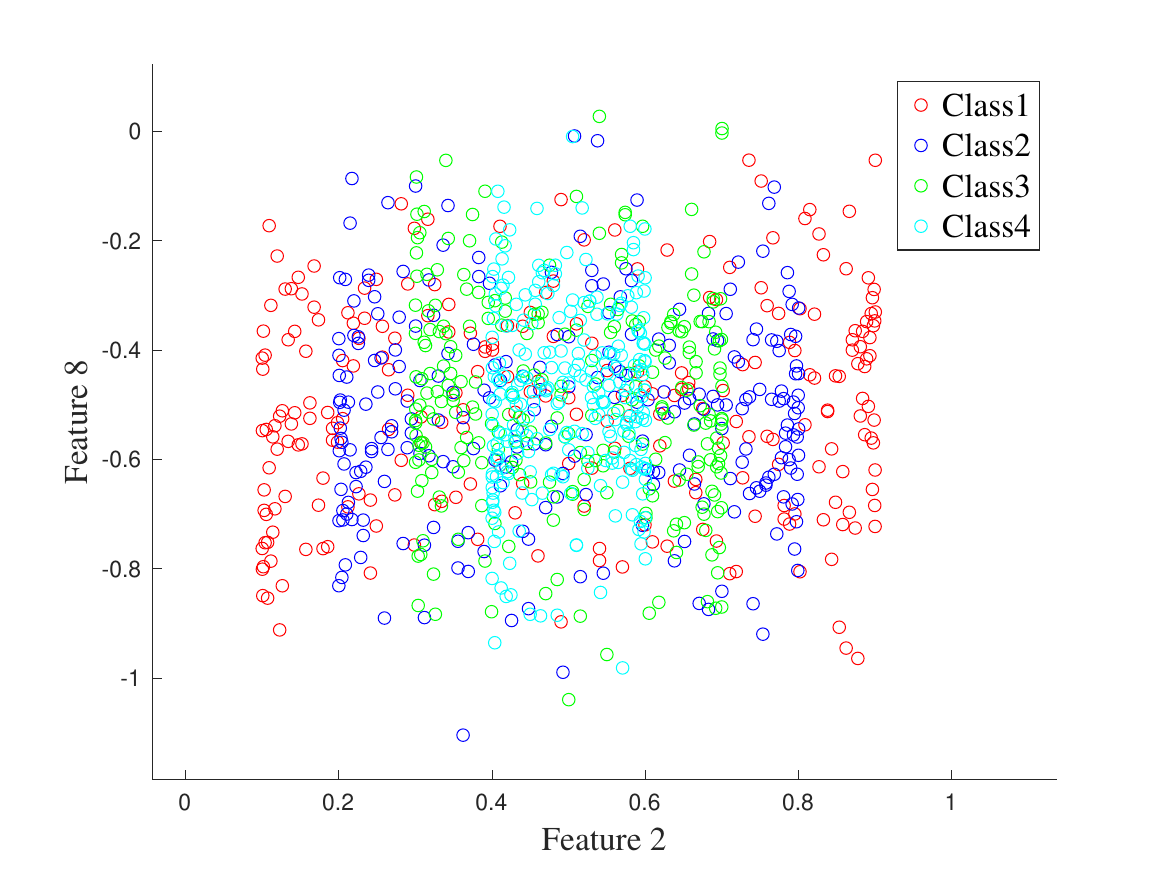}}
  \subfigure[(g) FSPCA]{
    \includegraphics[width=1.35 in]{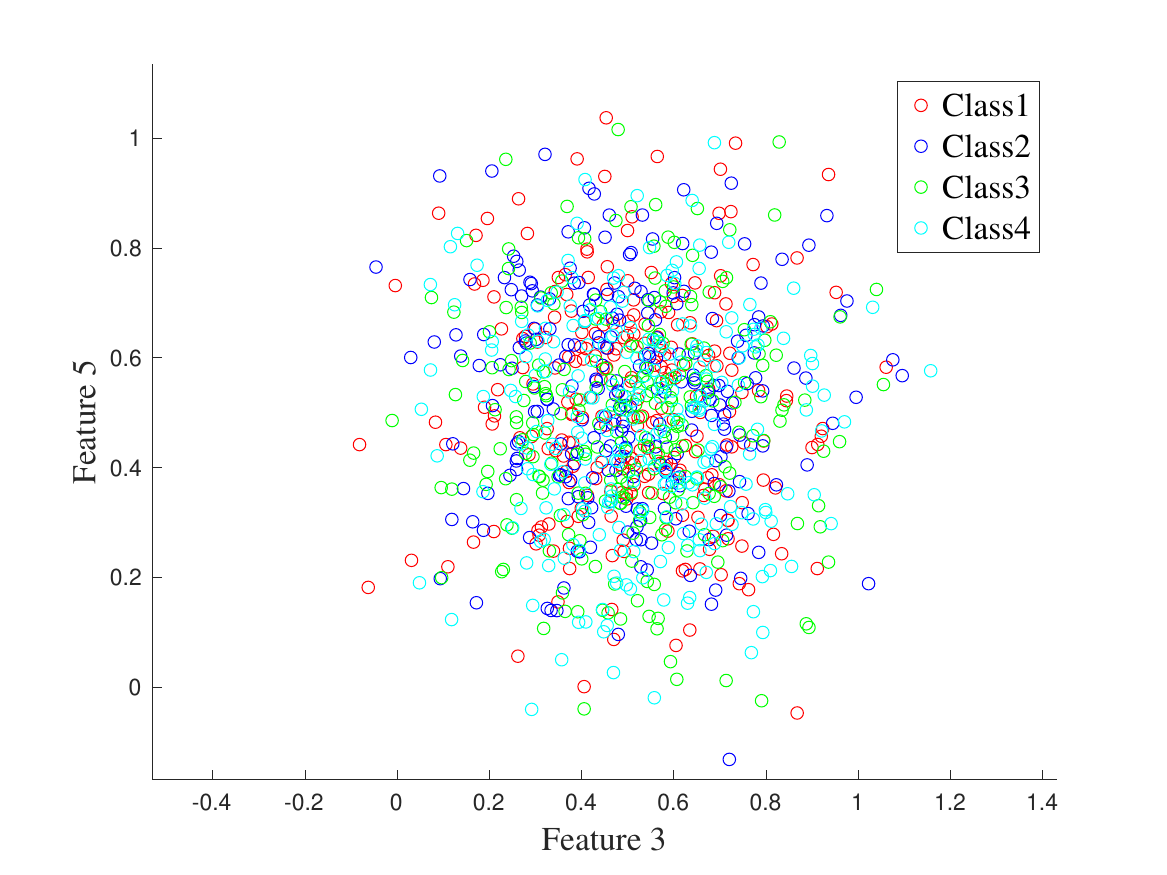}}
  \subfigure[(h) BLSFE]{
    \includegraphics[width=1.35 in]{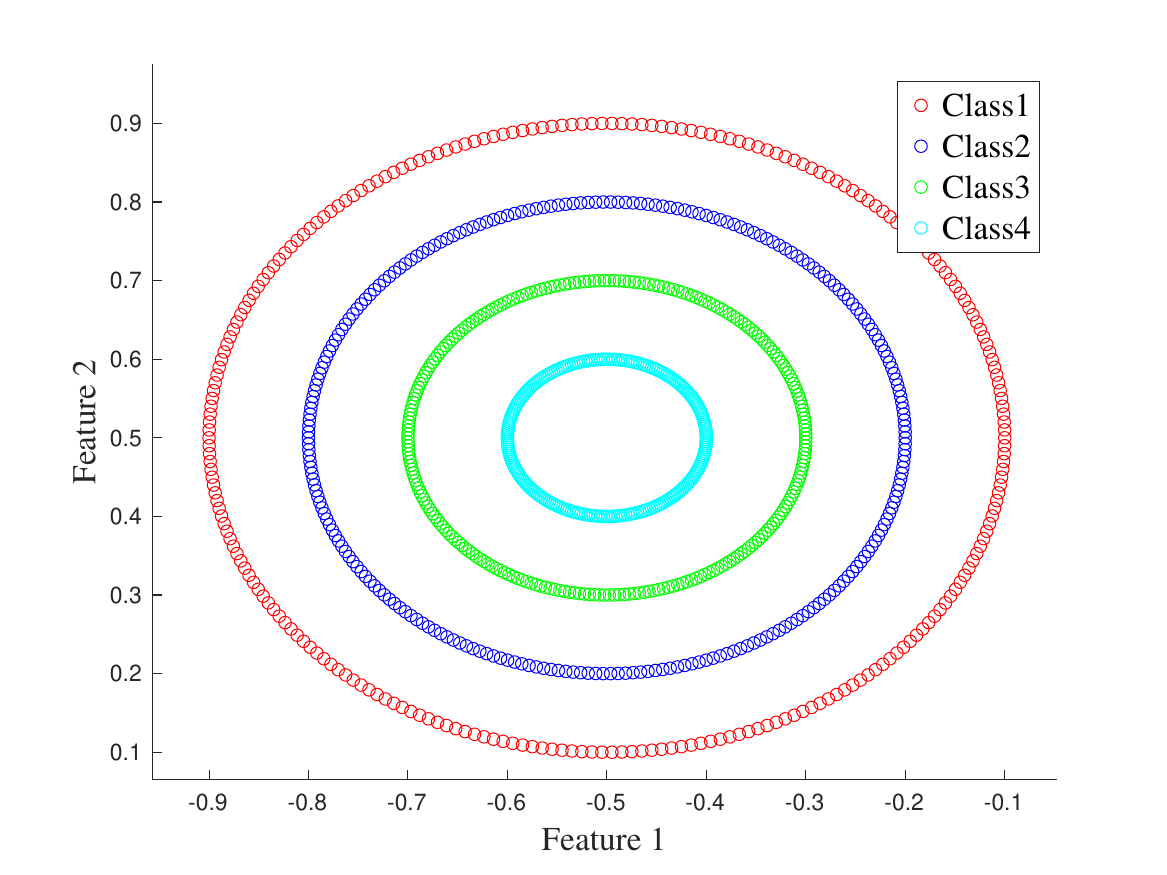}}
    \subfigure[(i) BLSFS]{
    \includegraphics[width=1.35 in]{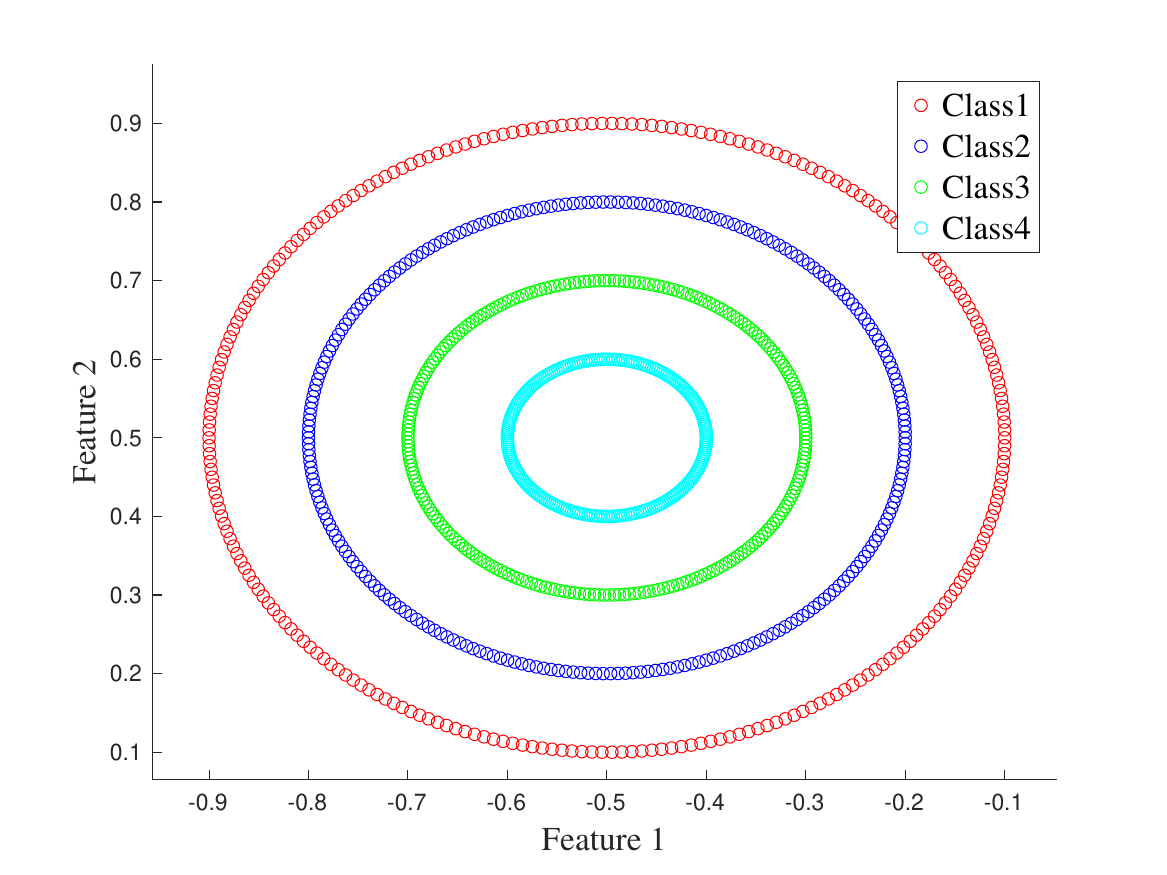}}
  \subfigure[(j) BLUFS]{
    \includegraphics[width=1.35 in]{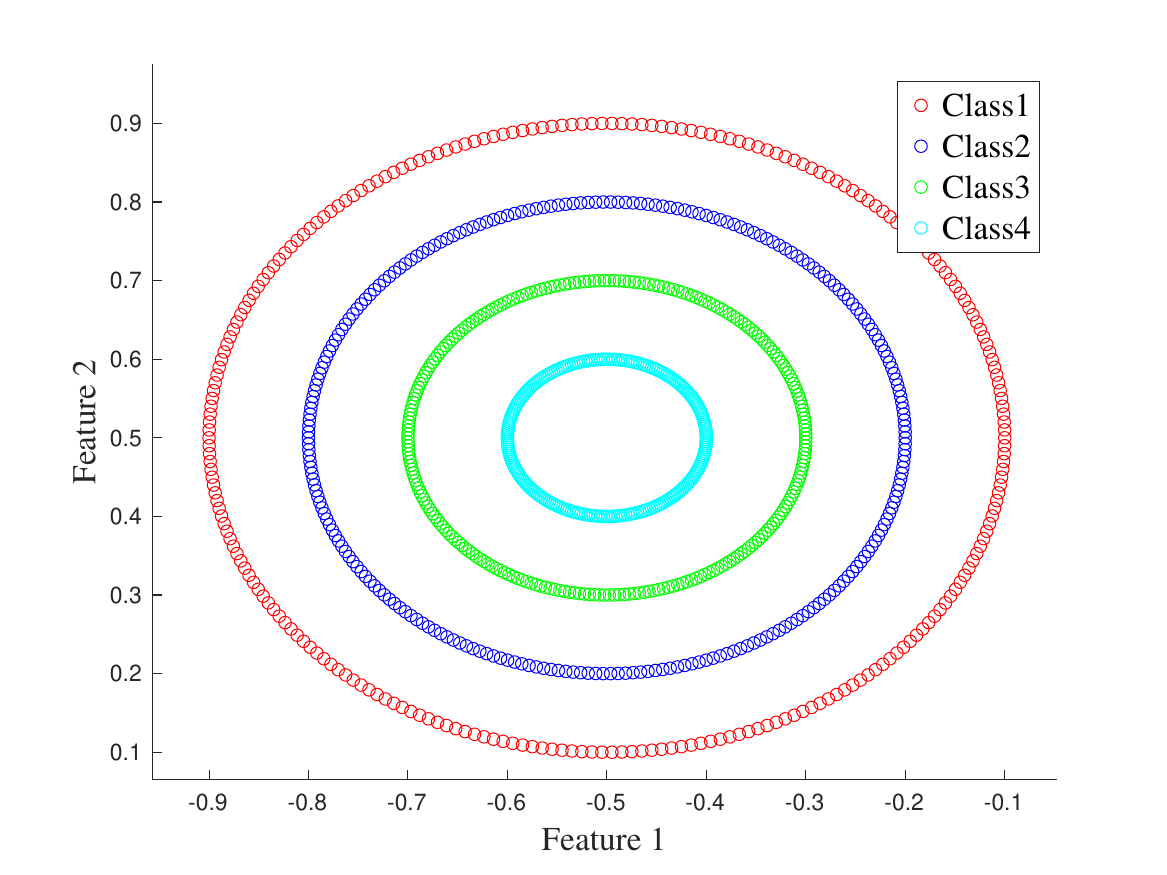}}
  \vskip-0.2cm
  \caption{Visual comparisons on the Dartboard1 dataset,  where (a)-(j) are the feature selection results.}
  \label{figure:1}
\end{figure*}

\begin{figure*}[t]
  \makeatletter
  \renewcommand{\@thesubfigure}{\hskip\subfiglabelskip}
  \makeatother
  \centering

  \subfigure[(a) LapScore]{
    \includegraphics[width=1.35 in]{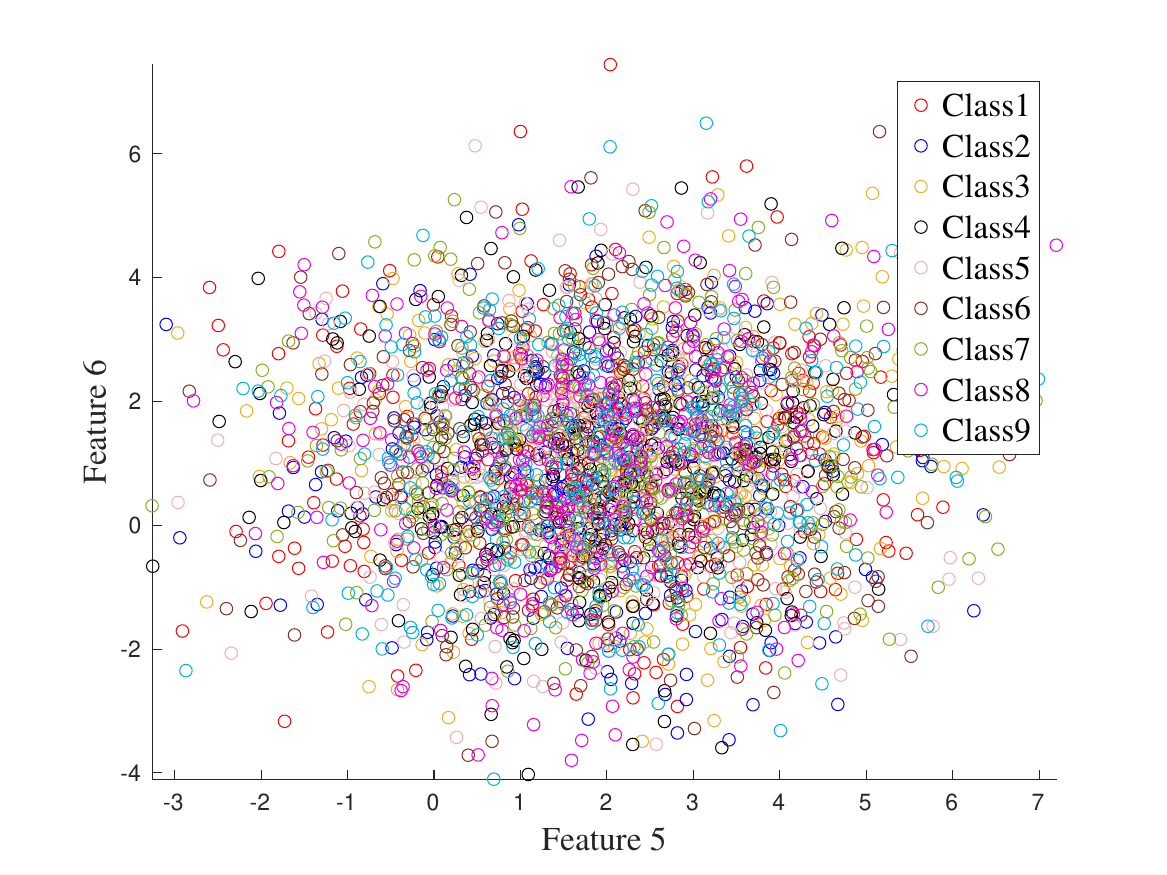}}
  \subfigure[(b) MCFS]{
    \includegraphics[width=1.35 in]{figures/Synthetic/dartboard1-MCFS.pdf}}
    \subfigure[(e) UDFS]{
    \includegraphics[width=1.35 in]{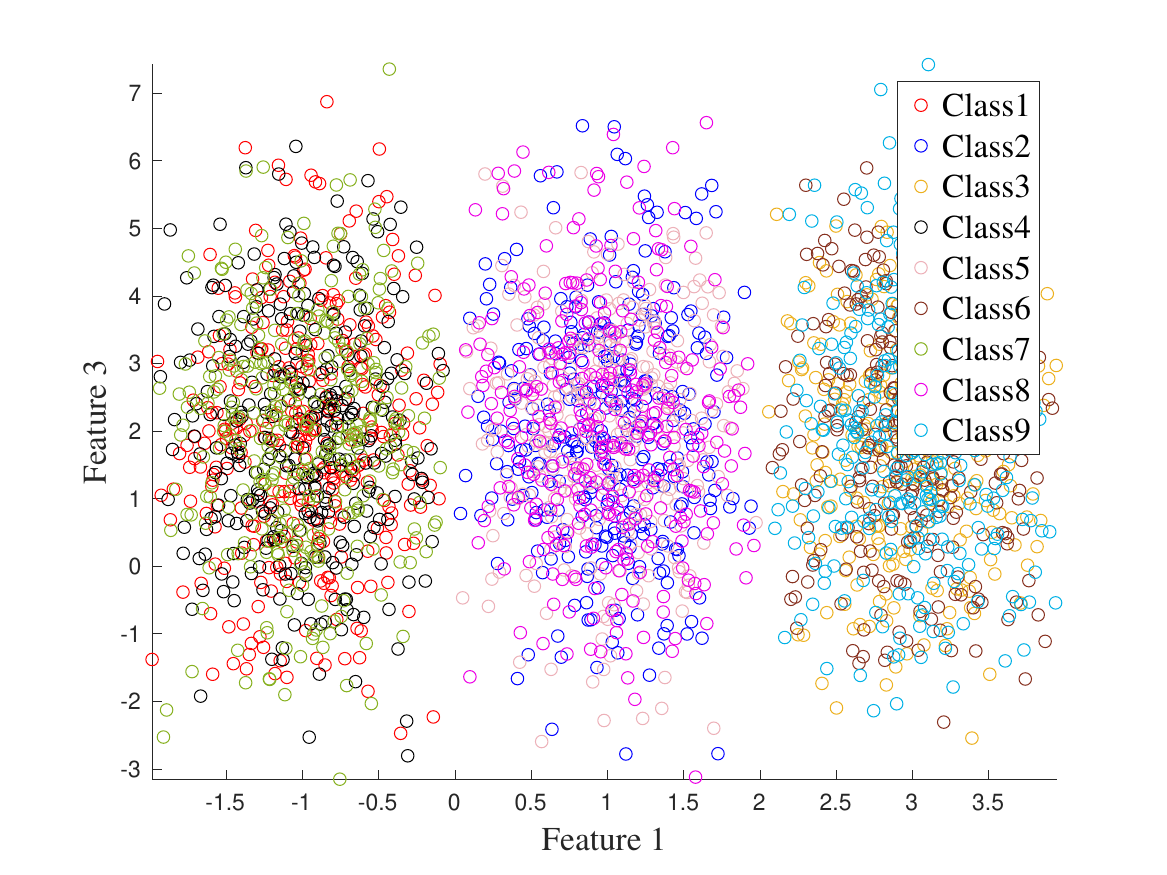}}
  \subfigure[(d) RNE]{
    \includegraphics[width=1.35 in]{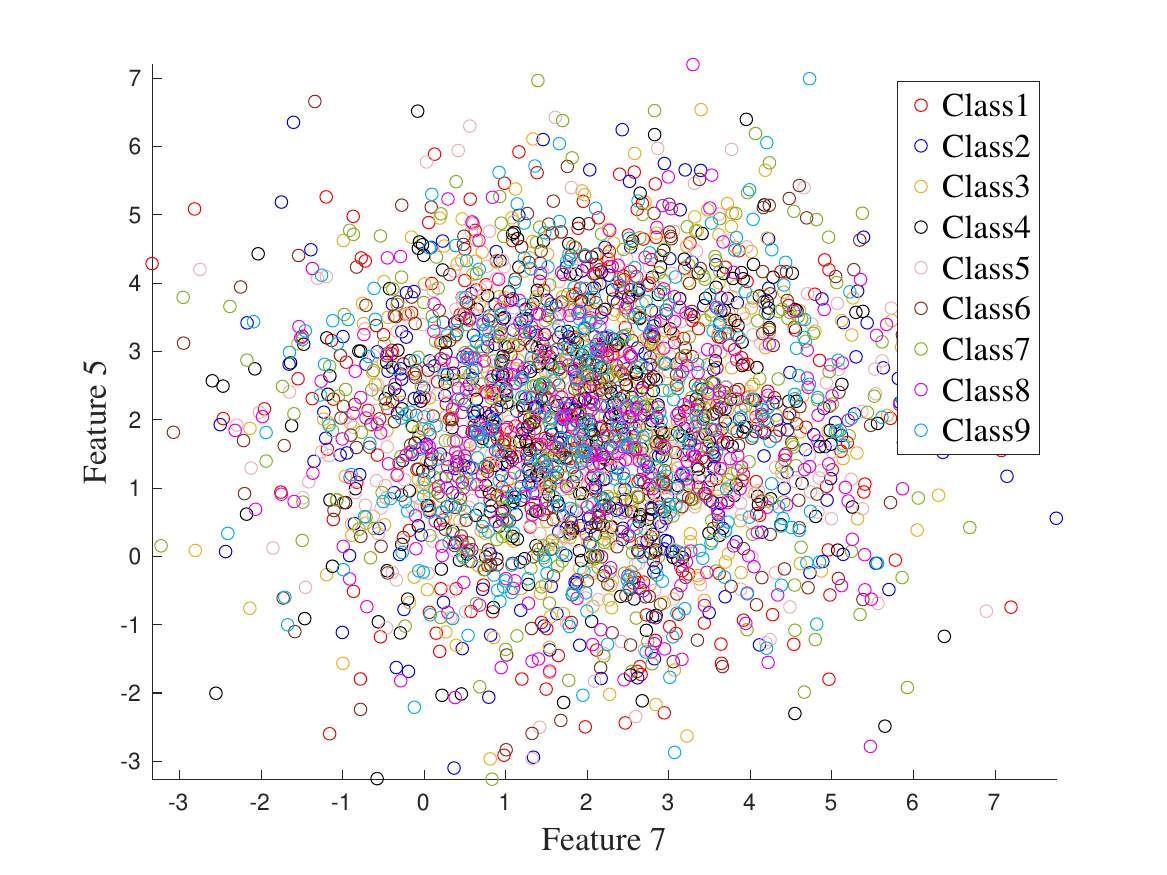}}
        \subfigure[(c) SOGFS]{
    \includegraphics[width=1.35 in]{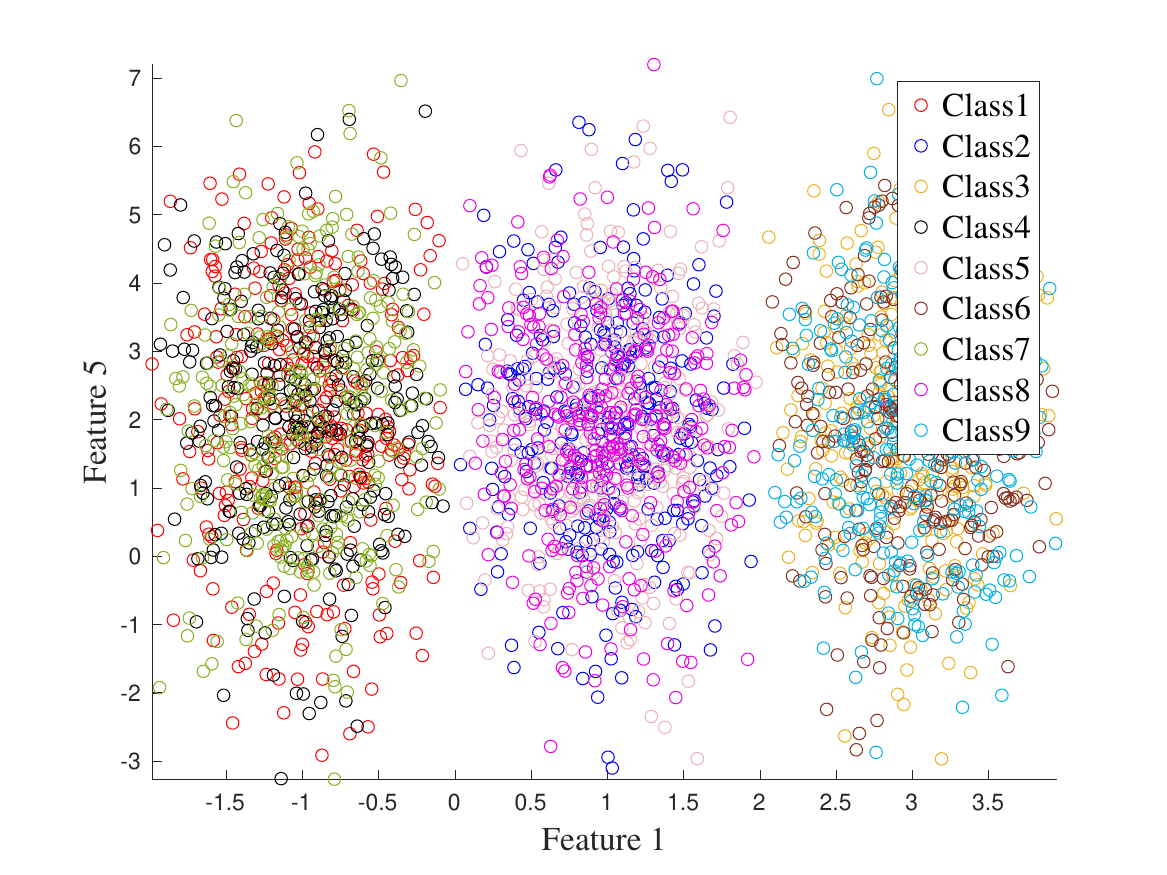}}
  \subfigure[(f) SPCAFS]{
    \includegraphics[width=1.35 in]{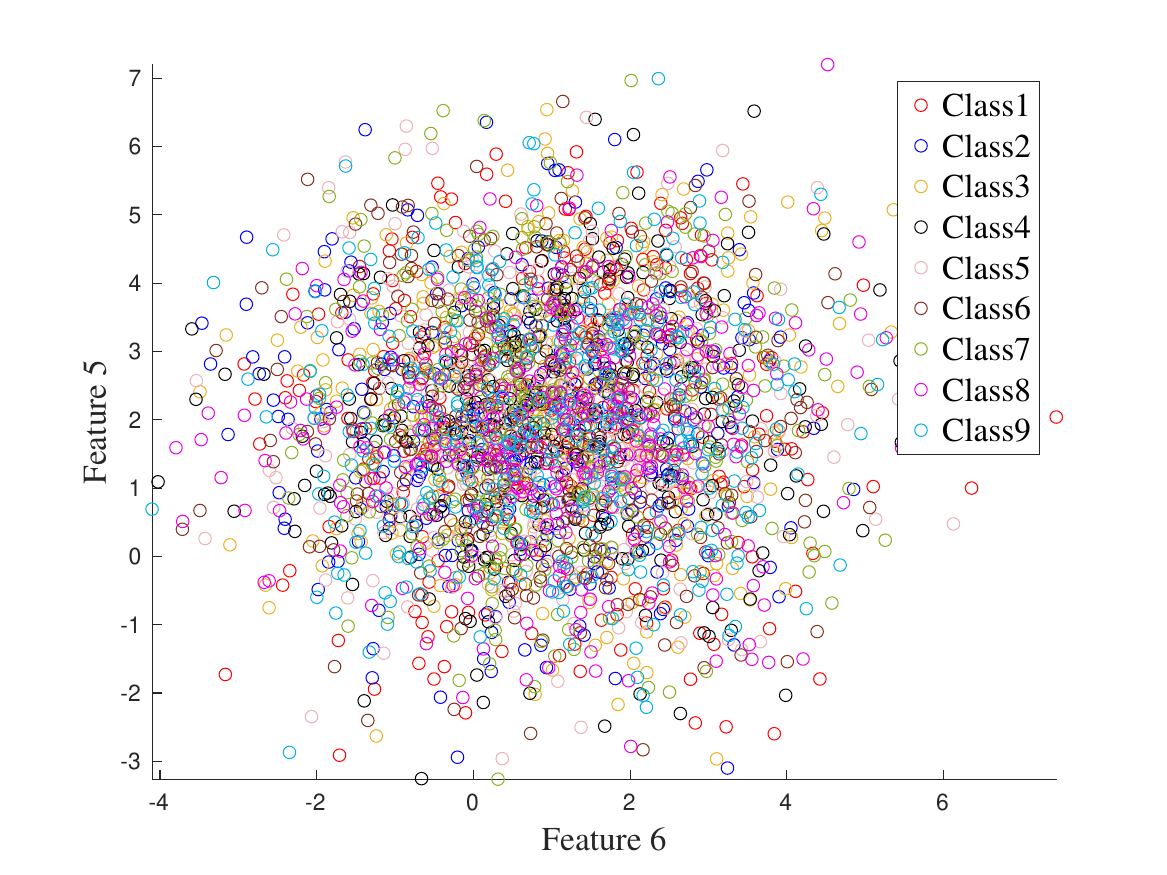}}
  \subfigure[(g) FSPCA]{
    \includegraphics[width=1.35 in]{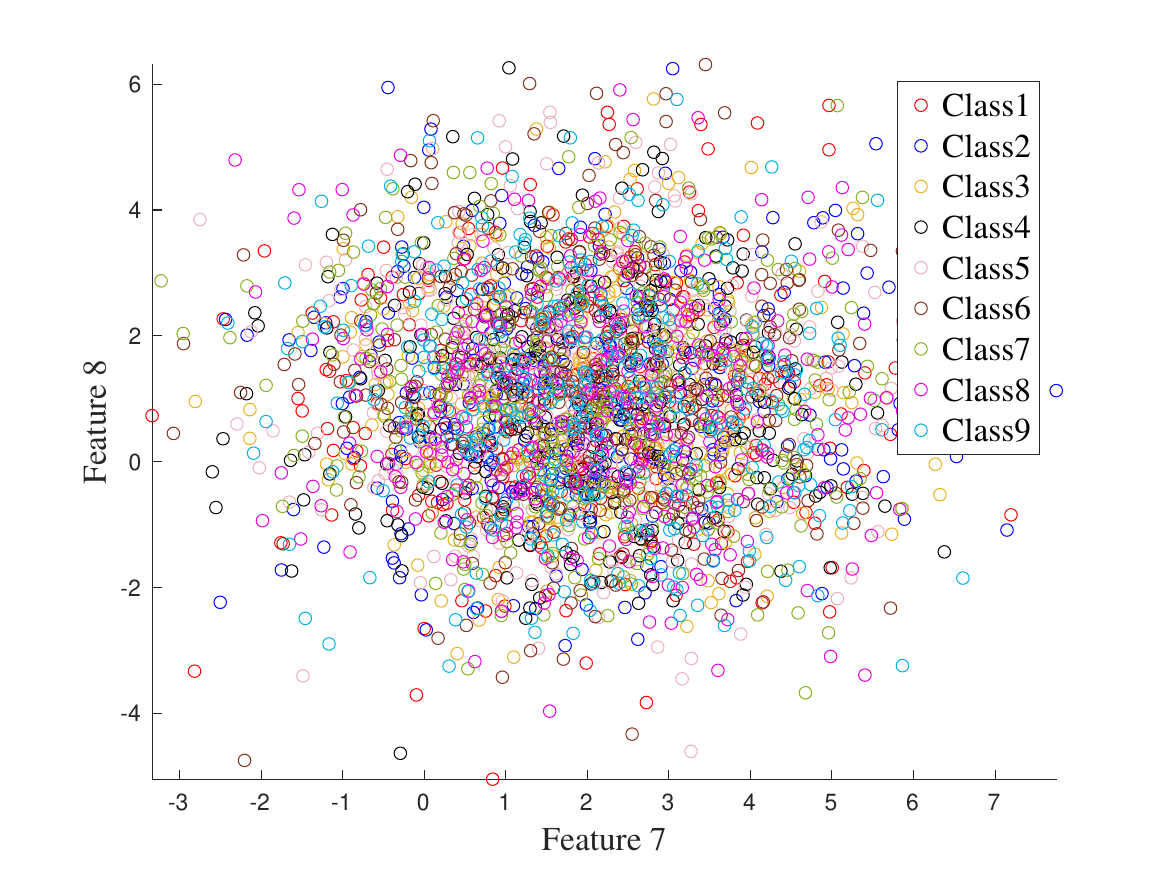}}
  \subfigure[(h) BLSFE]{
    \includegraphics[width=1.35 in]{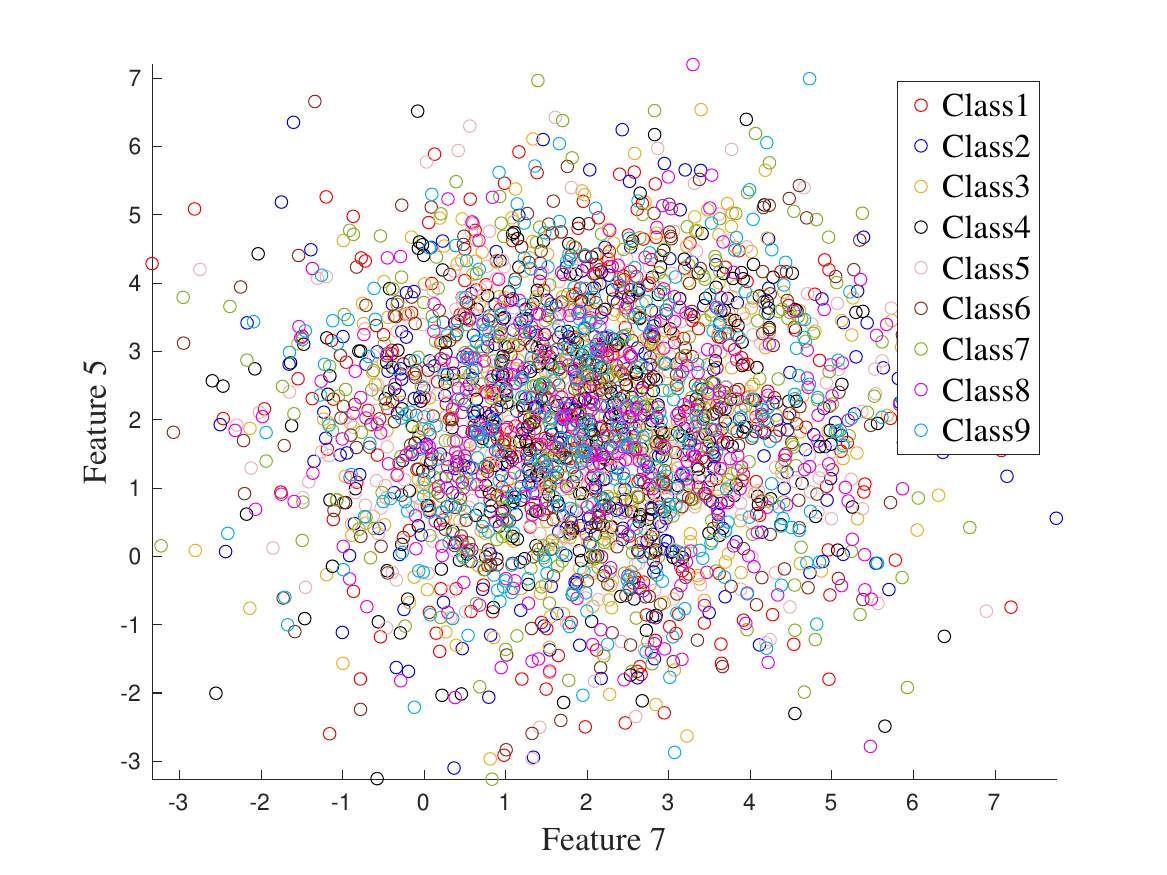}}
    \subfigure[(i) BLSFS]{
    \includegraphics[width=1.35 in]{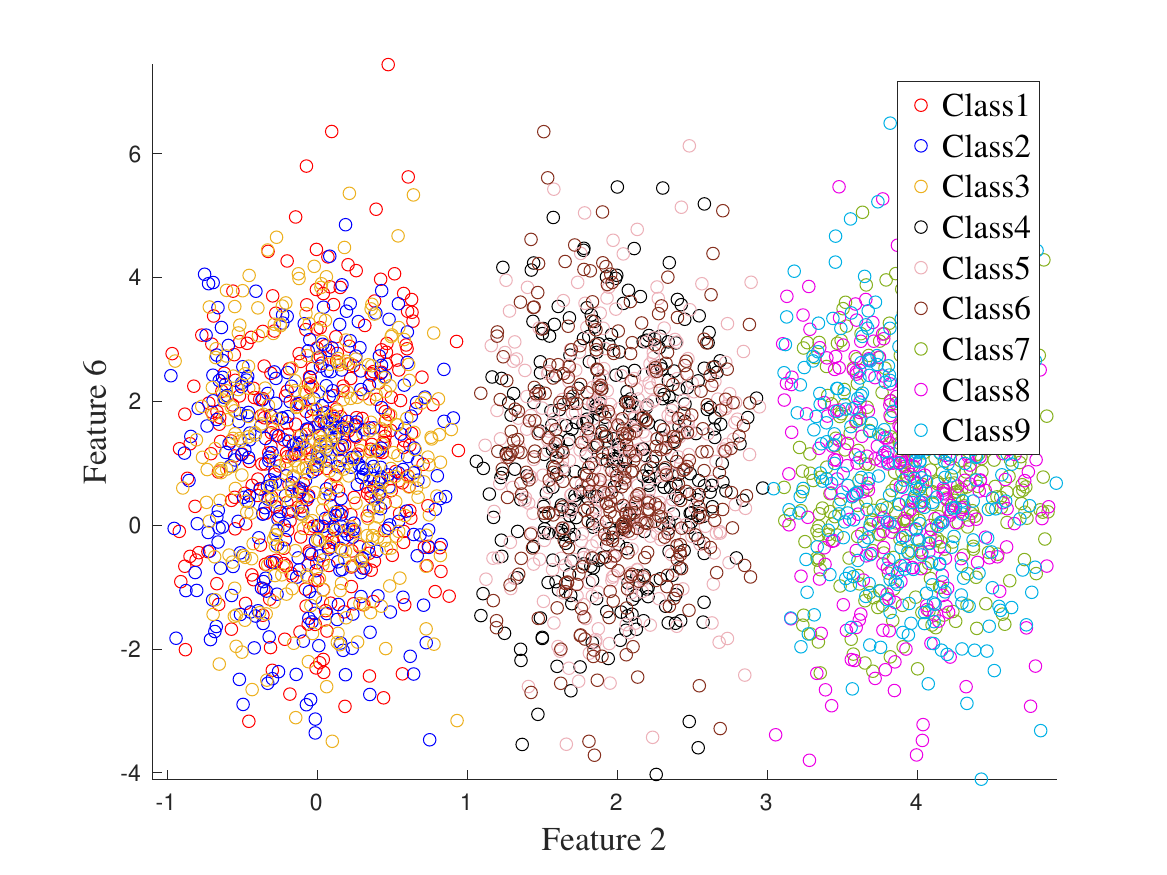}}
  \subfigure[(j) BLUFS]{
    \includegraphics[width=1.35 in]{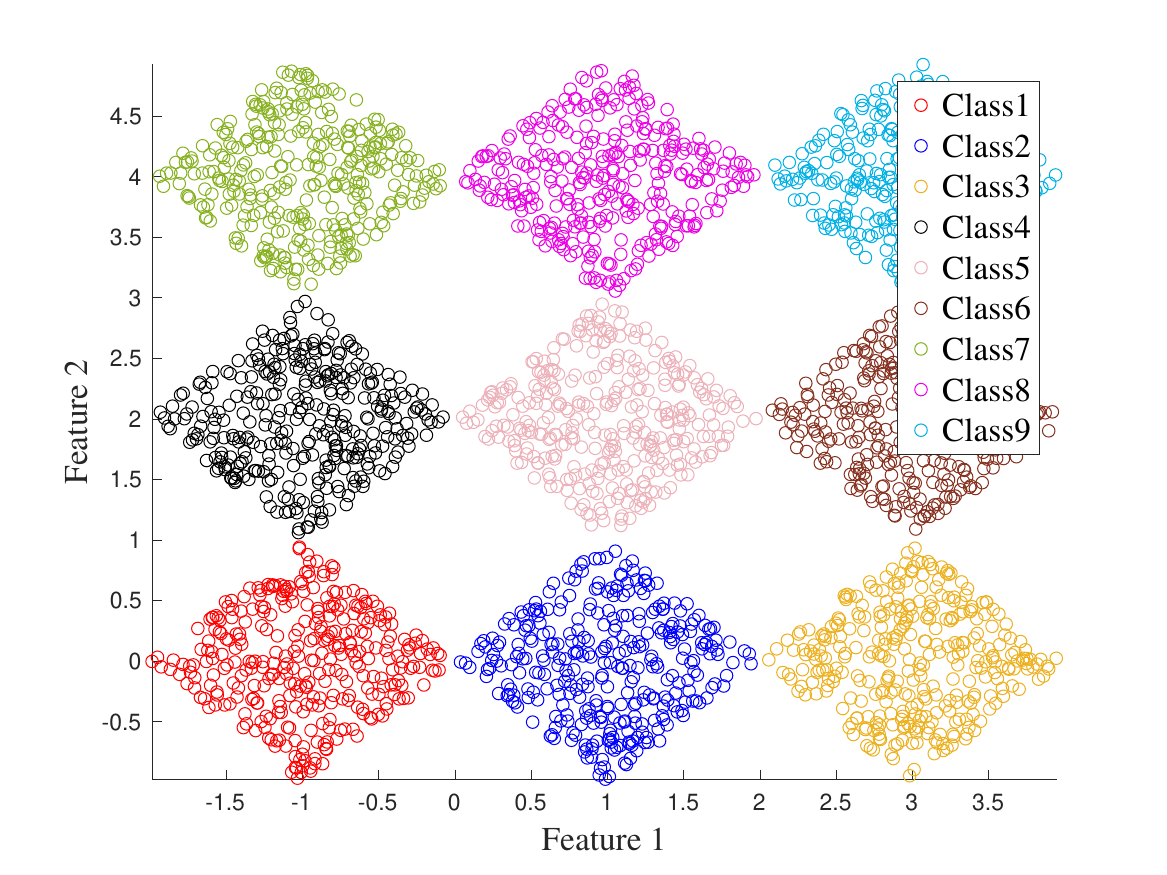}}
  \vskip-0.2cm
  \caption{Visual comparisons on the Banana dataset,  where (a)-(j) are the feature selection results.}
  \label{figure:2}
\end{figure*}

\section{Experiments}\label{Experiments}
This section verifies the superiority of BLUFS by comparing with nine UFS methods
on synthetic datasets\footnote{https://github.com/milaan9/Clustering-Datasets} and real-world datasets\footnote{https://jundongl.github.io/scikit-feature/datasets.html \label{web-data}}. Table \ref{tab:datasets} provides the detailed information of these real-world datasets.



\begin{table}[t]
\centering
\renewcommand\arraystretch{1.2}
\caption{Information of selected datasets.}
\label{tab:datasets}
\begin{tabular}{|c|c|c|c|c|}
\hline
 & ~~Datasets~~ & ~~Features~~ & ~~Samples~~ & ~~Classes~~ \\ \hline\hline
\multirow{2}{*}{Small}
&Isolet & 617 & 1560 & 26 \\ \cline{2-5}
&Jaffe & 676 & 213 & 10 \\ \cline{2-5}\hline
\multirow{3}{*}{Medium}
&pie & 1024 & 1166 & 53 \\ \cline{2-5}
&COIL20 & 1024 & 1440 & 20 \\\cline{2-5}
&MSTAR & 1024 & 2425 & 10 \\\cline{2-5}\hline
\multirow{3}{*}{Large}
& warpAR & 2400 & 130 & 10 \\ \cline{2-5}
&lung & 3312 & 203 & 4 \\ \cline{2-5}
& gisette & 5000 & 7000 & 2 \\ \hline
\end{tabular}
\label{table1}
\end{table}

\subsection{Setup}

\subsubsection{Compared Methods}
In the experiments, the baseline and nine benchmark UFS methods are compared, including 
\begin{enumerate}
  \item $\mathbf{Baseline}$,  which utilizes the original data for clustering.
  \item $\mathbf{LapScore}$ \cite{he2005laplacian},  which evaluates features based on their locality-preserving power.
  \item $\mathbf{MCFS}$ \cite{cai2010unsupervised}, which selects features based on the multi-cluster structure.
   \item $\mathbf{UDFS}$ \cite{yang2011},  which jointly exploits discriminative information and feature correlations for UFS using $\ell_{2, 1}$-norm regularization.
  \item $\mathbf{SOGFS}$ \cite{nie2016unsupervised}, which performs feature selection with structured graph optimization.
  \item $\mathbf{RNE}$ \cite{liu2020robust},  which elects features by using $\ell_{1}$-norm regularization to minimize the reconstruction error, preserving the local manifold structure.
   \item $\mathbf{SPCAFS}$ \cite{9580680},  which performs UFS using a sparse PCA using $\ell_{2,p}$-norm regularization to enhance feature discriminability.
  \item $\mathbf{FSPCA}$ \cite{nie2022learning}, which simultaneously achieves feature selection and PCA through two innovative algorithms.
  \item $\mathbf{BLSFE}$ \cite{zhou2023bi},  which enhances feature selection performance through integration at both the feature and clustering levels,  significantly improving clustering results.
  \item $\mathbf{BLSFS}$ \cite{hu2024bi}, which improves the performance of feature selection through collaborative optimization at both classification and feature levels.
\end{enumerate}

\begin{figure*}[t]
	\makeatletter
	\renewcommand{\@thesubfigure}{\hskip\subfiglabelskip}
	\makeatother
	\centering
    \subfigure[(c) Isolet]{
		\includegraphics[width=1.7 in]{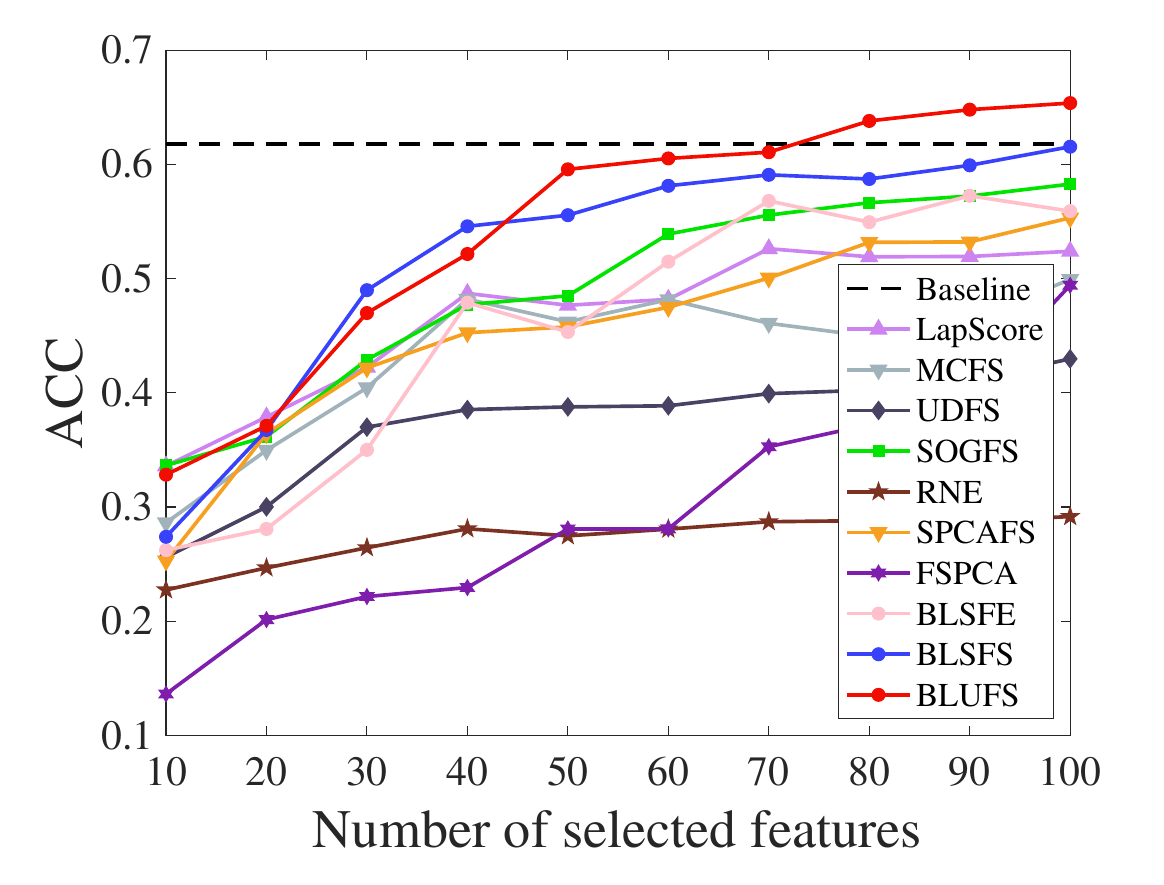}}
        \subfigure[(b) Jaffe]{
		\includegraphics[width=1.7 in]{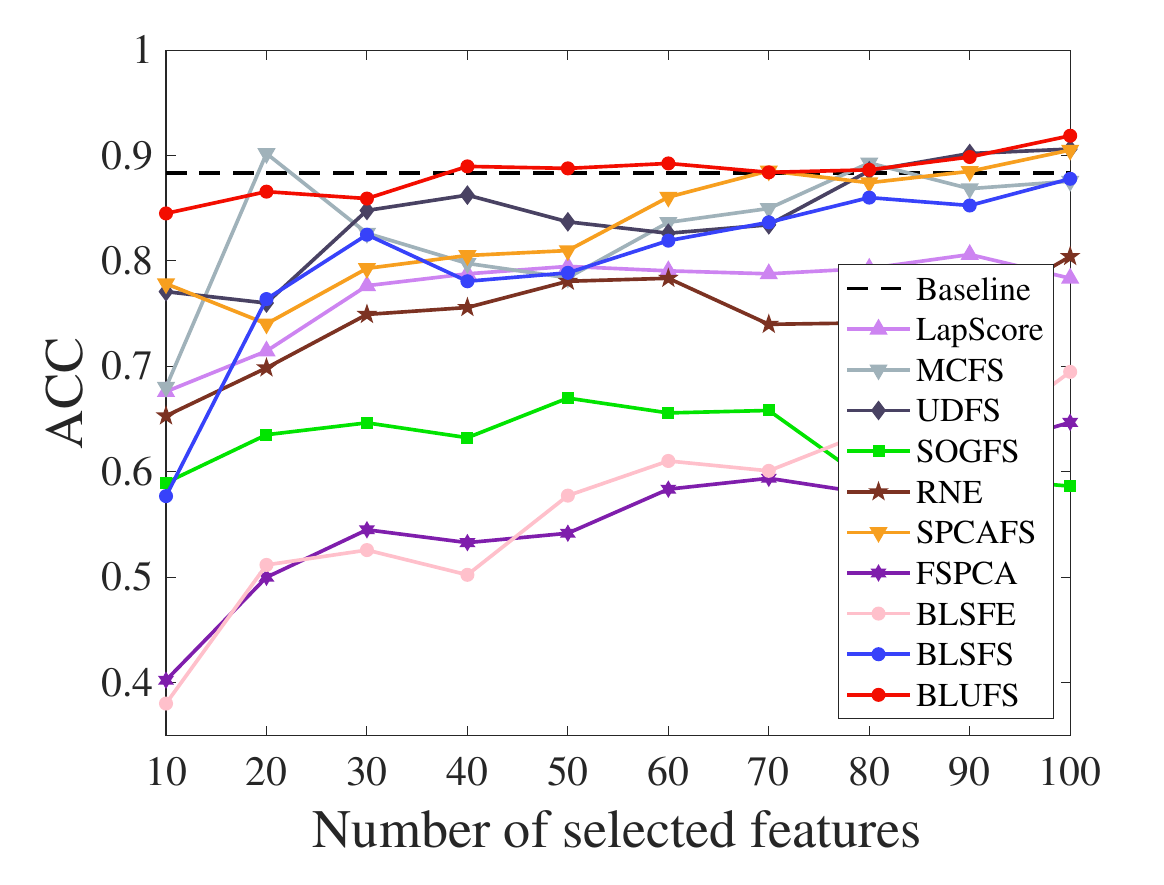}}
         \subfigure[(f) pie]{
		\includegraphics[width=1.7 in]{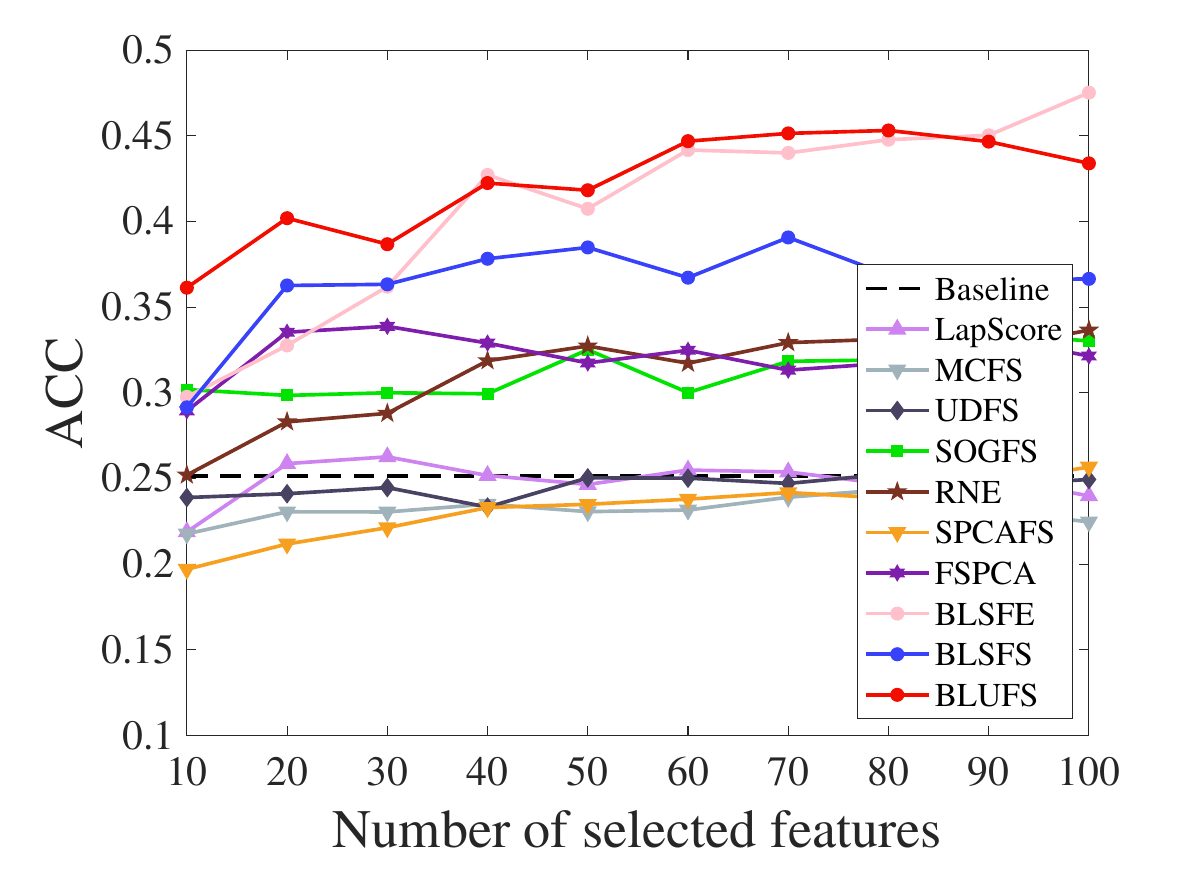}}
	\subfigure[(a) COIL20]{
		\includegraphics[width=1.7 in]{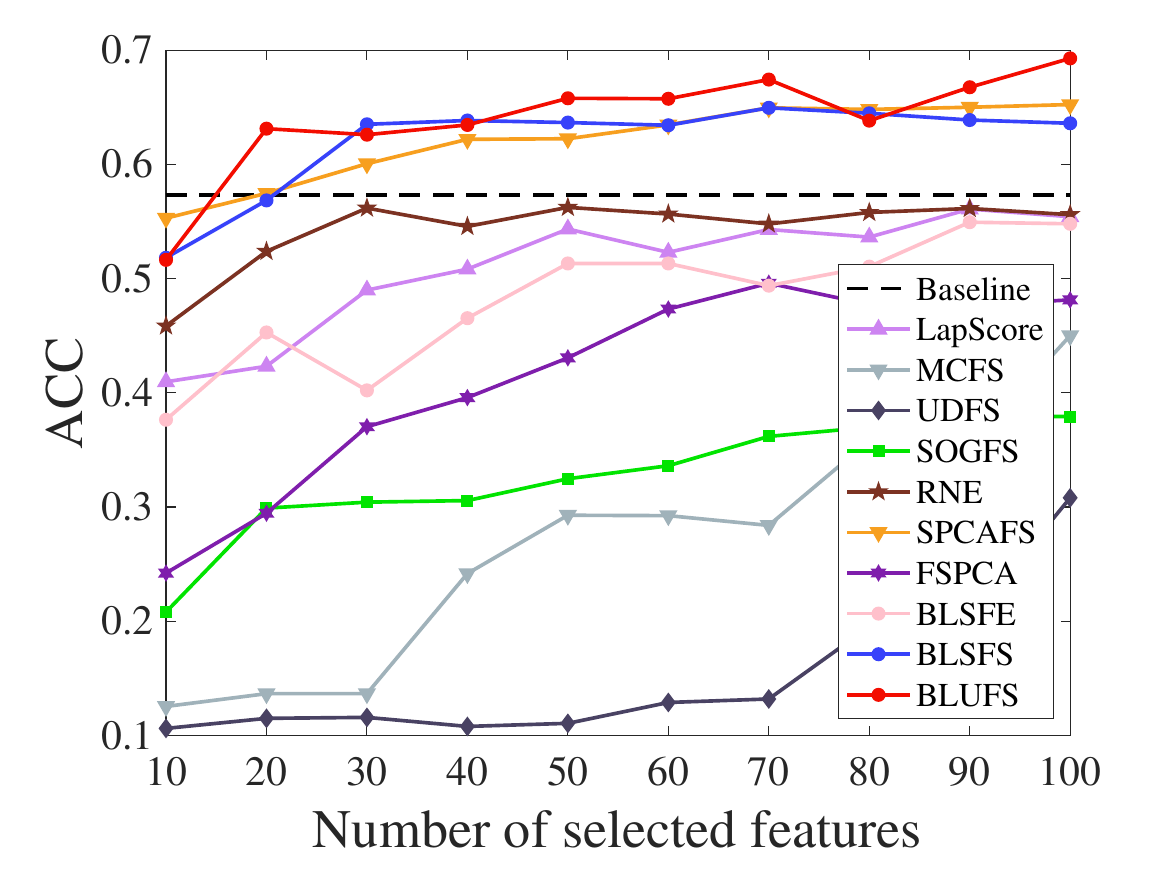}}
	\subfigure[(h) MSTAR]{
		\includegraphics[width=1.7 in]{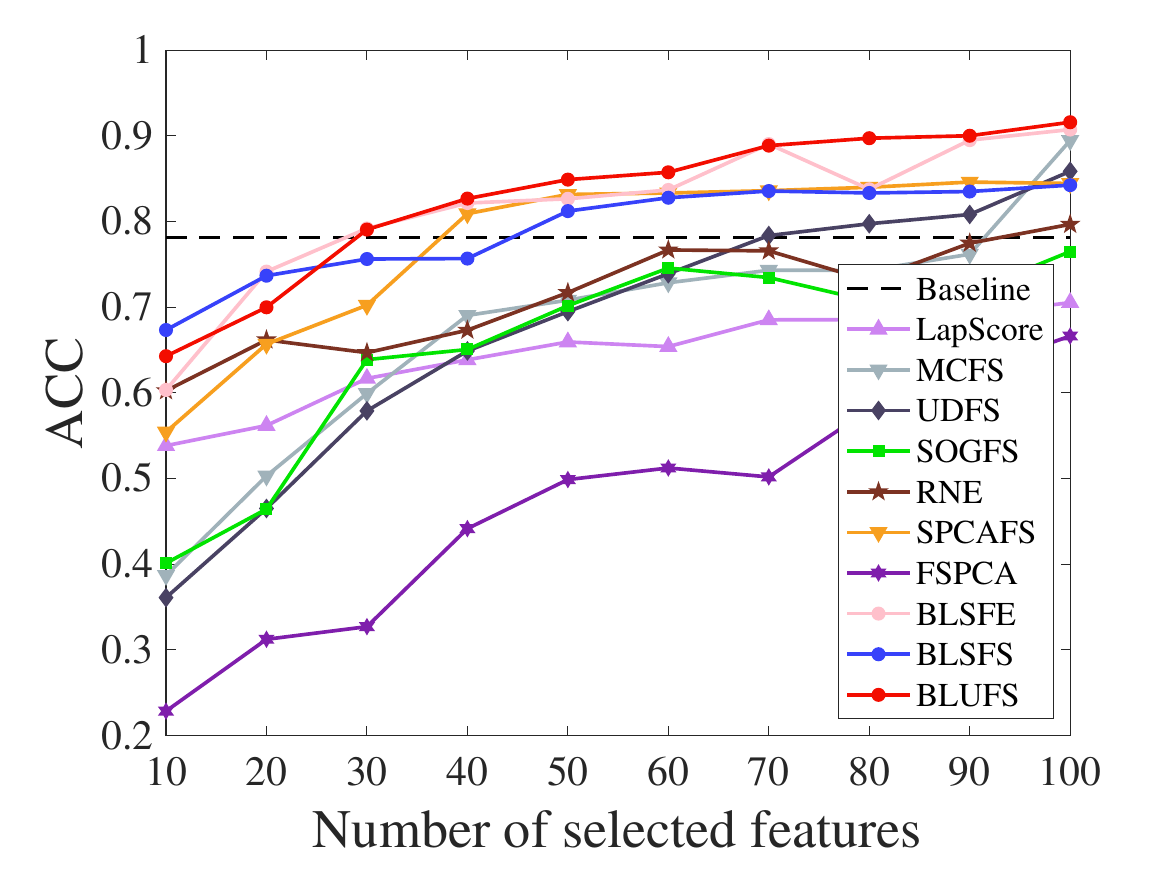}}
    \subfigure[(d) warpAR]{
		\includegraphics[width=1.7 in]{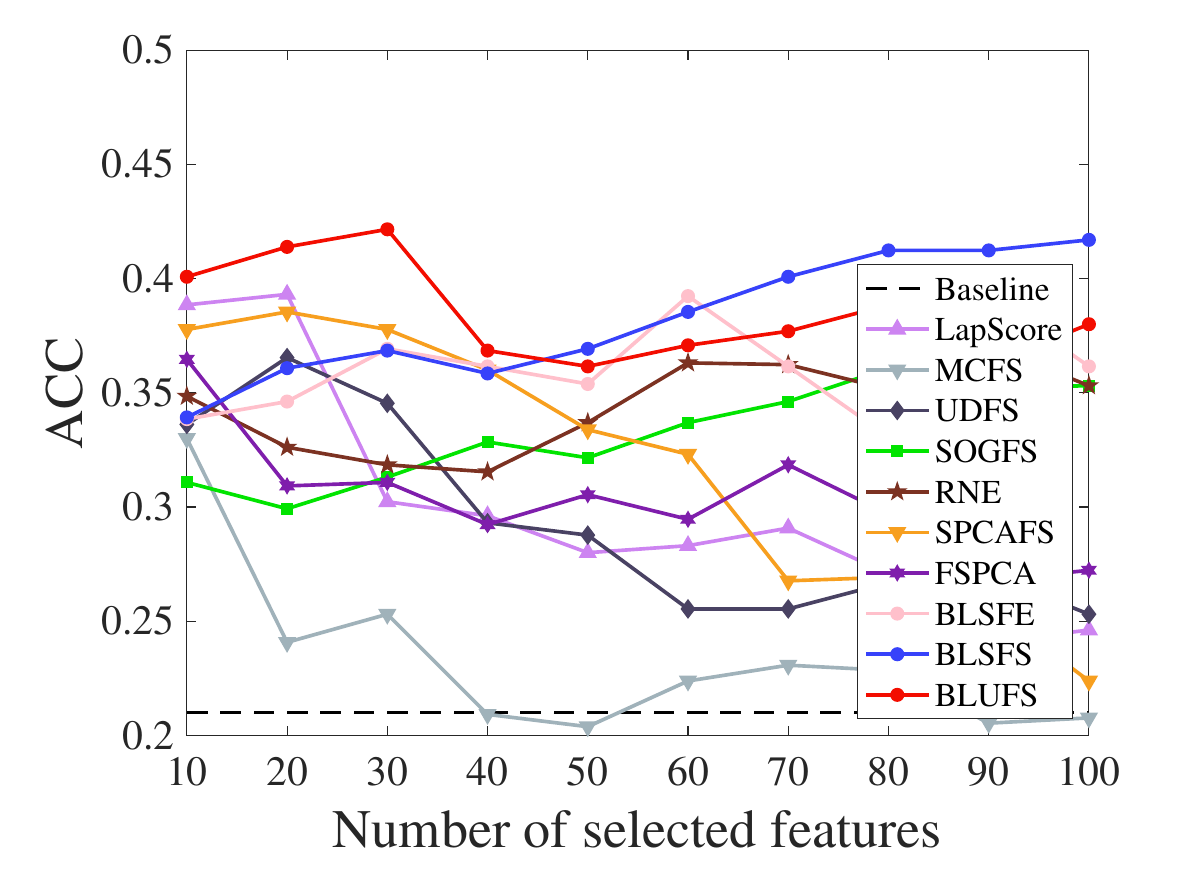}}
	\subfigure[(e) lung]{
		\includegraphics[width=1.7 in]{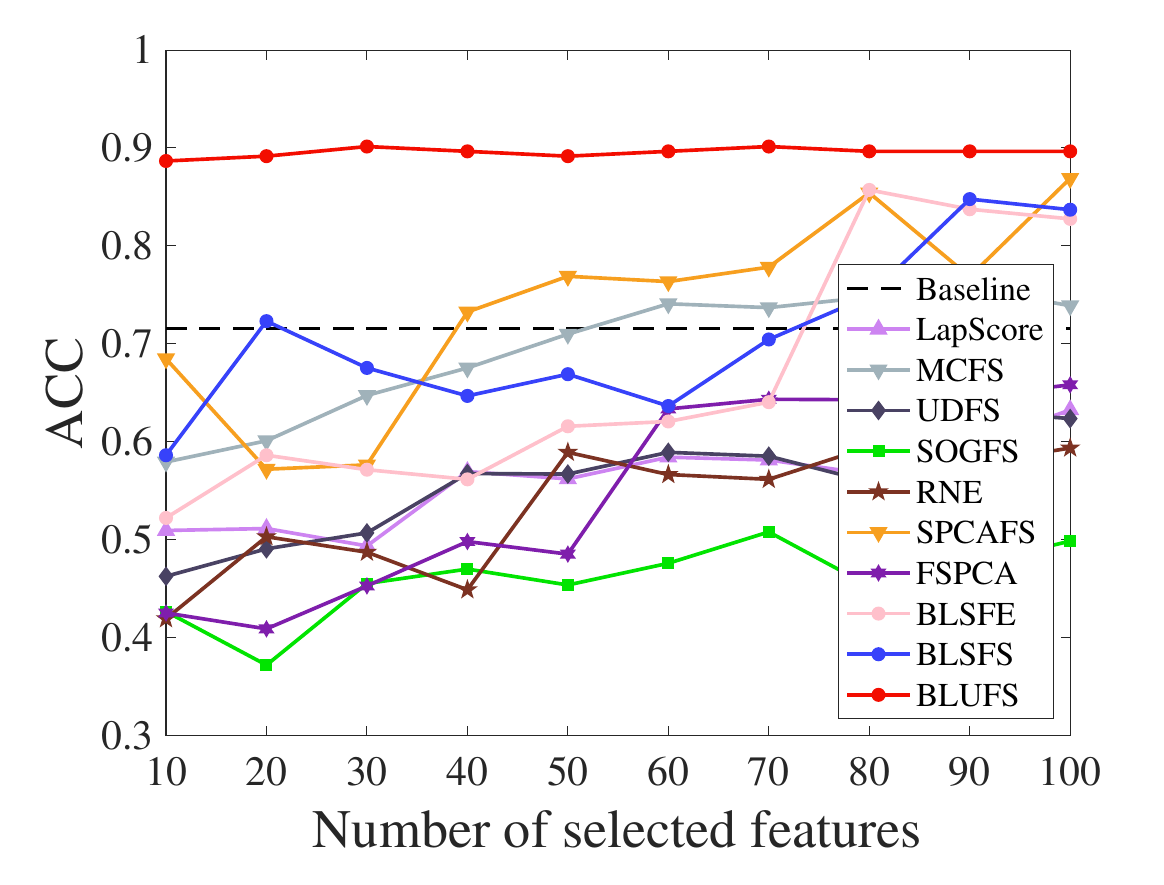}}
	\subfigure[(g) gisette]{
		\includegraphics[width=1.7 in]{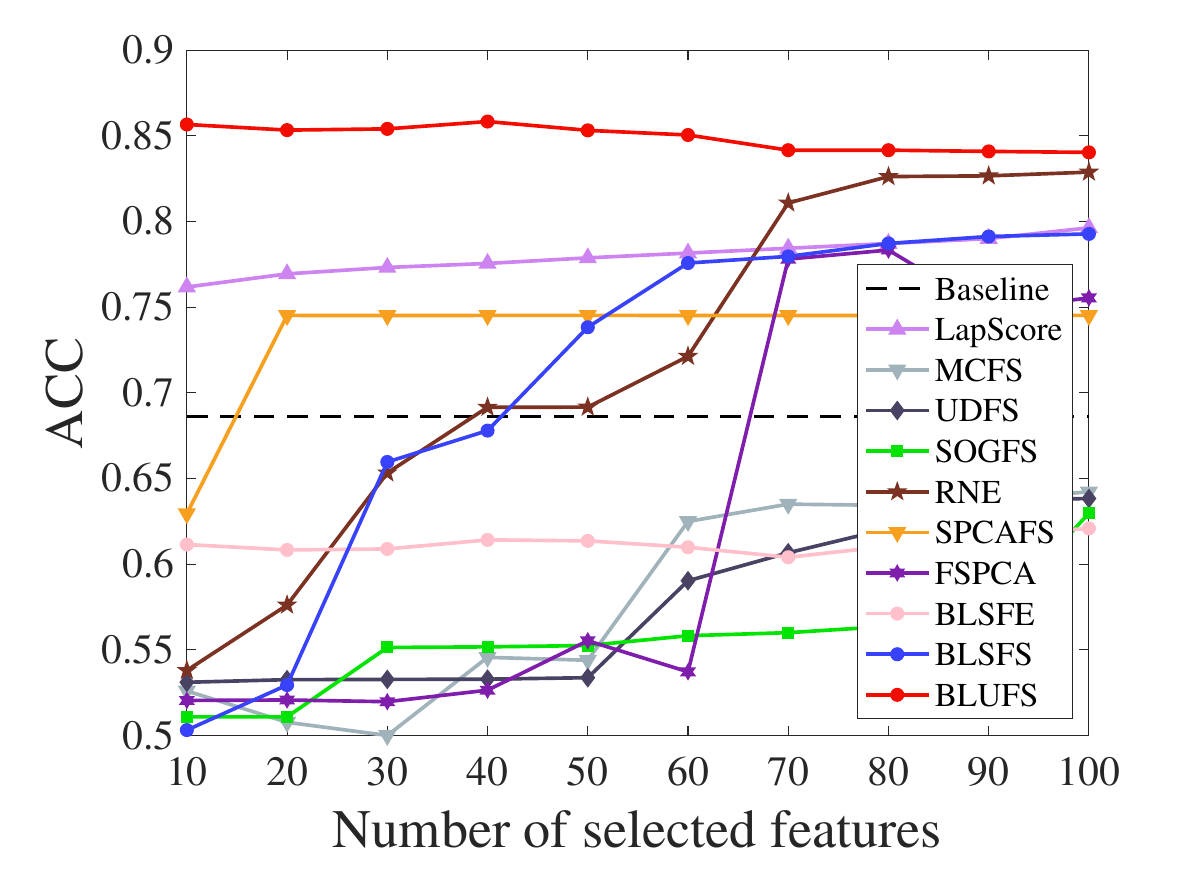}}
	\vskip-0.2cm
	\caption{Visual comparisons of the ACC metric under different datasets with different numbers of selected features.}
	\label{figure:3}
\end{figure*}

\begin{figure*}[t]
	\makeatletter
	\renewcommand{\@thesubfigure}{\hskip\subfiglabelskip}
	\makeatother
	\centering
    \subfigure[(c) Isolet]{
		\includegraphics[width=1.7 in]{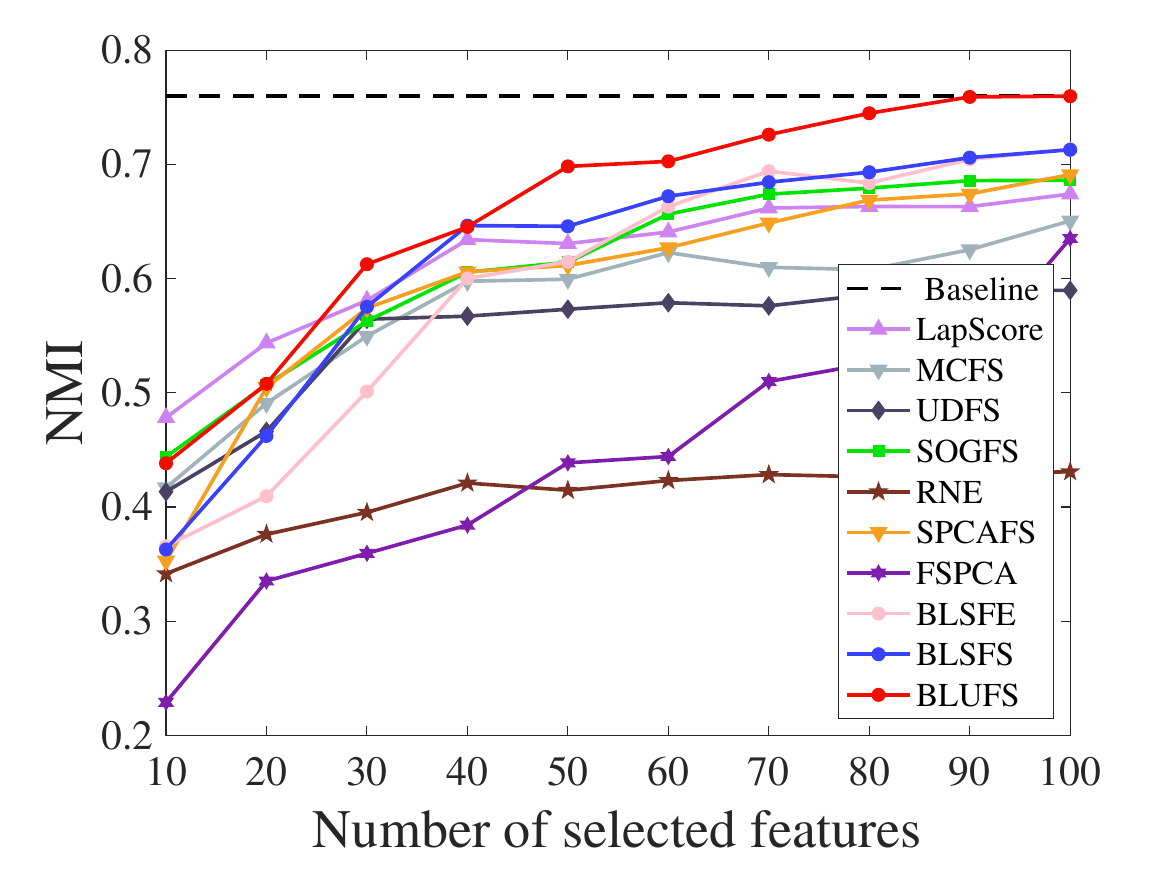}}
        \subfigure[(b) Jaffe]{
		\includegraphics[width=1.7 in]{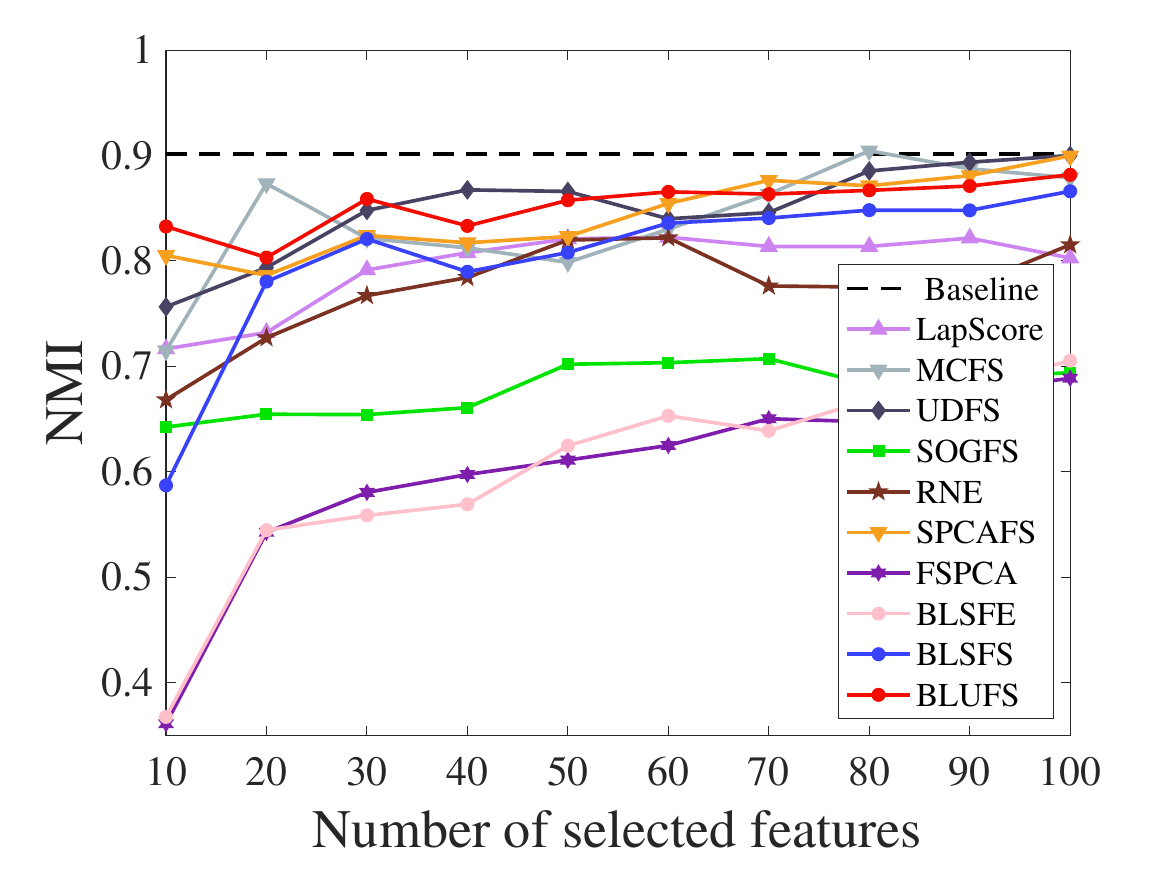}}
         \subfigure[(f) pie]{
		\includegraphics[width=1.7 in]{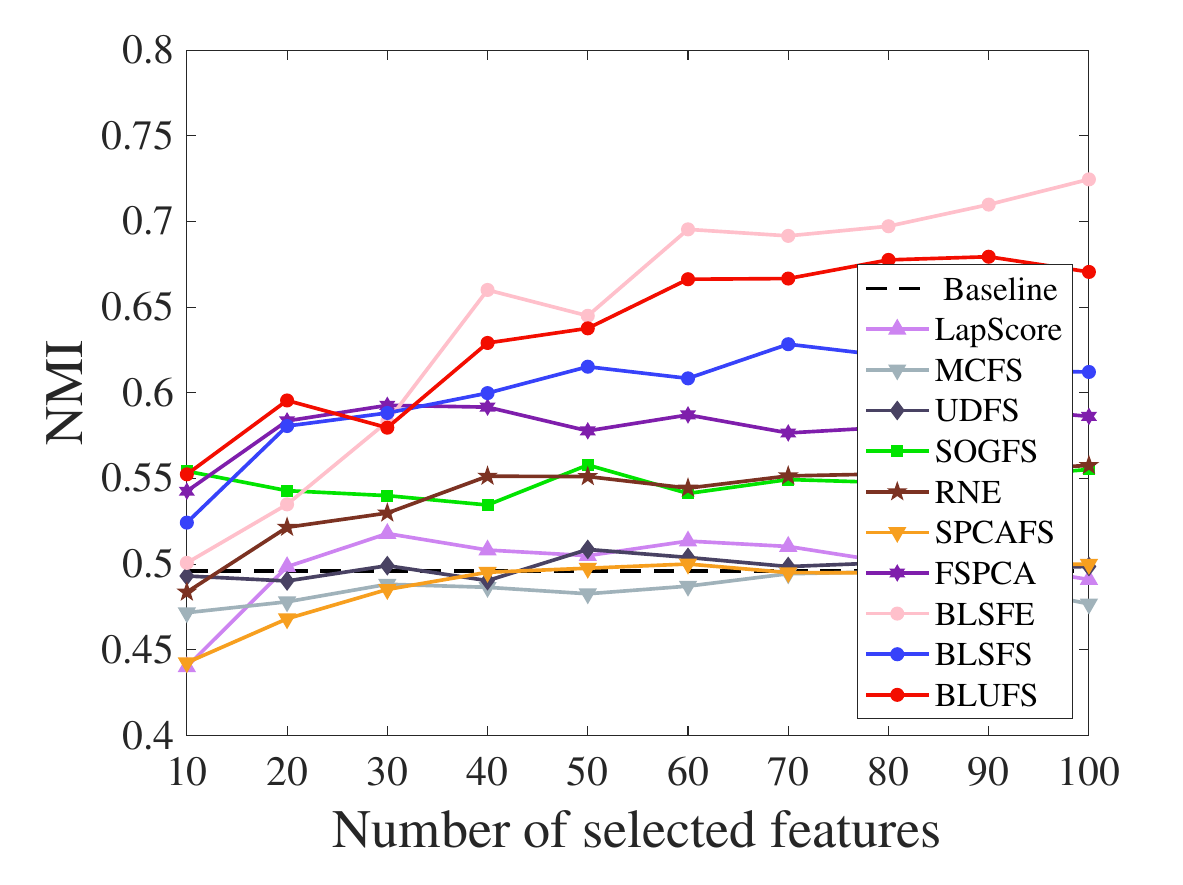}}
	\subfigure[(a) COIL20]{
		\includegraphics[width=1.7 in]{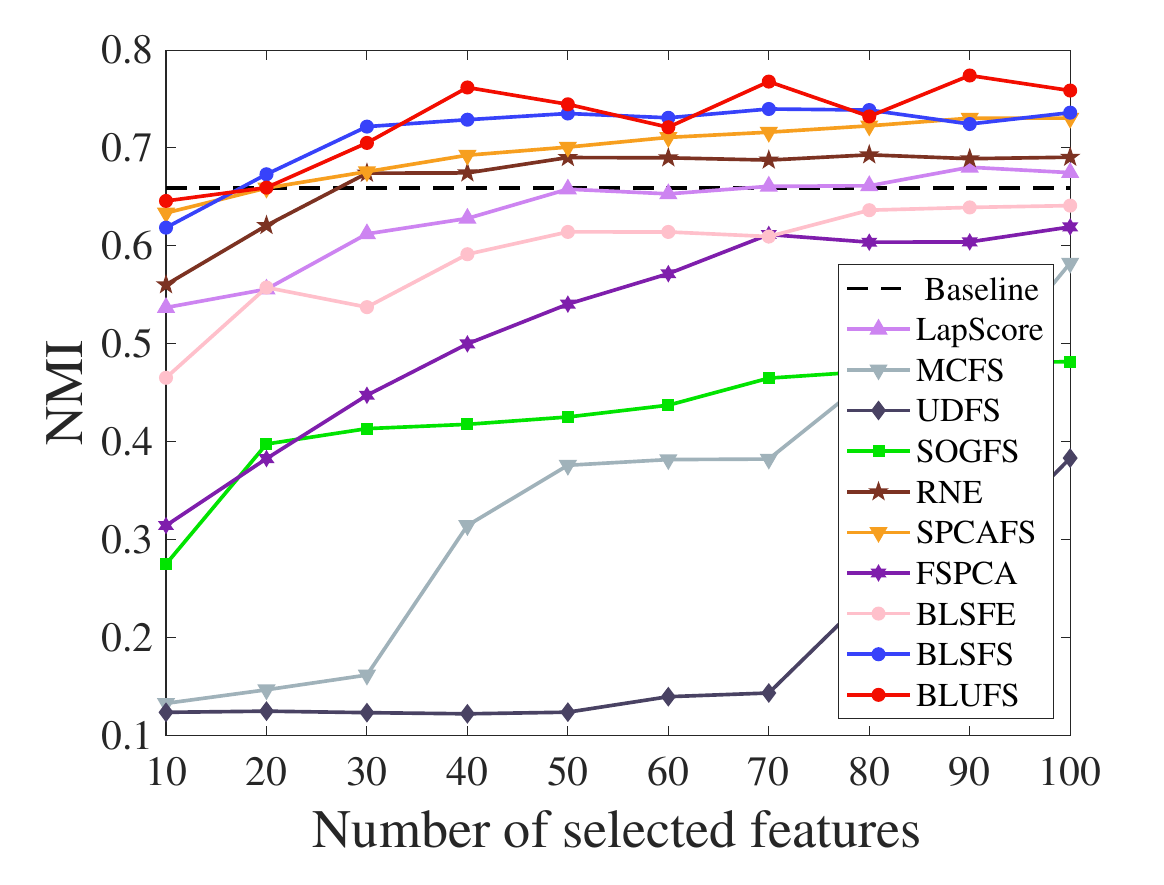}}
	\subfigure[(h) MSTAR]{
		\includegraphics[width=1.7 in]{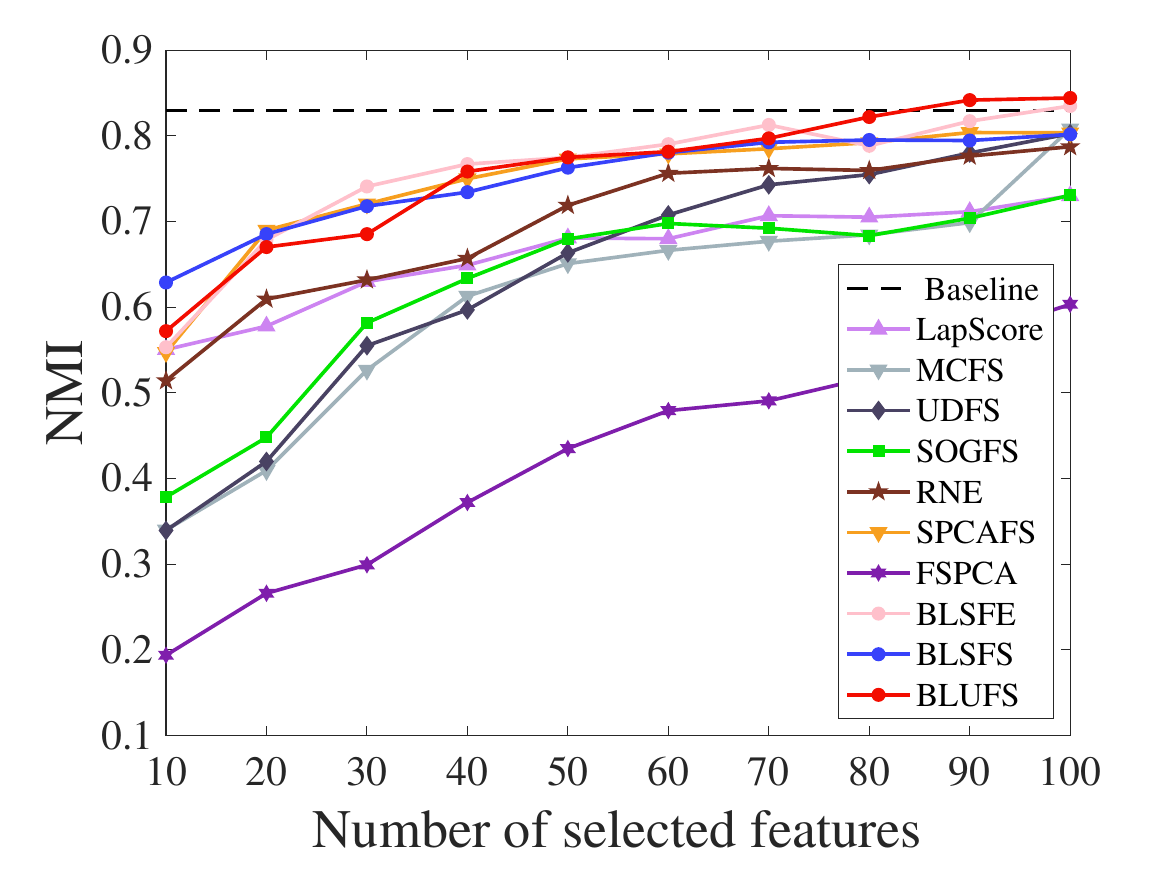}}
    \subfigure[(d) warpAR]{
		\includegraphics[width=1.7 in]{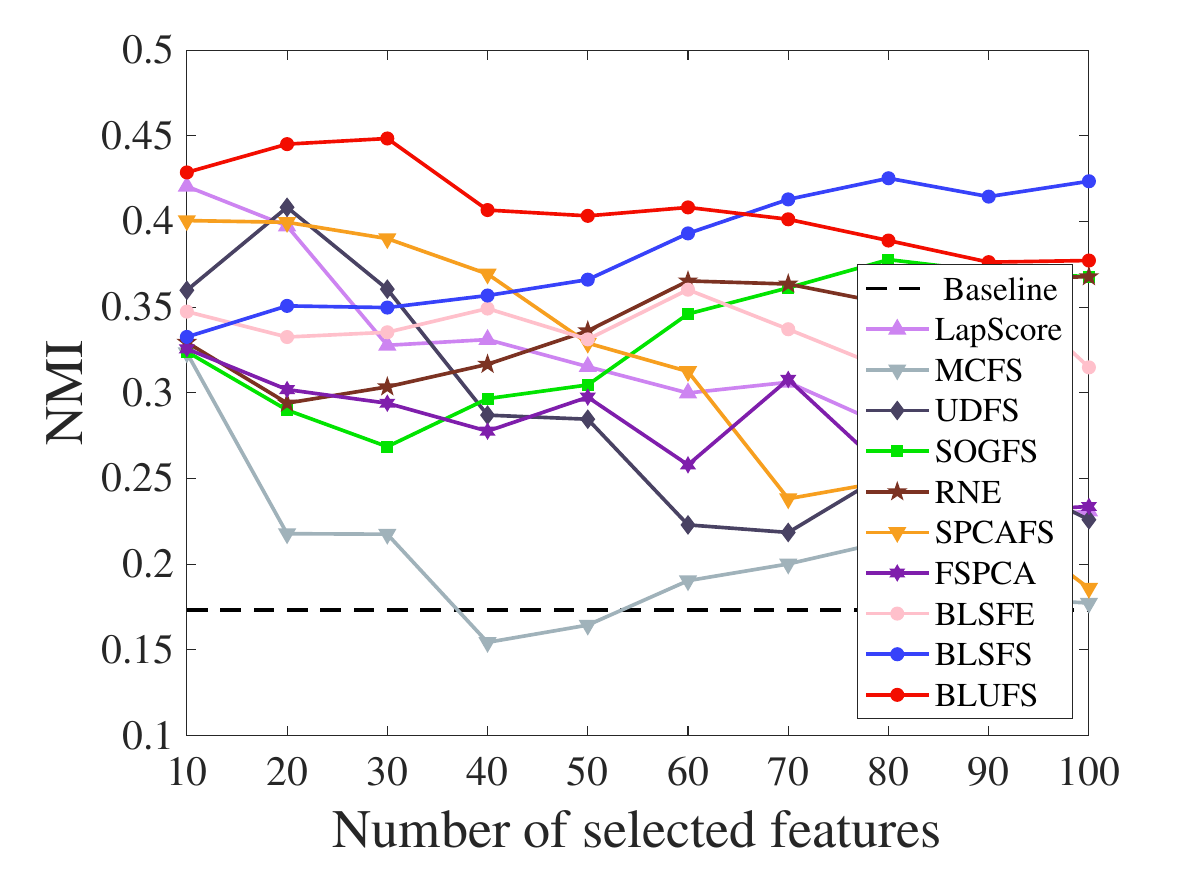}}
	\subfigure[(e) lung]{
		\includegraphics[width=1.7 in]{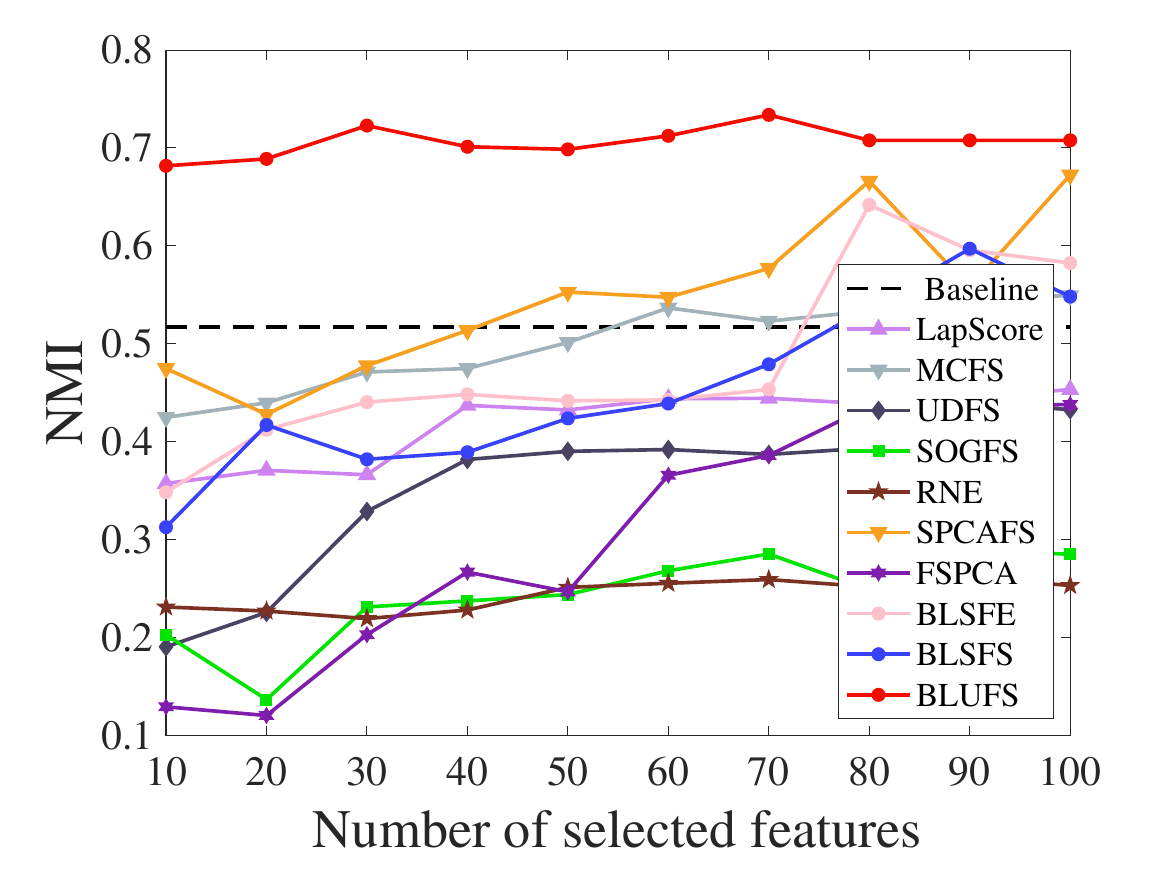}}
	\subfigure[(g) gisette]{
		\includegraphics[width=1.7 in]{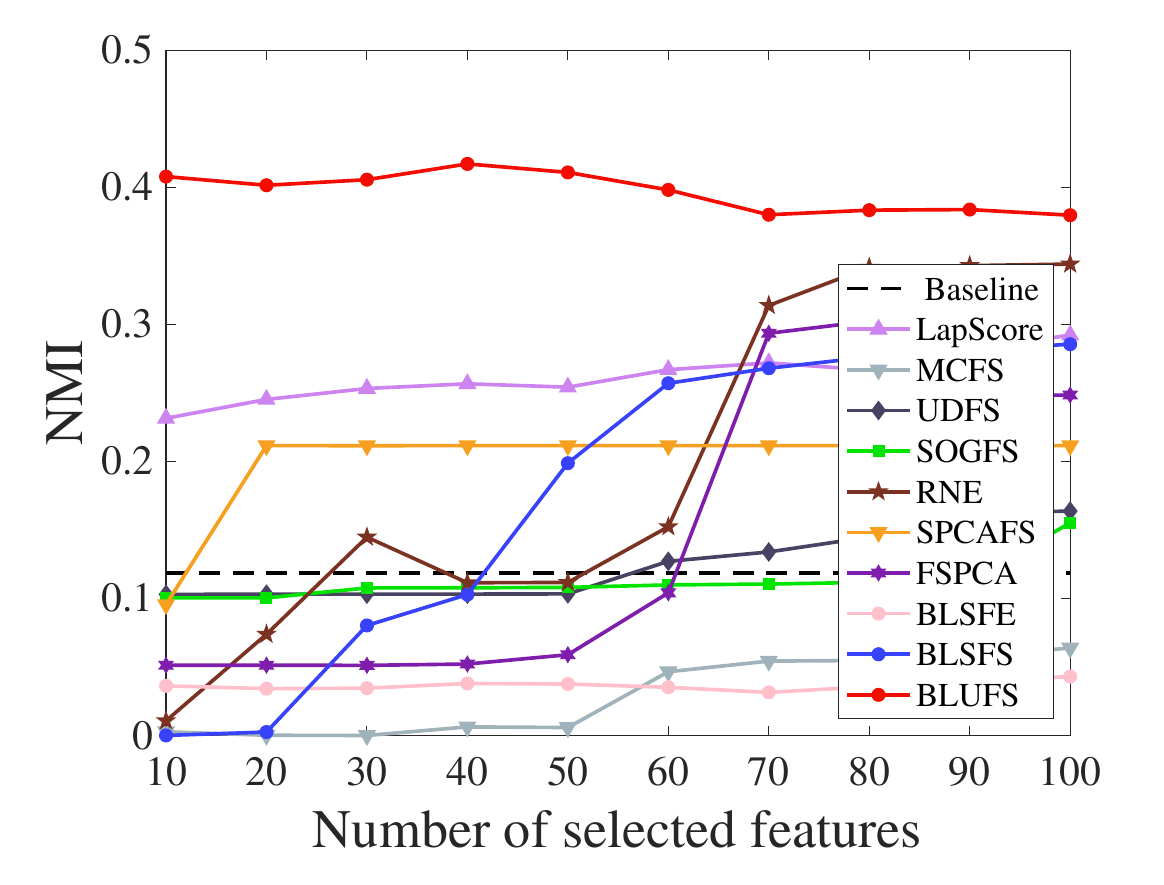}}
	\vskip-0.2cm
	\caption{Visual comparisons of the NMI metric under different datasets with different numbers of selected features.}
	\label{figure:4}
\end{figure*}

\begin{table*}[t]
  \centering
  \renewcommand\arraystretch{1.2}
  \caption{Average ACC (mean \% ± std \% ) and number of selected features of $k$-means clustering. The best results are marked in bold.}
  \resizebox{\textwidth}{!}{
  \label{tab:result1}
  \begin{tabular}{|c|c|c|c|c|c|c|c|c|c|c|c|}
  \hline
  Datasets & Baseline & LapScore & MCFS &UDFS  & SOGFS & RNE & SPCAFS & FSPCA & BLSFE & BLSFS & BLUFS \\
  \hline \hline
   \multirow{2}{*}{Isolet}
   & \multirow{2}{*}{61.79}  & 46.71±4.12  & 43.38±5.34  &  37.29±3.10 &49.05±1.67   &27.29±2.89  & 45.41±4.23  & 29.70±5.45 & 45.88±1.23 &  52.06±1.78 &  \textbf{54.41±2.90}  \\
  &  & (70) & (100) & (100) & (100) & (100) & (100) & (100) & (100) &  (100)&  \textbf{(100)} \\
  \hline
   \multirow{2}{*}{Jaffe}
   & \multirow{2}{*}{88.36}  & 77.11±5.12 & 83.14±1.45  &  84.32±4.10 &62.59±2.67   &74.56±3.89   & 83.37±5.23  & 55.45±1.56 & 56.67±2.38  & 79.82±2.78   &  \textbf{88.28±1.90}  \\
  &  & (90) & (20) &  (100)&(50)  &  (100)& (100) & (100) & (100) & (100) &  \textbf{(100)} \\
  \hline
   \multirow{2}{*}{pie}
   & \multirow{2}{*}{25.13}  & 24.84±5.53  & 23.14±1.93  & 24.53±5.63  &31.28±3.41   &   31.09±4.52& 23.20±1.94  & 32.19±3.51 &  40.76±1.31  & 36.06±4.62  &  \textbf{42.22±1.73}  \\
  &  & (70) & (80) & (100) & (100) &  (90)& (100) & (10) &  (100)& (70) &  \textbf{(100)} \\
  \hline
  \multirow{2}{*}{COIL20}
  & \multirow{2}{*}{57.32}  & 50.91±1.23  & 26.71±2.45   & 15.20±5.10  &32.68±3.67   &   54.31±4.89&  62.06±1.34 & 41.37±2.56 & 48.24±1.52 & 61.99±3.78  &  \textbf{63.95±2.90}  \\
  &  & (90) & (100) & (70) & (100) &  (50)&  (100) & (70) &(90) & (70) &  \textbf{(100)} \\
  \hline
  \multirow{2}{*}{MSTAR}
   & \multirow{2}{*}{78.11}  & 64.34±3.81  & 67.57±4.92  & 67.34±3.91  &65.25±5.03   &   71.35±1.98& 77.53±5.02 & 46.93±6.03 &  81.51±2.19  & 79.08±1.99  &  \textbf{82.67±2.01}  \\
  &  & (100) & (100) & (100) & (100) &  (100)& (90) & (100) &  (90)& (100) &  \textbf{(100)} \\
  \hline
  \multirow{2}{*}{warpAR}
   & \multirow{2}{*}{21.00}  & 29.92±3.12  & 23.33±4.34  &  29.29±3.01 & 33.22±5.56  &   34.46±1.89& 31.75±4.13 & 30.32±5.24 & 36.08±2.34 &  38.24±1.90  &  \textbf{38.45±2.11}  \\
  &  & (20) & (10) & (20) & (70) &  (100)& (20) & (10) & (80) &  (100) &  \textbf{(30)} \\
  \hline
  \multirow{2}{*}{lung}
   & \multirow{2}{*}{71.57}  & 56.06±4.22  & 69.34±5.33  & 55.83±4.32  & 45.86±1.91  &   53.40±3.21&  73.68±5.43  & 54.94±1.92 & 66.40±1.32 & 69.85±3.31   &  \textbf{89.56±1.42} \\
  &  & (100) & (90) & (90) & (70) &  (80)&  (100) & (100) & (70) & (100) &  \textbf{(70)} \\
  \hline
  \multirow{2}{*}{gisette}
   & \multirow{2}{*}{68.62}  &  77.98±1.95 & 71.64±5.83  & 55.55±4.72  & 57.56±1.96   &   57.97±3.61& 73.36±3.71  & 62.44±4.82 & 61.18±1.22 & 70.35±5.93  &  \textbf{84.90±1.97}  \\
  &  &  (100) & (100) & (100) & (100) & (100) & (100) & (80) & (90) & (100) &  \textbf{(40)} \\
  \hline\hline
  Average
   & 58.98  & 53.48 & 49.32  & 46.42  & 46.93  &  52.26 &58.79 & 44.16 & 54.59 & 60.93&  \textbf{68.05} \\
  \hline
  \end{tabular}
  }
\label{table:2}
\end{table*}

\begin{table*}[t]
  \centering
  \renewcommand\arraystretch{1.2}
  \caption{Average NMI (mean \% ± std \% ) and number of selected features of $k$-means clustering. The best results are marked in bold.}
 \resizebox{\textwidth}{!}{
  \label{tab:result2}
  \begin{tabular}{|c|c|c|c|c|c|c|c|c|c|c|c|}
  \hline
  Datasets & Baseline & LapScore & MCFS &UDFS  & SOGFS &  RNE& SPCAFS & FSPCA & BLSFE & BLSFS & BLUFS \\
  \hline \hline

  \multirow{2}{*}{Isolet}
   & \multirow{2}{*}{75.96}   &  61.70±4.12  & 57.70±3.98  & 55.04±4.34  &  61.15±1.45 &  40.86±2.89& 59.58±3.56  & 43.94±1.12  &59.49±1.31 & 61.60±2.09  &  \textbf{65.94±2.67 } \\
  &  &  (100) & (100) & (100) & (100) &(100)  & (100) & (100) & (100) & (100) &  \textbf{(100)} \\
  \hline
   \multirow{2}{*}{Jaffe}
   & \multirow{2}{*}{90.12}  & 79.43±2.45 & 83.85±4.76  &  84.96±2.23& 67.92±3.89  &   77.32±1.56& 84.39±4.45  & 59.81±3.34 & 60.11±1.34 & 80.25±1.78   &  \textbf{85.34±2.67 } \\
  &  & (60) & (80) &  (100)& (70)  & (60)& (100) & (100) &(100)& (100) &  \textbf{(100)}\\
  \hline
  \multirow{2}{*}{pie}
   & \multirow{2}{*}{49.59}   & 49.86±1.02  & 48.50±2.98  & 49.81±1.56  & 54.73±4.76  &  53.98±3.89 & 48.78±2.23  & 58.09±4.45  & \textbf{64.42±2.14}& 59.63±3.34  &  63.55±1.78 \\
  &  & (70) & (90) & (100) &  (100)&  (90)& (100) & (100) & \textbf{(100)}& (70) &  (70) \\
  \hline
  \multirow{2}{*}{COIL20}
   & \multirow{2}{*}{65.94}   & 63.22±1.34  & 34.00±2.56   &18.06±1.87   & 42.65±4.23  &   66.70±3.12&69.72±2.98 & 51.95±4.09  &59.07±1.20&   71.48±3.65  &  72.71±1.02  \\
  &  & (90) & (100) & (70) & (100) & (50) & (100) & (100) & (100) &   (70) &  (100)\\
  \hline 
 \multirow{2}{*}{MSTAR}
   & \multirow{2}{*}{82.93}  & 66.20±4.67  & 60.73±3.78  &63.61±4.56   & 62.28±1.98  &   69.72±2.34& 74.43±3.23 & 42.21±1.67  & \textbf{75.57±1.31} & 74.92±2.78  &  75.46±2.89 \\
  &  & (100) & (100) & (100) & (100) &  (100)& (90) & (100) & \textbf{(100)}& (100) &  (100) \\
  \hline
  \multirow{2}{*}{warpAR}
   & \multirow{2}{*}{17.32}   & 31.34±4.56  & 20.39±3.23  &  28.76±4.89 & 33.07±1.67  &   33.94±2.78& 31.03±3.45 & 27.80±1.23  &33.90±1.43 &  38.25±2.12 &  \textbf{40.84±2.98} \\
  &  & (10) & (10) & (10) & (80) &  (100)& (10) & (10) &(90)&  (80) &  \textbf{(30)} \\
  \hline
  \multirow{2}{*}{lung}
   & \multirow{2}{*}{51.73}   & 41.90±3.09  & 50.01±1.89  & 35.62±3.12   & 24.29±2.56  & 24.38±4.23 &  54.68±1.34  & 30.28±2.45  &48.07±1.21 & 50.17±4.09   &  \textbf{70.63±1.65}  \\
  &  & (100) & (100) & (90) & (90) &(90)  &  (100) & (100) &(70)& (100) &  \textbf{(70)} \\
  \hline
  \multirow{2}{*}{gisette}
   & \multirow{2}{*}{11.86}   &  26.19±2.67 & 2.92±1.12  & 12.47±2.89  &   11.24±3.98 &   19.46±1.45&19.98±4.34  & 14.60±3.56 &3.65±1.41 & 17.52±1.12  & \textbf{ 39.66±2.09}  \\
  &  &  (100)& (100) & (100) & (100) & (100) & (100) & (80) &(10)& (60) &  \textbf{(40)} \\
  \hline\hline
  Average
   & 55.68  & 52.48 & 44.76  & 43.53  & 44.66  &   48.29&43.54 & 41.08 & 50.53 &  56.72&  \textbf{64.26} \\
  \hline
  \end{tabular}
  }
  \label{table:3}
\end{table*}

\subsubsection{Parameter Settings}
For LapScore,  MCFS,  SOGFS,  RNE,  BLSFS,  and BLUFS,  the number of $k$-NN is set to 10.
For SOGFS,  SPCAFS, BLSFS, and BLUFS, their regularization parameters are selected from \(\{10^{-4},  10^{-3}, \cdots,  10^{3}\}\).
For other parameters, the default values or the best parameters provided by the authors are used.
In addition, the convergence criterion is 
\begin{equation}
\frac{|f^{k+1} - f^k|}{\max\{|f^k|,  1\}} < 10^{-4},
\end{equation}
where \(f\) is the objective function and \(k\) is the iteration number,  with the maximum value of \(k\) set to 50.

\subsubsection{Evaluation Metrics}
In our study, to explore the performance of BLUFS in clustering applications, we use the $k$-means method to evaluate its performance. It is worth noting that we adopt two metrics, namely,accuracy (ACC) and normalized mutual information (NMI).
For every data point \(\mathbf{x}_i\), we use \(p_i\) to represent its true class label and \(q_i\) to represent its assigned cluster label.
The ACC is defined as 
\begin{equation}
  \text{ACC} = \frac{1}{n} \sum_{i=1}^{n} \zeta(p_i,  f(q_i)).
  \end{equation}
Here, \( n \) denotes the total count of data instances. The indicator function \(\zeta(p, q)\) equals 1 when \( p \) matches \( q \), and equals 0 otherwise. The function \( f(\cdot) \) serves as a permutation function, aiming to align the cluster labels with their corresponding true labels. 

Let $A$ denote the ground truth labels and $B$ denote the cluster labels.
The NMI is defined as 
\begin{equation}
\text{NMI}(A,  B) = \frac{\text{MI}(A,  B)}{\max(H(A),  H(B))}, 
\end{equation}
where $H(\cdot)$ denotes the entropy and MI$(\cdot, \cdot)$ represents the mutual information,  the range for ACC and NMI is [0, 1] \cite{strehl2002cluster}.

Note that we choose different numbers of selected features, i.e., \(\{10, 20, \ldots, 100\} \), and repeat it ten times to obtain the average as the final ACC and NMI.
In addition, we also evaluate the classification performance with the $k$-NN classifier.

\subsection{Numerical Experiments}

\subsubsection{Synthetic Results}

In this experiment, all the methods compared are applied to two synthetic datasets. These synthetic datasets are composed of two key features and seven randomly generated Gaussian noise features. Each UFS method is used to select two features and the selected features along with their corresponding samples are visualized in the form of scatter plots.

The results of the feature selection process are shown in Figs. \ref{figure:1} and \ref{figure:2}. For the Dartboard1 dataset, it is evident that LapScore, RNE, BLSFE, BLSFS, and BLUFS surpass the other methods discussed in terms of identifying the most suitable features. In the case of the Diamond9 dataset, the feature selection results highlight that BLUFS is the only method that can accurately identify the two most discriminative features. Collectively, these findings demonstrate the efficacy of our proposed BLUFS for synthetic datasets.

\subsubsection{Real-World Result}

Figs. \ref{figure:3} and \ref{figure:4} show the clustering results of ACC and NMI under different datasets with different numbers of selected features, respectively. 
It can be observed that the proposed BLUFS has significant advantages in most real-world datasets, which highlights the effectiveness of our method in feature selection.
Furthermore, in the lung and gisette datasets, which have the highest dimensions, BLUFS significantly outperforms other methods regardless of the number of features selected. 
This demonstrates the advantage of our method on high-dimensional datasets.

Tables \ref{table:2} and \ref{table:3} list the corresponding ACC and NMI , with the best method highlighted in bold, and the number of selected features given in brackets (except for the Baseline method).
Our proposed BLUFS method achieves better results than other methods in most real-world datasets.
In Table \ref{table:2}, the average ACC values of BLUFS across all datasets with different feature selection quantities are higher than those of other methods. It is worth noting that these bi-level methods generally outperform other single-level methods on most datasets. For example, on the pie dataset, the bi-level methods improve the ACC by at least 5\% compared to single-level methods. Additionally, on the gisette dataset, which has the highest dimensionality, BLUFS improves ACC by at least 14\% compared to other bi-level methods (BLSFS and BLSFE), indicating the advantage of BLUFS on high-dimensional datasets. Furthermore, in Table \ref{table:3}, except for the pie and MSTAR datasets, BLUFS demonstrated the highest average NMI values on all datasets. It is worth noting that NMI performance on the highest-dimensional dataset gisette is at least 20\% higher than that of other bi-level methods. Among the average ACC and NMI results across eight datasets, BLUFS outperformed the second-place BLSFS by at least 8\%. These results also indicate that bi-level methods perform better than other methods.

The good performance can be attributed to the following reasons:
(i) Unlike other spectral clustering methods, the proposed BLUFS introduces a scaled cluster indicator matrix \( Y \) to guide feature selection. This matrix can more flexibly represent the similarities between data points and their relationship with structural information, thereby enhancing the accuracy and robustness of feature selection.
(ii) Compared to BLSFS, the proposed BLUFS can better control the number of selected features by using \( \ell_{2, 0} \)-norm constrained optimization, thus exhibiting greater advantages in high-dimensional datasets.

\begin{table}[t]
  \centering
  \renewcommand{\arraystretch}{1.2}
  \caption{Three degenerated versions and BLUFS.}
  \begin{tabular}{|c|c|c|c|}
    \hline
    Cases&  Feature level & Clustering level& Graph learning\\ \hline \hline
     Case I   &  \ding{52} &  \ding{55} &  \ding{52} \\ \hline
     Case II  & \ding{55} &  \ding{52}&  \ding{52} \\ \hline
     Case III   &  \ding{52} &  \ding{52}&  \ding{55} \\ \hline
    BLUFS &   \ding{52} &  \ding{52}&  \ding{52} \\
    \hline
  \end{tabular}
  \label{sum1}
\end{table}

\begin{table}[t]
    \centering
    \renewcommand\arraystretch{1.2}
    \caption{Clustering results compared with three degenerated versions of BLUFS. The best results are marked in bold.}
    \begin{tabular}{|c|c|c|c|c|c|c|}
    \hline 
     Datasets & Metrics & Case I & Case II & Case III & BLUFS \\   \hline \hline
      \multirow{2}{*}{Isolet}
      & ACC &  61.98   & 58.88  & 42.97  &  \textbf{65.36}  \\
     & NMI & 70.38  &  72.90  & 56.35  &  \textbf{75.96}  \\  \hline
     \multirow{2}{*}{Jaffe}
      & ACC &  83.19   & 74.84   & 79.15  &  \textbf{91.87 }  \\
     & NMI &  84.01    & 77.04  & 79.71  &  \textbf{88.17}  \\  \hline
     \multirow{2}{*}{pie}
      & ACC & 38.13  & 32.99  &  41.90  &  \textbf{42.22}  \\
     & NMI & 59.31  & 55.58  &  62.43  &  \textbf{63.55} \\  \hline
     \multirow{2}{*}{COIL20}
      & ACC & 62.63  & 59.90  &  65.94  &  \textbf{69.27} \\
     & NMI & 73.02  & 67.43  &  73.54  &  \textbf{77.40} \\  \hline
     \multirow{2}{*}{MSTAR}
      & ACC &  82.81  & 77.51  & 77.67  &  \textbf{91.58} \\
     & NMI &  78.47  & 77.21  & 78.09  &  \textbf{84.41}\\  \hline
     \multirow{2}{*}{warpAR}
      & ACC &  37.92 & 35.31  & 22.54  &  \textbf{42.15}\\
     & NMI & 39.59  &  39.83  & 20.09  &  \textbf{44.84}\\  \hline
     \multirow{2}{*}{lung}
      & ACC & 62.07  & 59.70  &  78.82 &  \textbf{90.15} \\
     & NMI & 42.95  & 42.30  &  52.89  &  \textbf{73.38} \\ \hline
     \multirow{2}{*}{gisette}
      & ACC &  71.48  & 50.89  & 67.76  &  \textbf{85.83} \\
     & NMI &  17.95  & 12.35 & 10.37  &  \textbf{41.69}\\  \hline
     
    \end{tabular}
    \label{table:5}
\end{table}
\begin{table}[t]
    \centering
    \renewcommand\arraystretch{1.2}
    \caption{Visual comparisons of selected image samples from the pie dataset with the corresponding ACC (\%) and NMI (\%). The best results are marked in bold. }
    \label{graph}
    \begin{tblr}{
      cells={halign=c, valign=m}, 
      colspec={p{1cm}|p{9mm}|p{9mm}|p{9mm}|p{9mm}|p{7mm}|p{7mm}}, 
      hlines, 
      vlines
    }
    Cases & \SetCell[c=4]{c} Samples &&&& ACC & NMI \\
    \hline
    Case I & \centering \includegraphics[width=1\linewidth, valign=c]{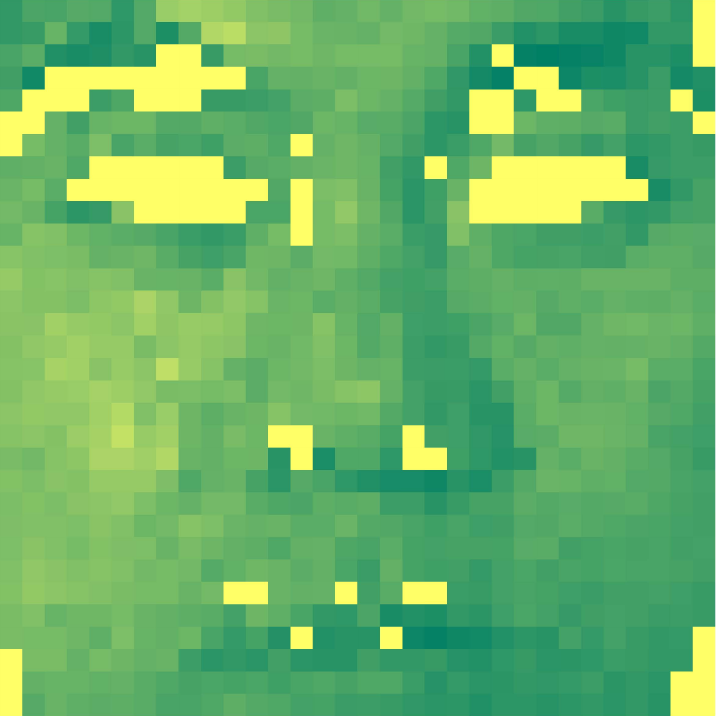} & \includegraphics[width=1\linewidth, valign=c]{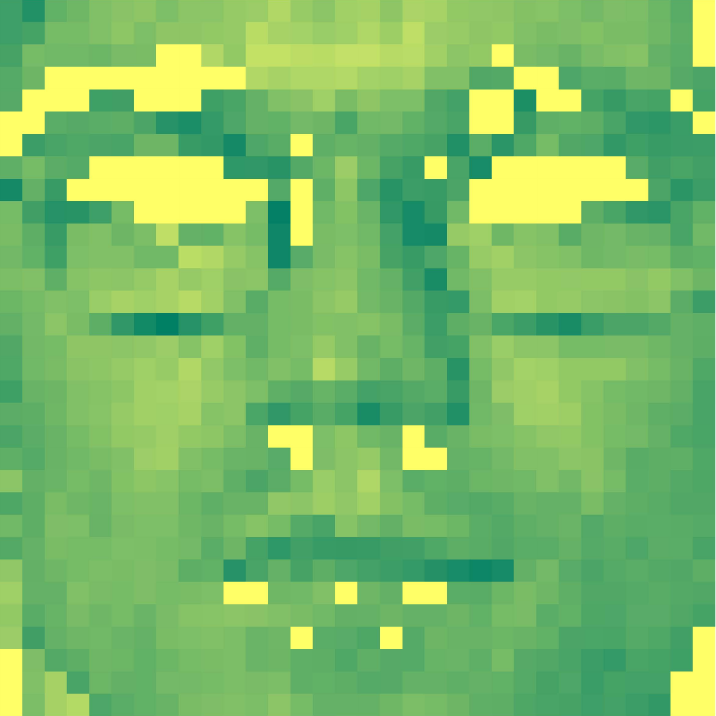} & \includegraphics[width=1\linewidth, valign=c]{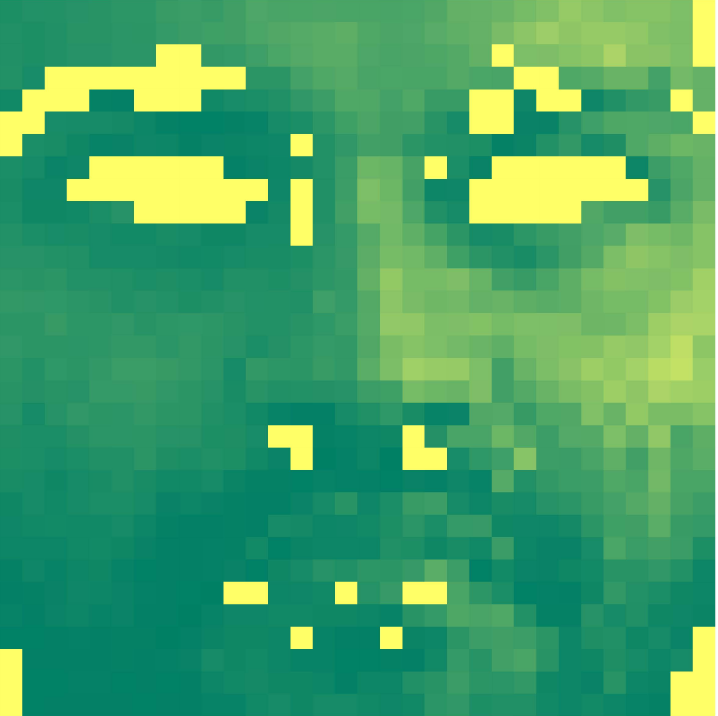} & \includegraphics[width=1\linewidth, valign=c]{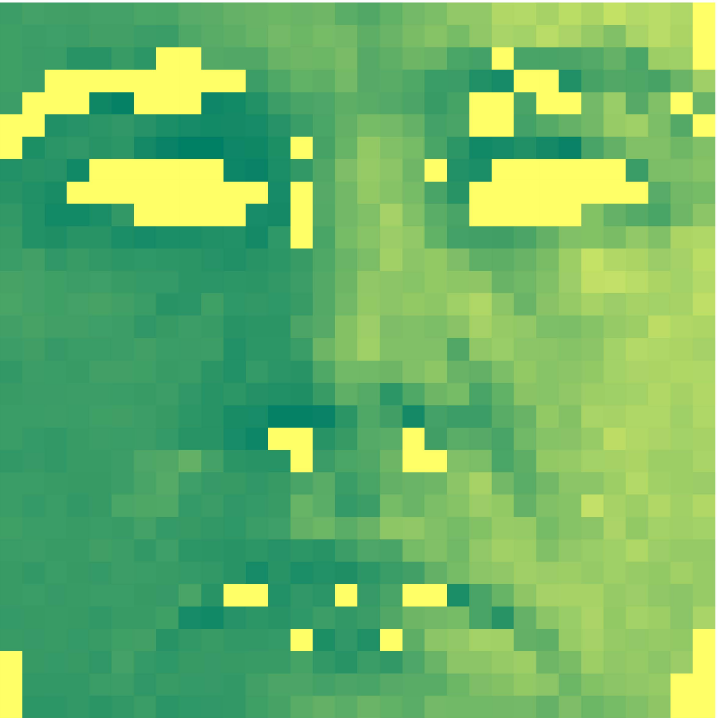} & 41.98 & 63.02 \\
    Case II & \includegraphics[width=1\linewidth, valign=c]{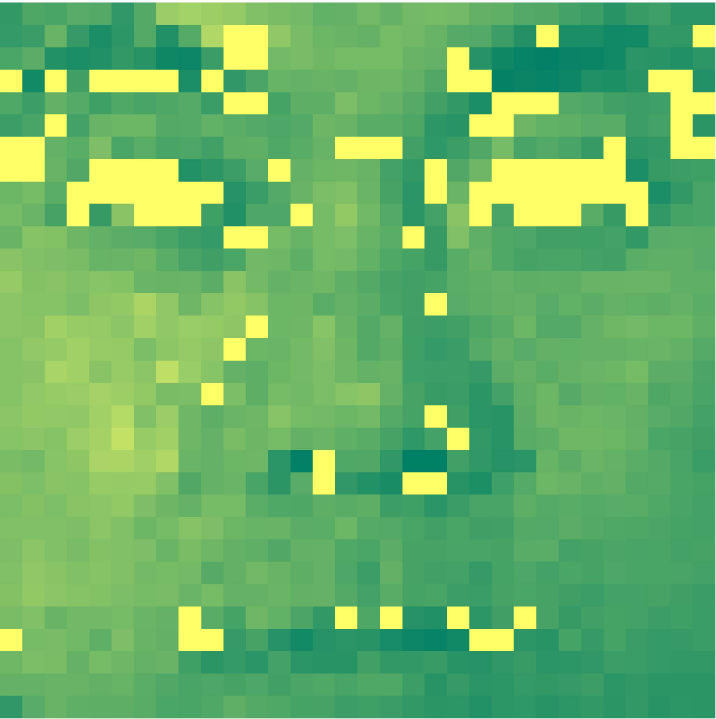} & \includegraphics[width=1\linewidth, valign=c]{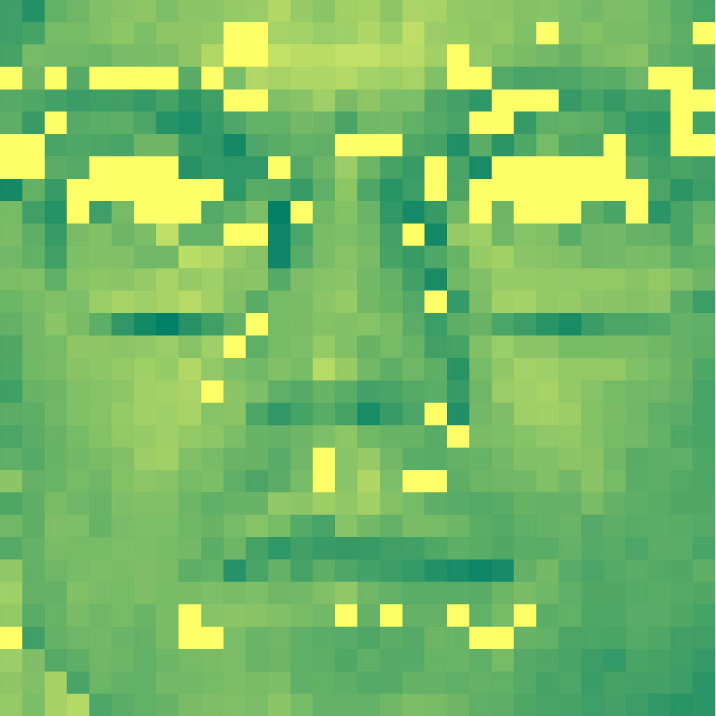} & \includegraphics[width=1\linewidth, valign=c]{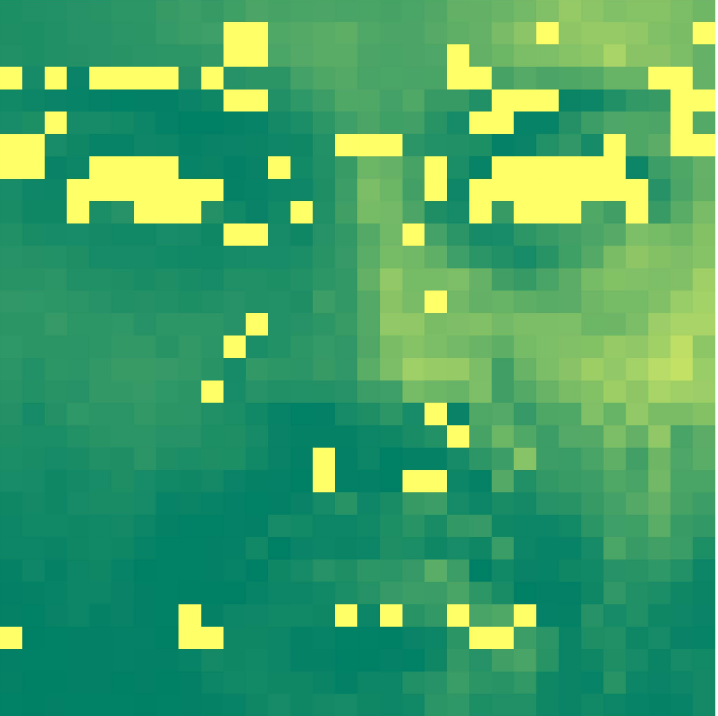} & \includegraphics[width=1\linewidth, valign=c]{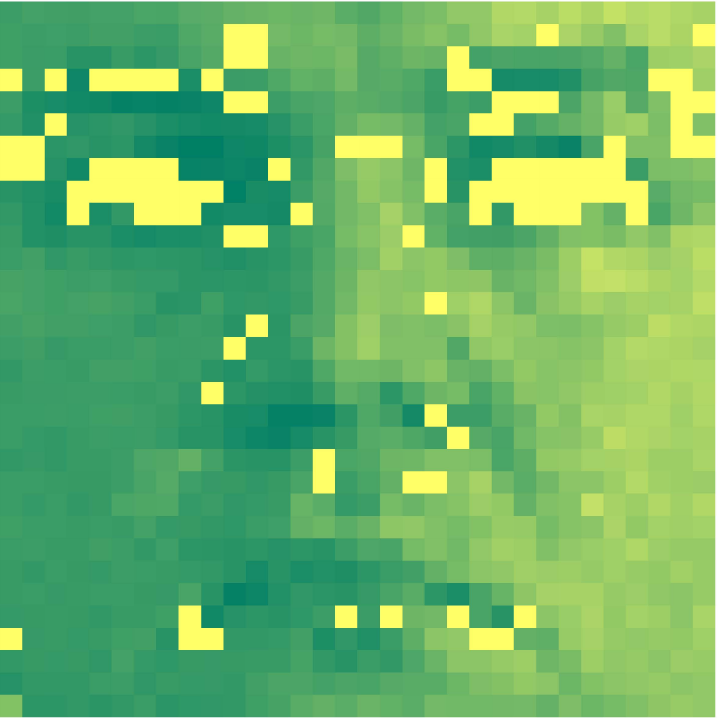} & 38.11 & 62.51 \\
   Case III & \includegraphics[width=1\linewidth, valign=c]{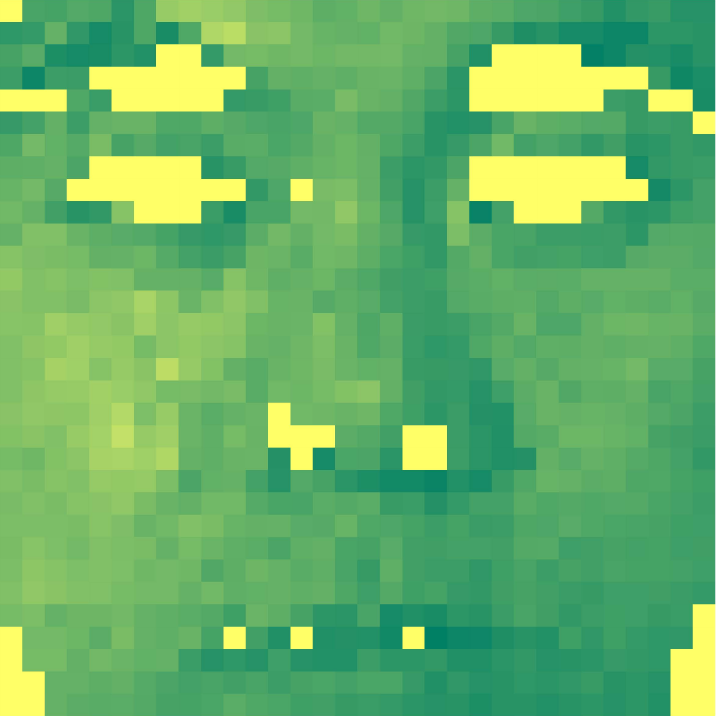} & \includegraphics[width=1\linewidth, valign=c]{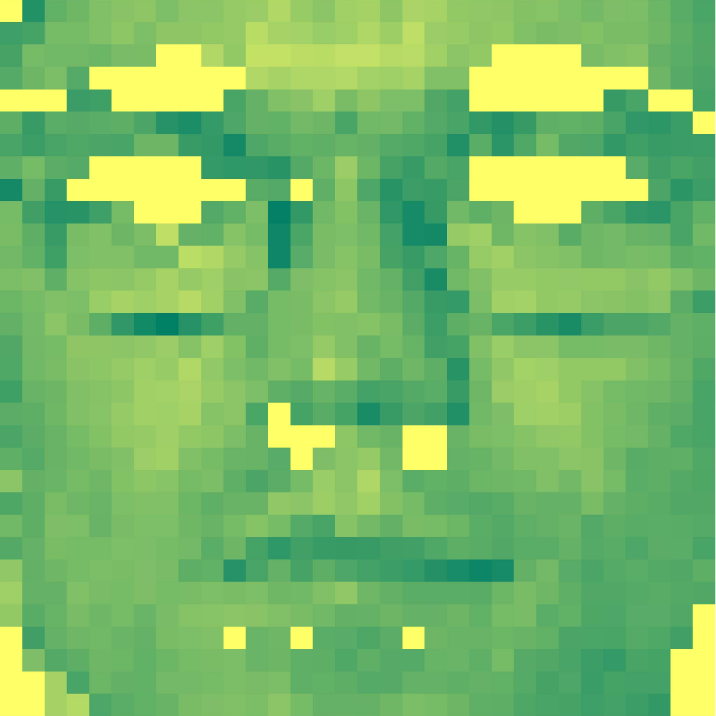} & \includegraphics[width=1\linewidth, valign=c]{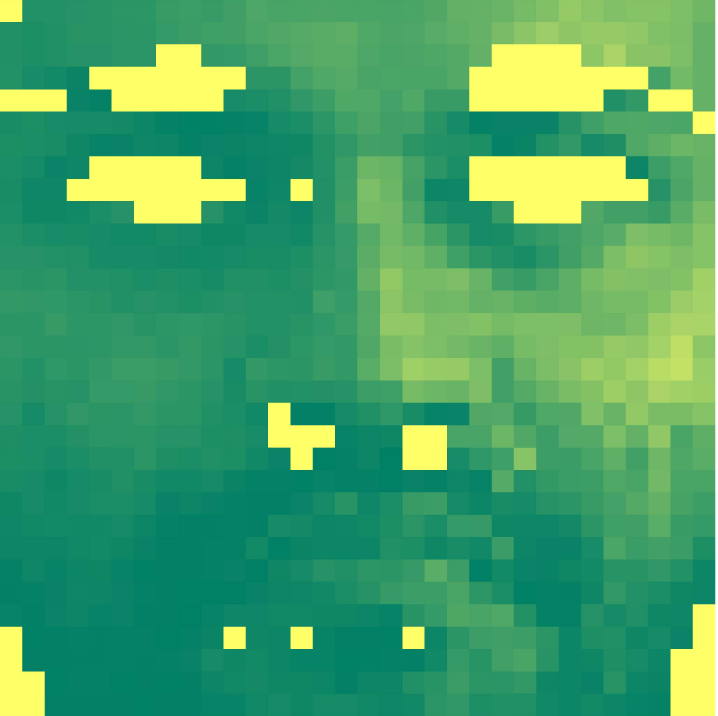} & \includegraphics[width=1\linewidth, valign=c]{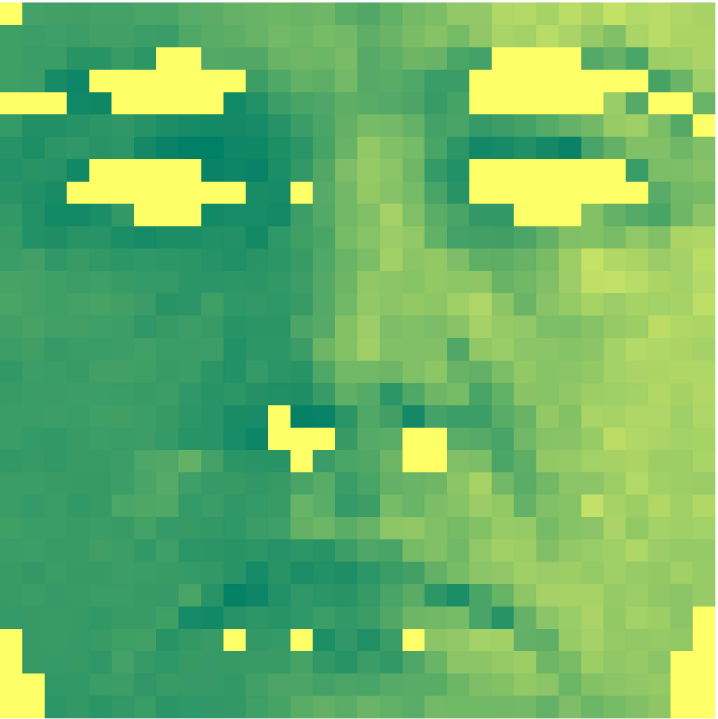} &  42.35 &  63.24 \\
   BLUFS & \includegraphics[width=1\linewidth, valign=c]{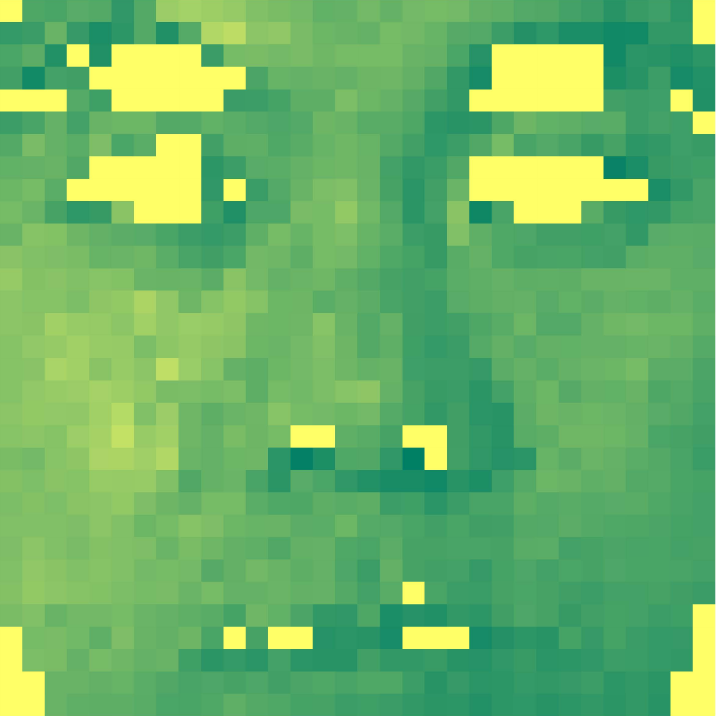} & \includegraphics[width=1\linewidth, valign=c]{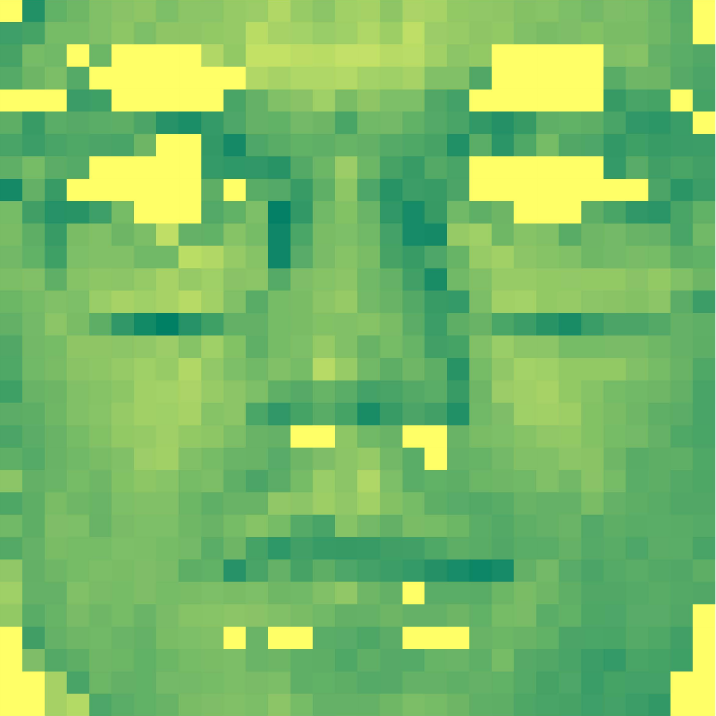} & \includegraphics[width=1\linewidth, valign=c]{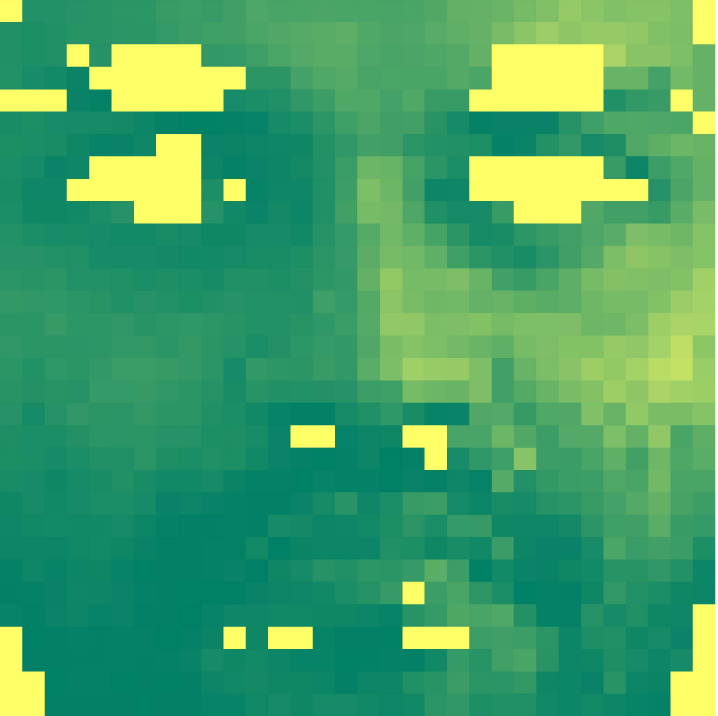} & \includegraphics[width=1\linewidth, valign=c]{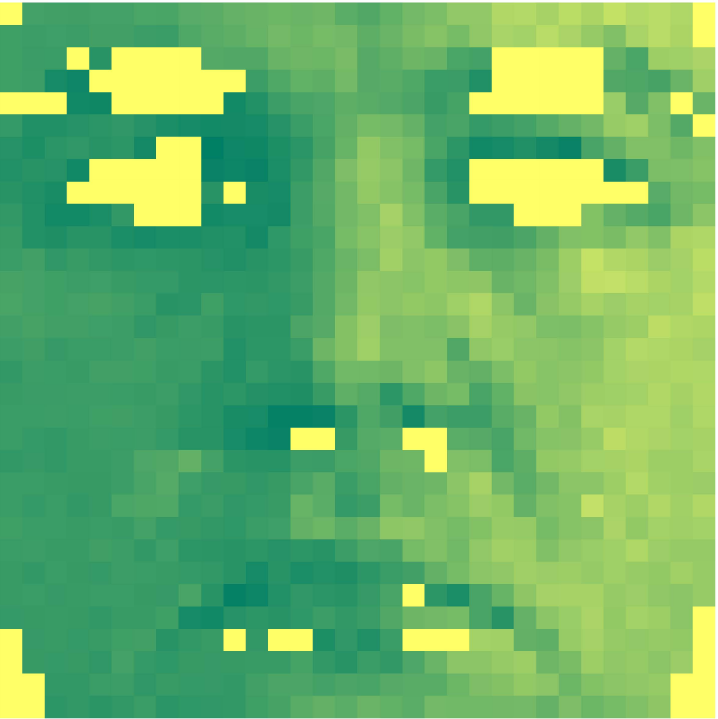} & \textbf{45.30} & \textbf{67.75} \\
    \end{tblr}
    \label{face}
\end{table}

\begin{table*}[t]
  \centering
  \renewcommand\arraystretch{1.2}
  \caption{Classification results on different datasets. The best results are marked in bold.}

  \begin{tabular}{|c|c|c|c|c|c|c|c|c|c|c|}
  \hline 
  Datasets & LapScore & MCFS &UDFS   &  SOGFS& RNE& SPCAFS & FSPCA & BLSFE & BLSFS & BLUFS \\
  \hline \hline
    Isolet & 69.21 & 68.44 & 66.62 & 38.96 &  38.96& 56.79 & 71.67 &76.02&  78.37 &  \textbf{79.51} \\
  \hline
   Jaffe & 82.12 & 90.56 & 91.41 & 81.50 &  81.50& 69.81 & 89.95 &83.96&  90.93&  \textbf{92.45} \\
  \hline
   pie & 64.68 & 62.81 & 66.35 & 58.33 &  58.33& 62.50 & 64.58 &66.72&  73.75&   \textbf{80.96}\\
  \hline
  COIL20 & 85.97 & 86.21 & 87.21 &  57.14& 57.14 & 68.75 & 80.83 &78.47&  87.71 &  \textbf{87.91} \\
  \hline
  MSTAR & 97.38 & 96.72 & 97.15 & 97.03 & 96.46 & 83.49 & 97.51 &96.33&  97.64&  \textbf{98.01} \\
  \hline
  warpAR & 28.38 & 28.84 & 25.76 & 26.69 & 26.69 & 26.15 & 33.30 & 33.84& 30.15 &  \textbf{37.43} \\
  \hline
  lung & 81.08 & 84.60 & 86.53 & 78.76 & 78.76 & 90.09 & 86.73 & 91.39& 85.97 &  \textbf{92.07} \\
  \hline
  gisette & 86.00 & 75.88 & 73.41 &  75.19&  76.56& 84.82 & 57.19 & 86.54& 76.73 &  \textbf{94.07}\\
 
  \hline\hline
  Average
   & 71.85 & 70.50  & 74.30  &  59.20 &   64.30&67.80 &72.72  & 70.53 & 77.65 & \textbf{82.80}\\
  \hline
  \end{tabular}
  
  \label{table2}
  \end{table*}

\subsection{Ablation Experiments}

Since the proposed BLUFS involves two levels, this subsection compares it with two degraded versions that use only one level. In addition, we also compare it with a degraded version that does not include graph learning. See Table \ref{sum1} for more details.


\subsubsection{Clustering Efficacy}As shown in Table \ref{table:5}, Case I often exhibits better performance compared to Case II, suggesting that the feature level is of greater importance than the clustering level. This result could arise because the clustering level disregards the initial feature selection information, relying solely on the clustering result. In contrast to both Case I and Case II, BLUFS consistently outperforms these single-level approaches, thereby highlighting the efficacy of the bi-level framework. Moreover, compared to Case III, BLUFS often demonstrates improved performance, emphasizing the benefits of incorporating adaptive graph learning to further elevate performance.

\subsubsection{Feature Visualization}
As illustrated in Table \ref{face}, visual comparisons of feature selection result are presented from various degenerate versions in the pie dataset. In this experiment, the number of selected features is maintained at 100, and four image samples are employed to assess the effectiveness. Although all methods are capable of identifying crucial facial features such as eyes, mouth, eyebrows, and nose, the features selected by BLUFS are more concentrated, with almost no scattered irrelevant points. This focused feature selection aids in preserving the overall geometric structure of the face, enhances the utilization of smaller regions, and effectively reduces the number of redundant features.
As a result, our proposed BLUFS achieves higher ACC and NMI scores, indicating its superior performance.

\subsection{Classification Verifications}

In this section, we evaluate the classification performance of BLUFS against other methods (excluding the Baseline) to demonstrate its effectiveness in the classification task. For each dataset, we randomly select 50\% of the instances to form the training set, the remaining instances serving as the testing set. It is important to note that this procedure is repeated 50 times, resulting in the creation of 50 distinct partitions.

\begin{figure}[t]
  \makeatletter
  \renewcommand{\@thesubfigure}{\hskip\subfiglabelskip}
  \makeatother
  \centering
  \subfigure{
    \includegraphics[width=3.2 in]{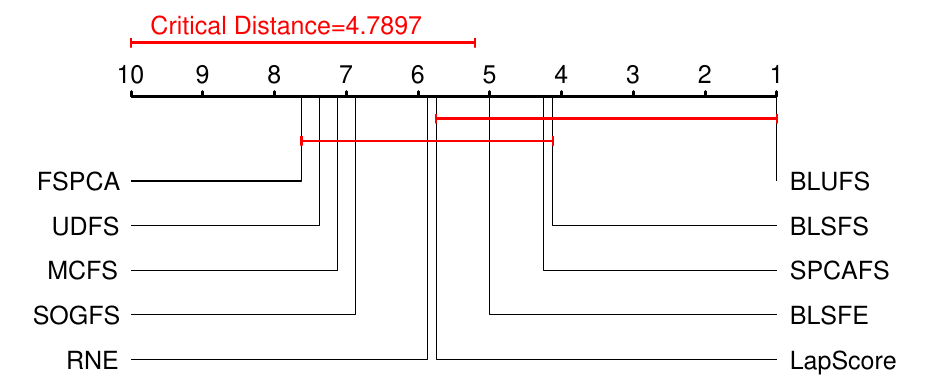}}
  \vskip-0.2cm
  \caption{Post-hoc Nemenyi test of all compared methods.}
  \label{figure-5}
\end{figure}

\begin{figure}[t]
  \centering
  \subfigcapskip=-1pt
  \subfigure[Isolet (BLSFS)]{
      \centering
      \includegraphics[width=3.8cm]{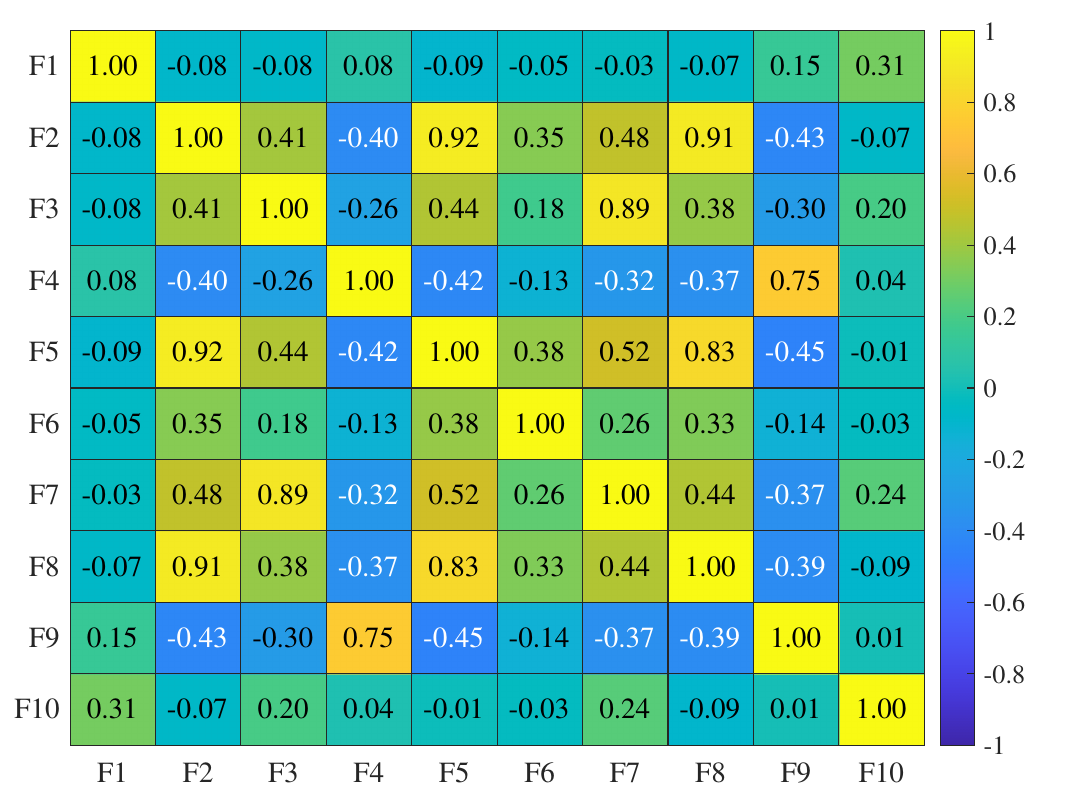}
  }\hspace{-0mm}
  \subfigcapskip=-1pt
  \subfigure[COIL20 (BLSFS)]{
      \centering
      \includegraphics[width=3.8cm]{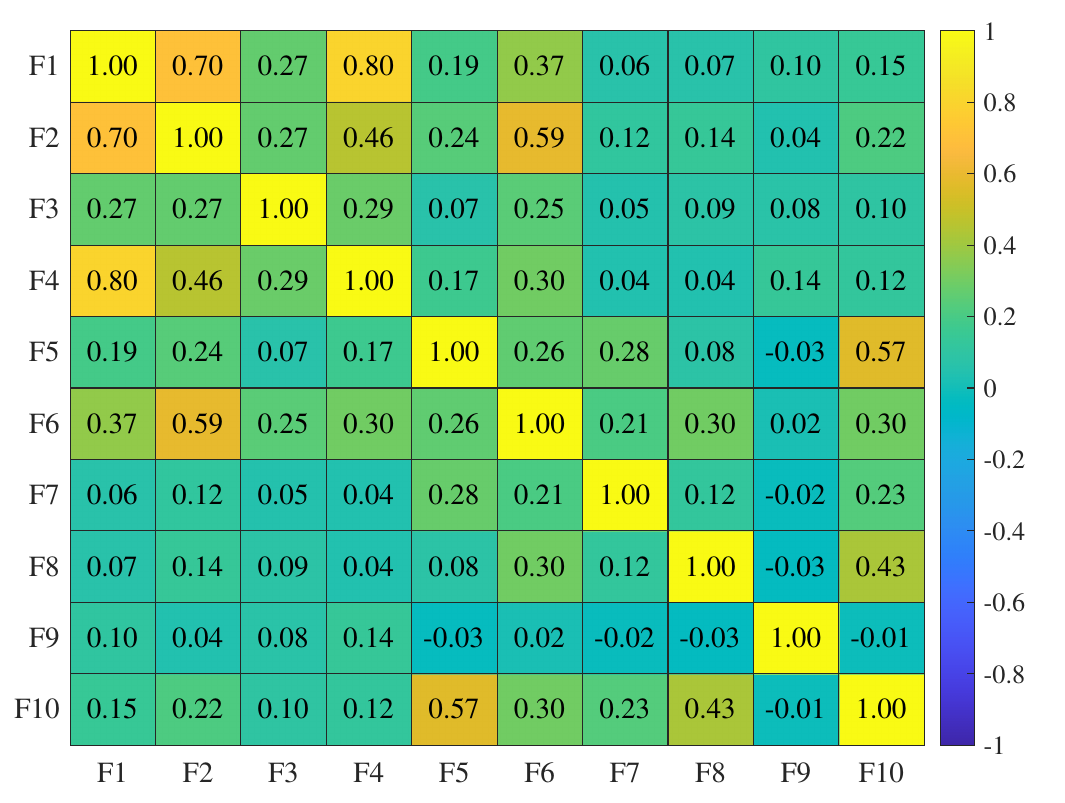}
  }\hspace{-0mm}
  
\subfigcapskip=-1pt
  \subfigure[Isolet (BLUFS)]{
      \centering
      \includegraphics[width=3.8cm]{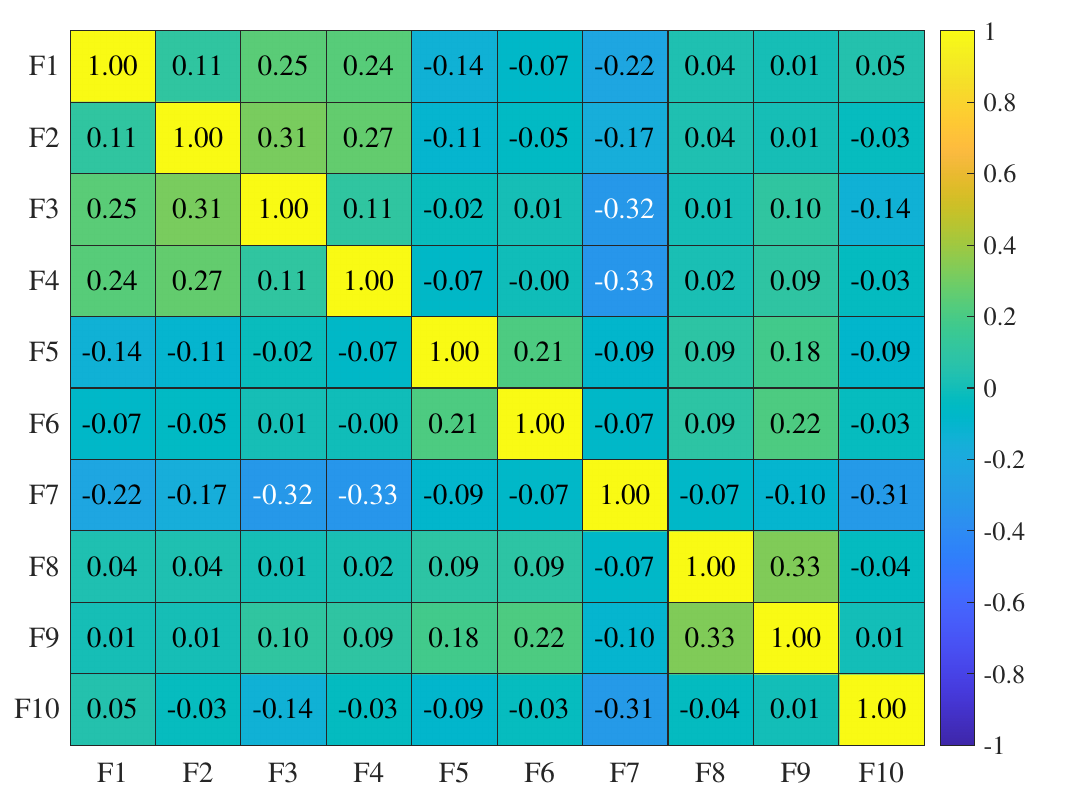}
  }\hspace{-0mm}
  \subfigcapskip=-1pt
  \subfigure[COIL20 (BLUFS)]{
      \centering
      \includegraphics[width=3.8cm]{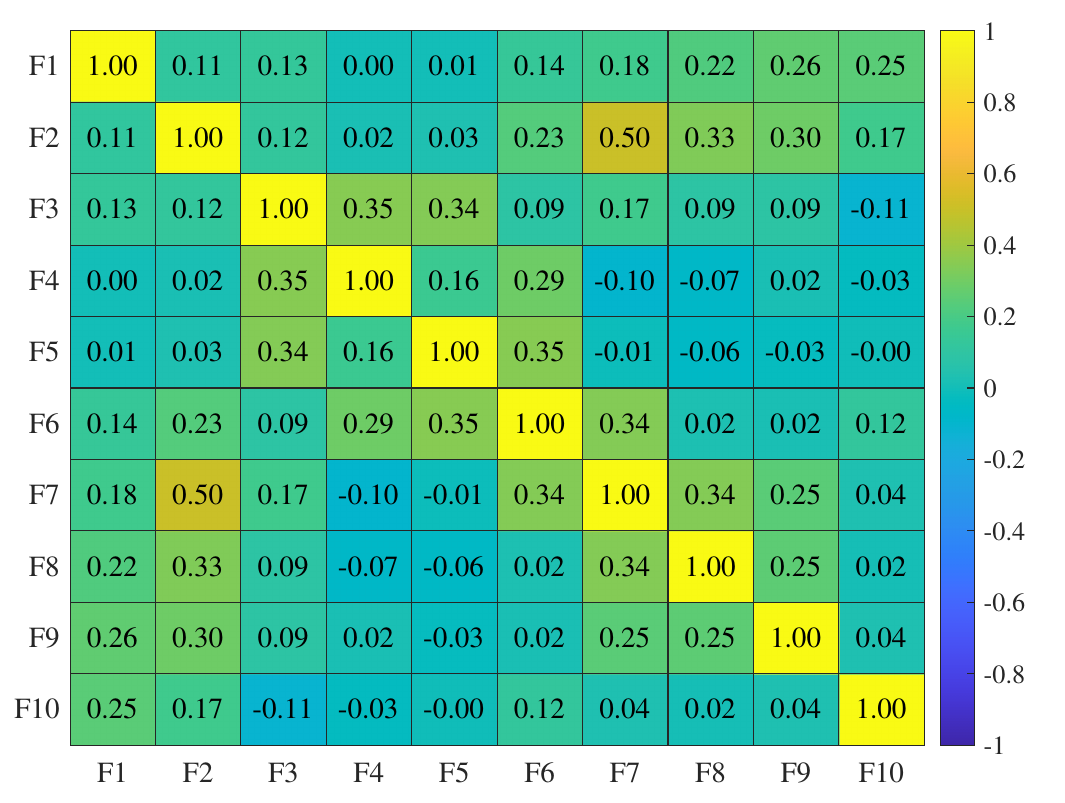}
  }\hspace{-0mm}
  
  \vspace{-0.1cm}
  \caption{Heatmap visualizations of correlations for 10 selected features,  where (a)-(b) are the results of BLSFS and (c)-(d) are the results of BLUFS.}
  \centering
  \label{correlation}
  \end{figure}

\begin{figure}[t]
    \centering
\subfigure[$\alpha$ and $\beta$ (Isolet)]{
      \includegraphics[width=1.34 in]{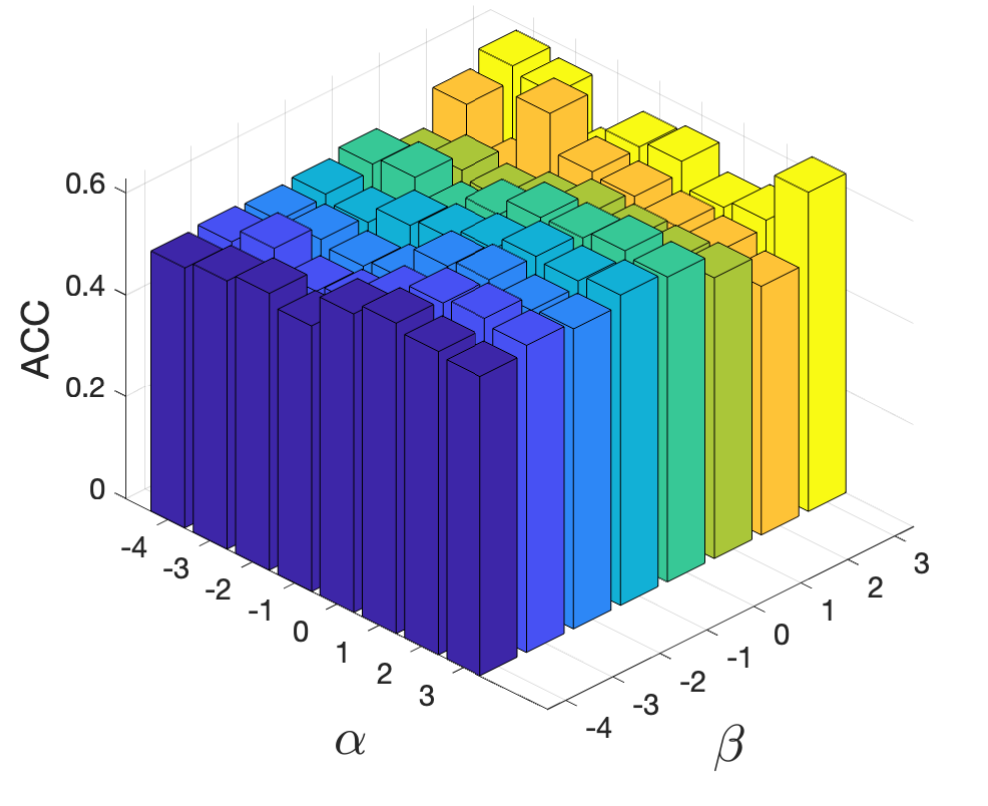}}
         \subfigure[ $\mu$ and $\lambda$ (Isolet)]{
      \includegraphics[width=1.34 in]{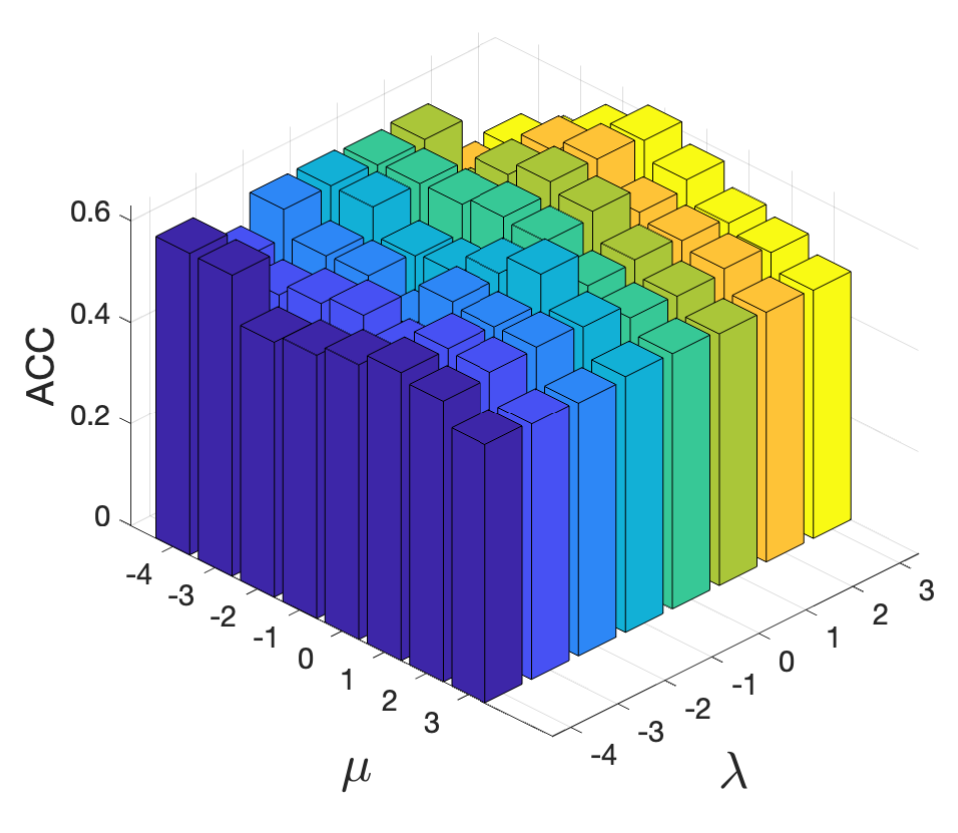}}
    \subfigure[$\alpha$ and $\beta$ (COIL20)]{
      \includegraphics[width=1.34 in]{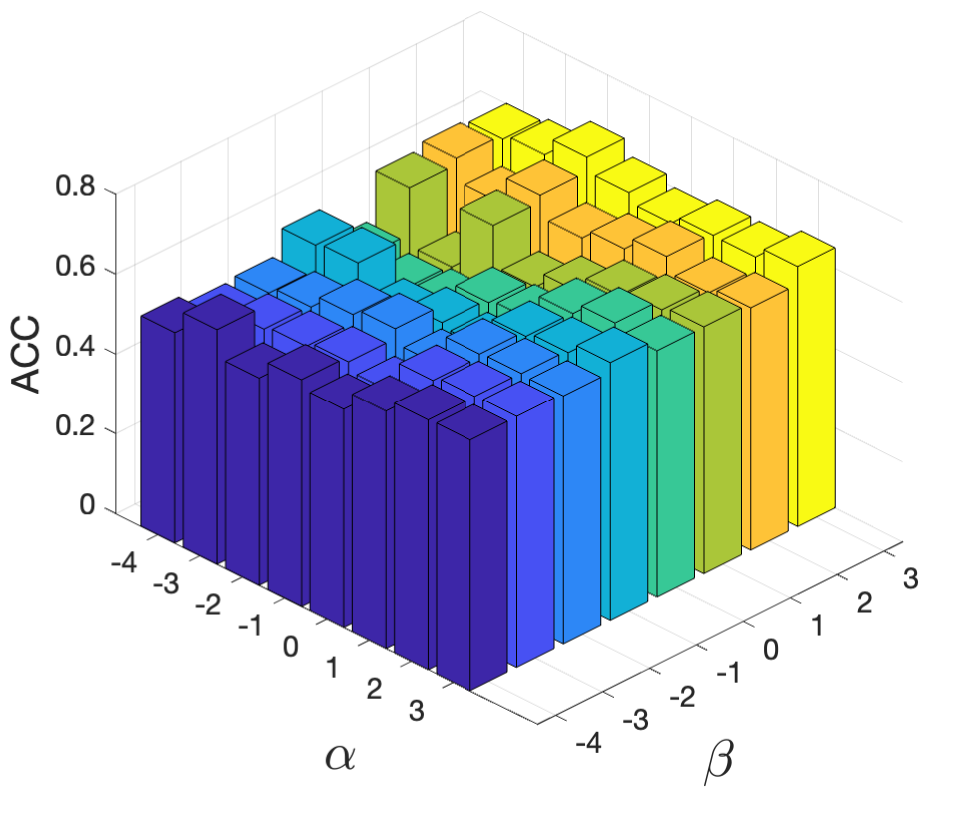}}
      \subfigure[ $\mu$ and $\lambda$ (COIL20)]{
      \includegraphics[width=1.34 in]{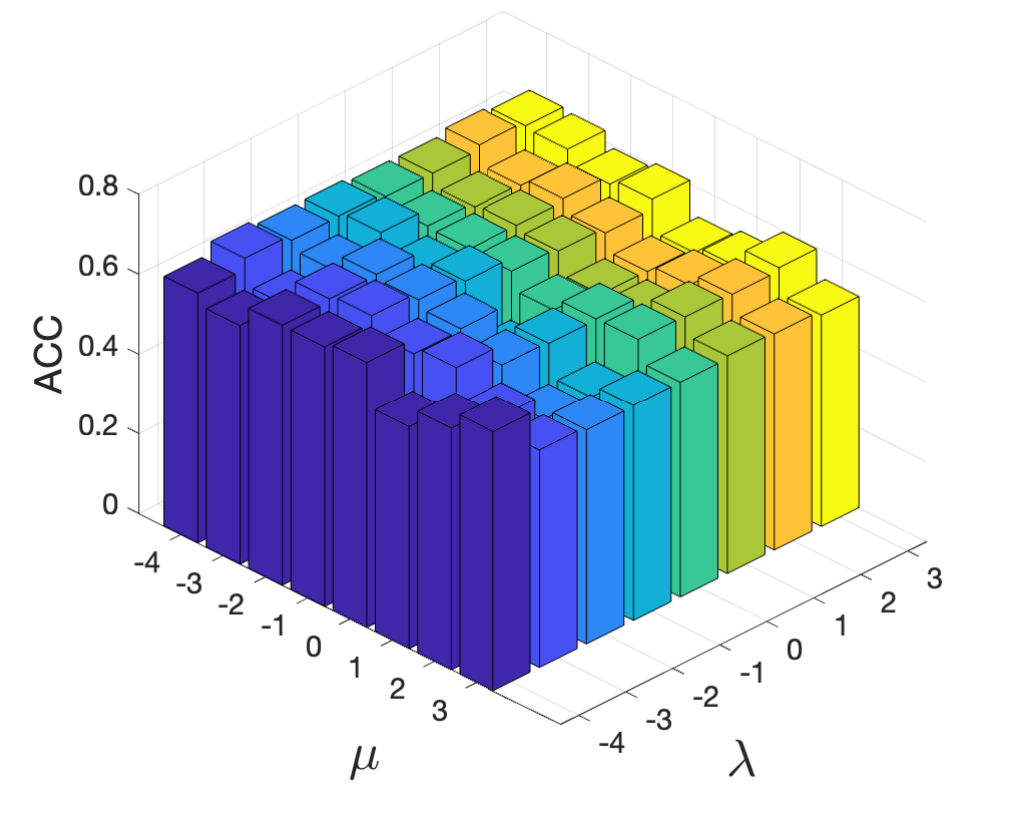}}
    \caption{Clustering accuracy with different parameters.}
    \label{figure:7}
  \end{figure}

 \begin{figure}[t]
    \centering
      \subfigcapskip=-1pt
  \subfigure[Isolet]{
      \centering
      \includegraphics[width=4cm]{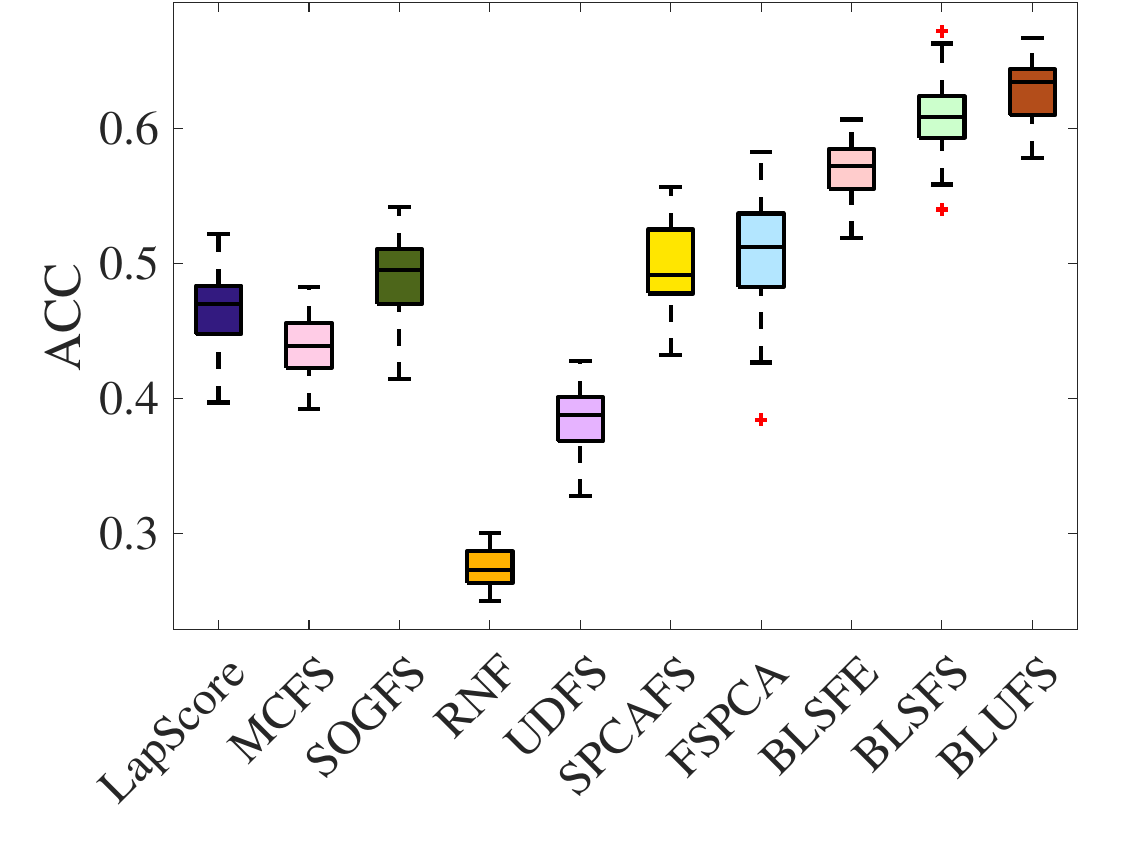}
  }\hspace{-0mm}
  \subfigcapskip=-1pt
  \subfigure[pie]{
      \centering
      \includegraphics[width=4cm]{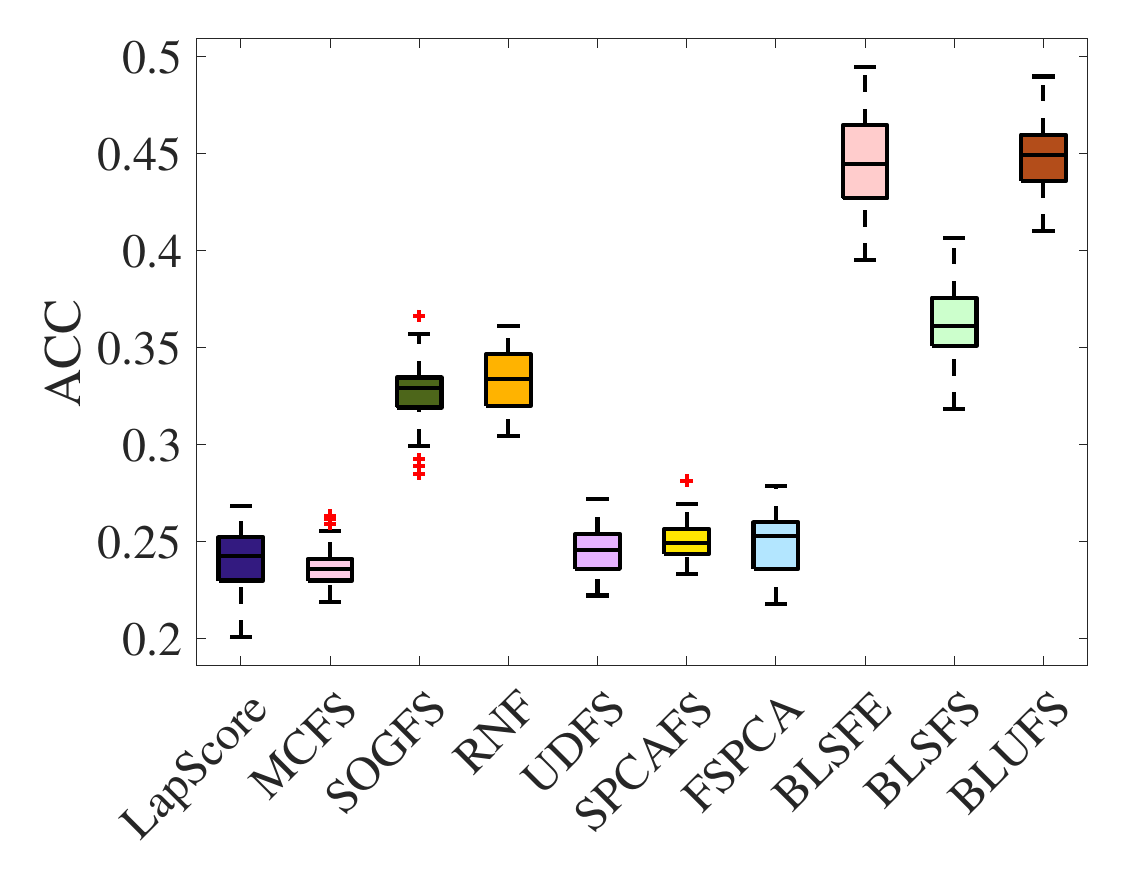}
  }\hspace{-0mm}
 
   \subfigcapskip=-1pt
 
  \subfigure[COIL20]{
      \centering
      \includegraphics[width=4cm]{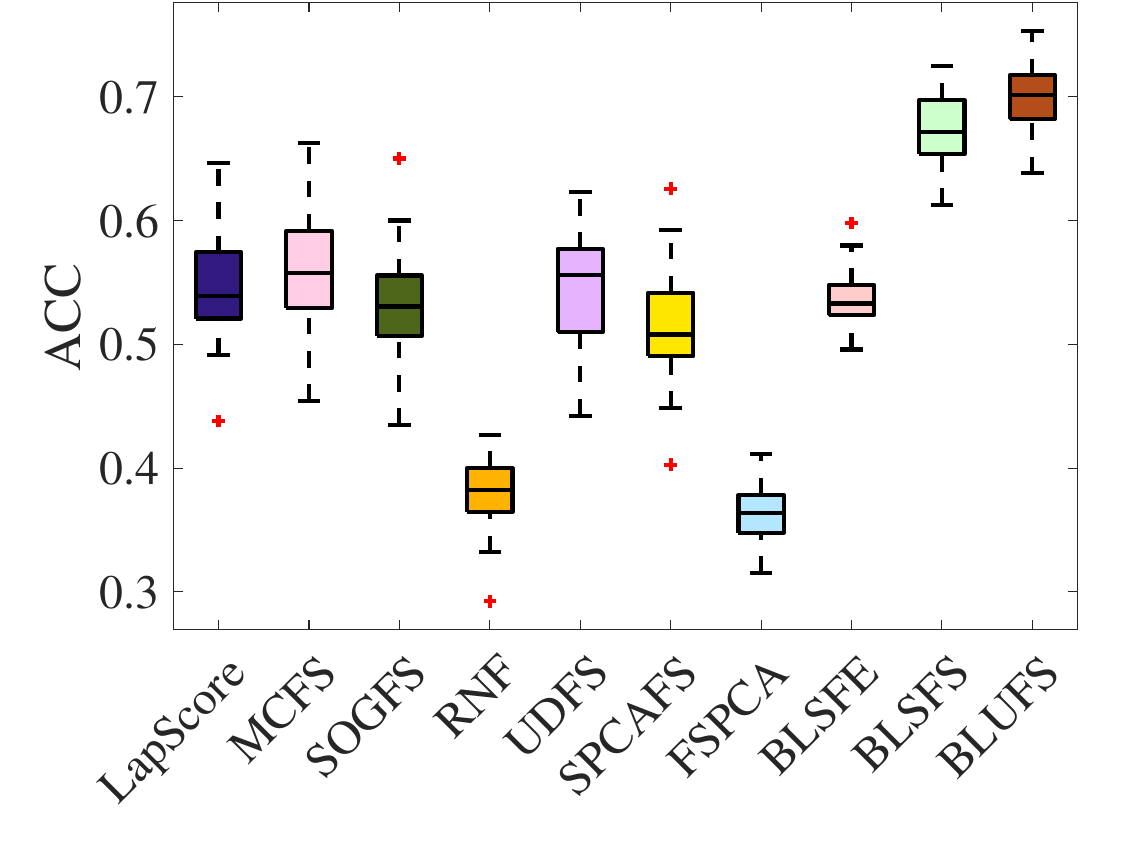}
  }\hspace{-0mm}
  \subfigcapskip=-1pt
  \subfigure[warpAR]{
      \centering
      \includegraphics[width=4cm]{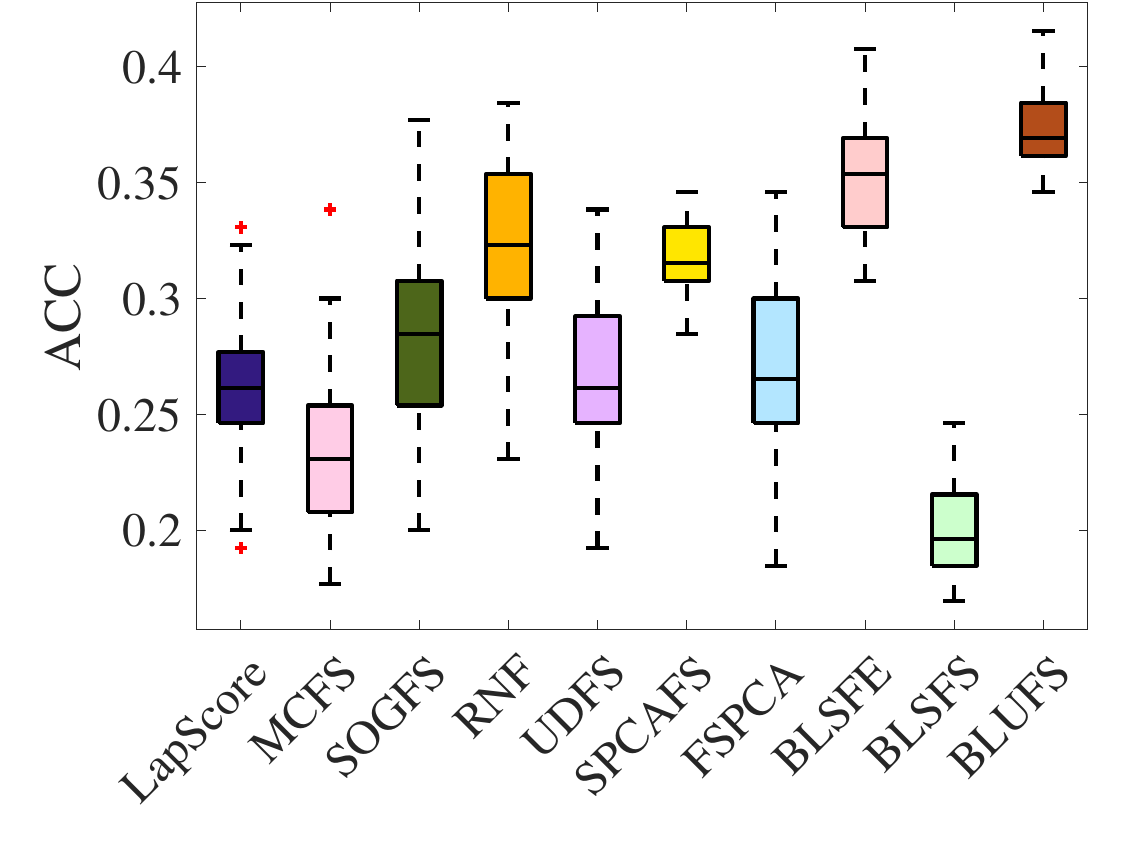}
  }\hspace{-0mm}
   
  \vspace{-0.1cm}
  \caption{Model stability comparisons on four selected datasets.}
  \label{box}
  \end{figure}

  \begin{figure}[!ht]
    \makeatletter
    \renewcommand{\@thesubfigure}{\hskip\subfiglabelskip}
    \makeatother
    \centering
     \subfigure[(a) Isolet]{
      \includegraphics[width=1.34 in]{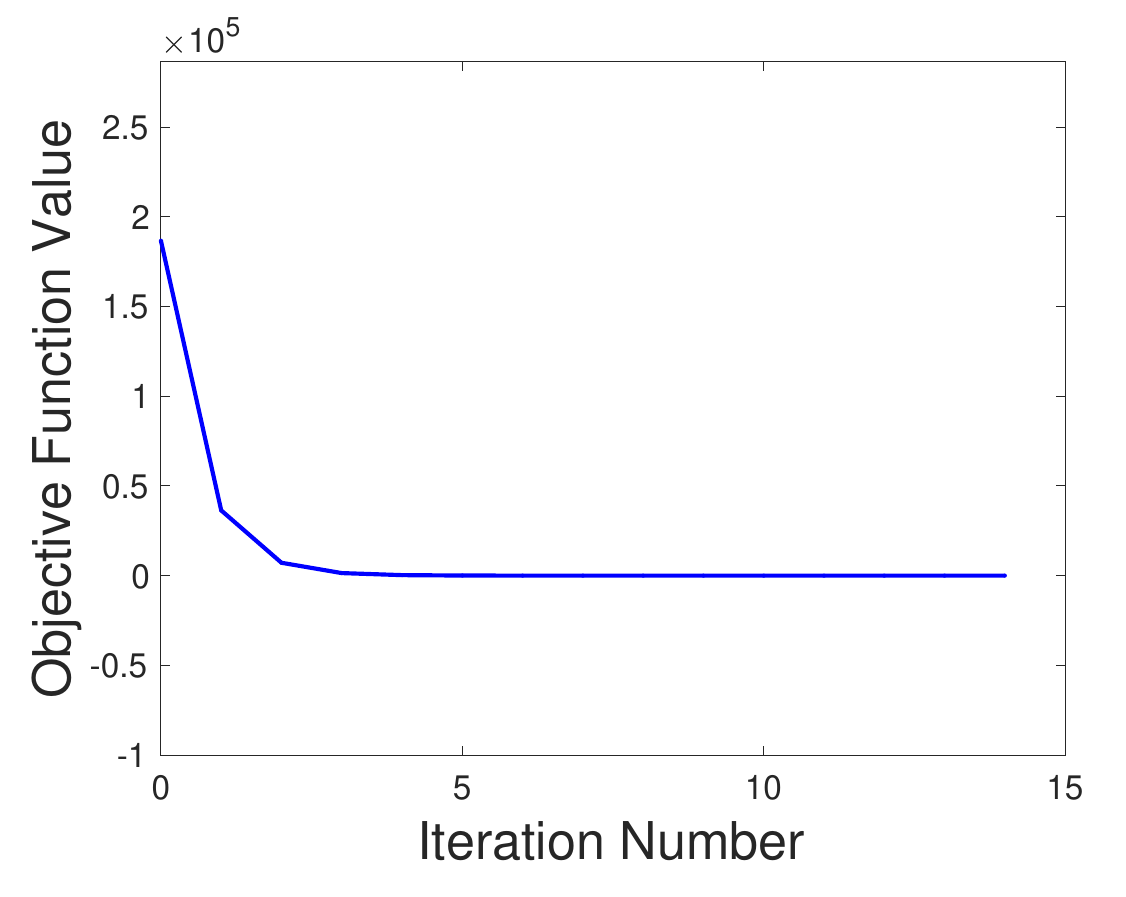}}
         \subfigure[(b) pie]{
      \includegraphics[width=1.34 in]{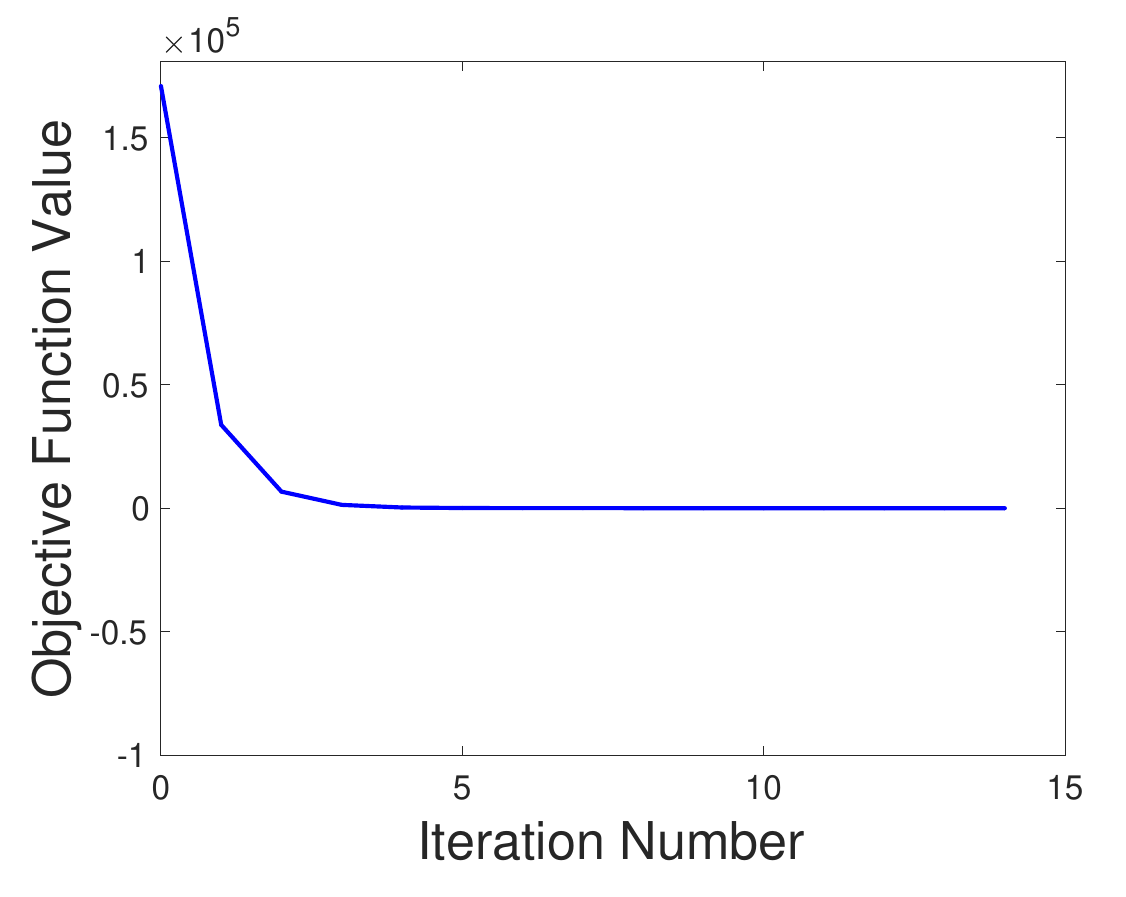}}
    \subfigure[(c) COIL20]{
      \includegraphics[width=1.34 in]{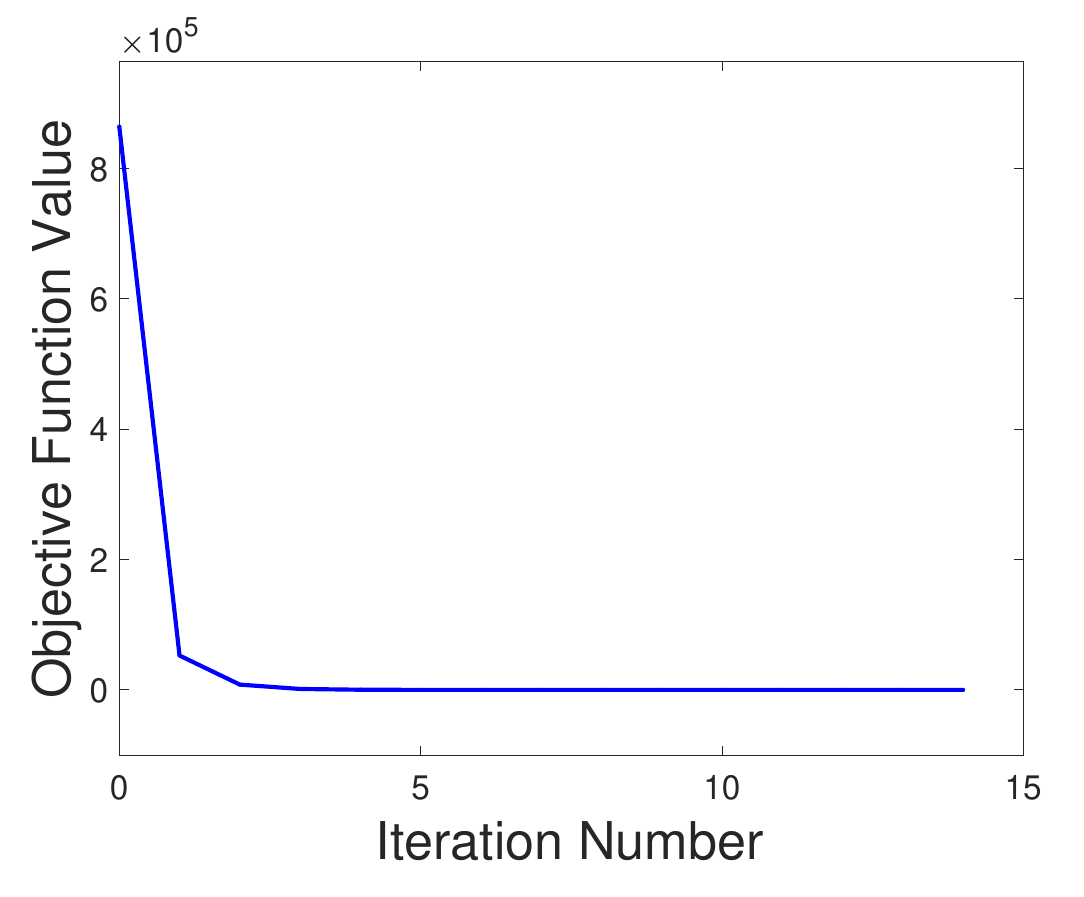}}
      \subfigure[(d) warpAR]{
      \includegraphics[width=1.34 in]{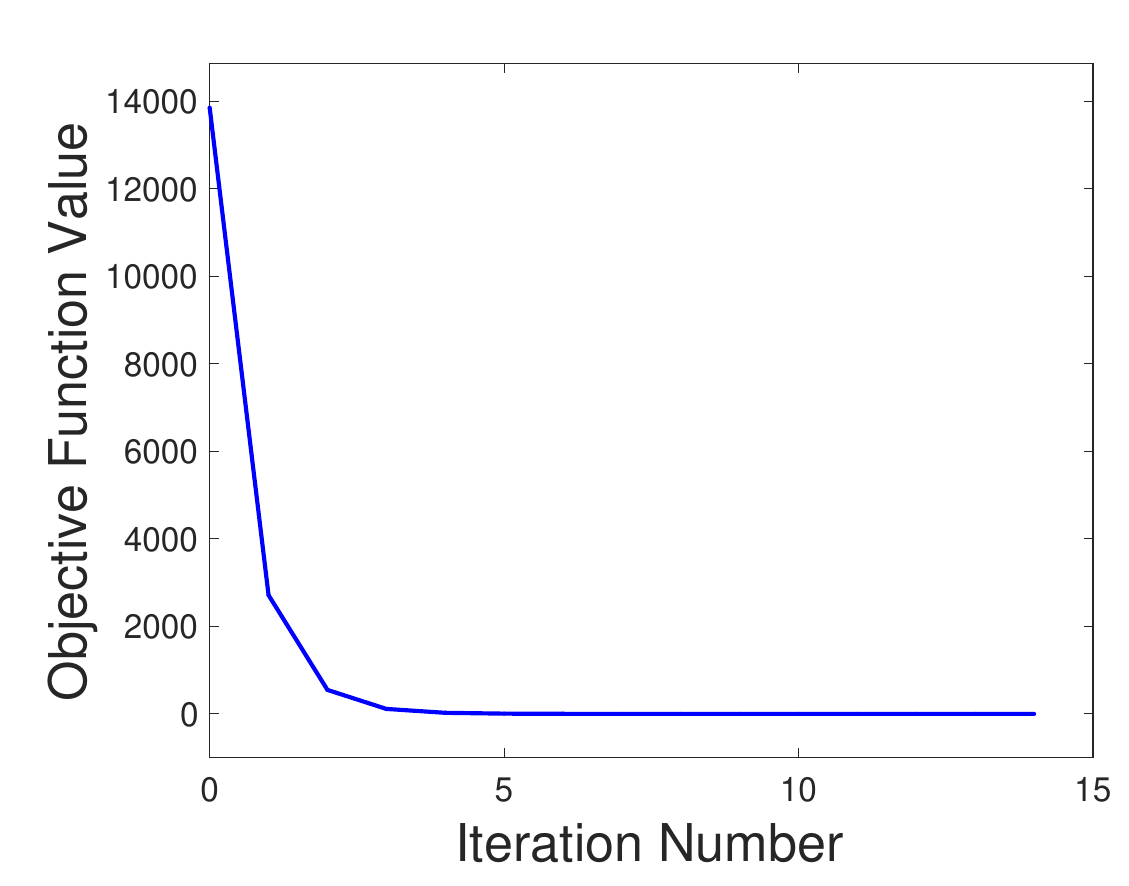}}
    \vskip-0.2cm
    \caption{Convergence curves of BLUFS on selected datasets.}
    \label{figure:6}
  \end{figure}

  \begin{table}[t]
    \centering
    \renewcommand\arraystretch{1.2}
    \caption{Running time in seconds on different datasets. The best results are marked in bold.}
    
    \begin{tabular}{|c|c|c|c|c|c|c|c|c|c|c|}
    \hline
    Datasets  & SPCAFS & FSPCA &  BLSFE & BLSFS & BLUFS \\
    \hline \hline
    COIL20 &    8.39 &   6.14 & 10.60  & \textbf{2.78}   & 6.35  \\
    \hline
    Jaffe &   2.96 &  16.53  & 1.77   & \textbf{0.61}   & 0.92  \\
    \hline
    Isolet &  10.70  &  1.18 & 10.19  & \textbf{2.53}   & 7.29  \\
    \hline
    warpAR &  17.42 & 436.03  & 13.17  & \textbf{2.15}   & 3.47  \\
    \hline
    lung &   40.56&  934.15 & 38.19  & 4.03   & \textbf{3.92}  \\
    \hline
    pie &    6.88 &  \textbf{3.61}  & 12.58  & 7.26   & 4.38  \\
    \hline
    gisette &  290.04& 3496.44 & 318.53 & 122.73 & \textbf{74.03} \\
    \hline
    MSTAR &   8.73  & 3.51   & 13.71  & \textbf{6.97}   & 11.08\\
    \hline
    \end{tabular}
    \label{table8}
    \end{table}

Table \ref{table2} records the average classification results.
Our method leads to superior classification performance on eight datasets. It is noteworthy that in the pie dataset, BLUFS outperforms BLSFS by 7.21\% in classification accuracy, and in the gisette dataset, BLUFS surpasses BLSFE by 7.53\%. Furthermore, on average across the eight datasets, BLUFS has a classification accuracy that is 5.15\% higher than that of BLSFS.
It can be seen that, for the classification task, our proposed method can still achieve the best performance.
Additionally, we can also observe that in classification experiments, the top two ranking methods are both bi-level methods, indicating that bi-level methods can indeed fully utilize information from different levels to achieve better performance in classification tasks.



\subsection{Discussion}

This subsection discusses statistical tests, feature correlation, parameter sensitivity, model stability, convergence analysis, and running time.

\subsubsection{Statistical Tests}
To demonstrate the reliability of the experimental results, we employed the post-hoc Nemenyi test, using the critical difference (CD) as a criterion to evaluate the pairwise differences in ACC between the compared methods. As shown in Fig. \ref{figure-5}, although there is no statistically significant difference between our proposed BLUFS and other methods such as BLSFS, BLSFE, SPCAFS, and LapScore, these bi-level methods generally outperformed the other methods. Moreover, it can also be seen that there is no significant difference among all the other methods except BLUFS, which indirectly highlights the advantage of our proposed BLUFS.

\subsubsection{Feature Correlation}
Fig. \ref{correlation} shows the feature correlation results obtained using BLSFS and our proposed BLUFS on the Isolet and COIL20 datasets.
Here,  10 features are selected,  denoted \(F_1,  F_2,  \cdots,  F_{10}\),  and the correlations among these features are examined.
It can be concluded that BLUFS has a lower feature correlation compared to BLSFS,  indicating that it is very effective in eliminating redundant features.

\subsubsection{Parameter Sensitivity}
This subsection studies the sensitivity of the parameters \(\alpha\), \(\beta\), \(\mu\), and \(\lambda\).
First, set \(\mu\) and \(\lambda\) to 1, and investigate the sensitivity to \(\alpha\) and \(\beta\).
Then fix \(\alpha\) and \(\beta\), investigate the sensitivity to \(\mu\) and \(\lambda\).
Fig. \ref{figure:7} shows the ACC clustering results under different parameters in the Isolet and COIL20 datasets.
It can be observed that the clustering performance is not highly sensitive to the selection of these parameters,  indicating that our proposed BLUFS is relatively robust to parameter variations.

\subsubsection{Model Stability}
Fig. \ref{box} shows the box plots of the 50 clustering results.
It is obvious that in terms of ACC, the average values of BLUFS are generally higher than those of other methods.
In general, our proposed BLUFS exhibits excellent and stable performance.

\subsubsection{Convergence Analysis}
Fig. \ref{figure:6} illustrates the convergence behavior of our proposed BLUFS across four real-world datasets of varying sizes.
Obviously, Algorithm \ref{algorithm 1} usually only needs 15 iterations to achieve fast convergence, which is consistent with the previous theorem.

\subsubsection{Running Time}
In terms of running time comparison, our proposed BLUFS is compared with four recently proposed methods, including FSPCA, SPCAFS, BLSFE, and BLSFS.
Table \ref{table8} gives the average time of ten repeated runs in each dataset.
It can be seen that although the running time of our proposed BLUFS is not the fastest on all datasets, it is acceptable considering its excellent performance on clustering and classification tasks.

\section{Conclusion}\label{Conclusion}
This paper presents a novel bi-level unsupervised feature selection method. Previous bi-level approaches overly simplify the representation of clustering structures by using discrete pseudo-labels. Moreover, the use of the $\ell_{2,1}$-norm often leads to insufficient sparsity, which performs poorly on high-dimensional datasets. To address these issues, we relax the constraint on pseudo-labels and adopt continuous pseudo-labels to represent the clustering structure. Furthermore, we integrate the bi-level structure with the $\ell_{2,0}$-norm to ensure strict sparsity. We also propose the PAM algorithm to guarantee convergence to the global optimum from any initial point. Extensive experiments on both synthetic and real-world datasets demonstrate the superiority of our method in clustering and classification tasks.

Although our bi-level framework significantly improves feature selection performance, it introduces additional computational costs due to the bi-level optimization. In future work, we plan to explore acceleration techniques \cite{li2021efficient}  or simplified algorithms \cite{zhang2023physics} to reduce computational costs. Additionally, while we achieved optimal results on most datasets, the clustering accuracy remains relatively low on some datasets, even for traditional methods. Therefore, future research may explore the application of deep learning techniques \cite{zhao2025deep} in the feature selection domain to further enhance performance.

\bibliographystyle{IEEEtran}
\bibliography{mybibfile}

\end{document}